\documentclass[pdftex,sn-basic]{sn-jnl}% Math and Physical Sciences Reference Style
\usepackage{tabularx}
\usepackage{graphicx}%
\usepackage{multirow}%
\usepackage{amsmath,amssymb,amsfonts,amscd}%
\usepackage{bm,bbm}
\usepackage{amsthm}%
\usepackage{mathrsfs}%
\usepackage[title]{appendix}%
\usepackage{xcolor}%
\usepackage{textcomp}%
\usepackage{manyfoot}%
\usepackage{booktabs}%
\usepackage{algorithm}%
\usepackage{algorithmic}%
\usepackage{listings}%
\usepackage{subfigure}   %%%maybe not allowed

\usepackage{threeparttable}

\usepackage{tikz}
\newtheorem{theorem}{Theorem}%  meant for continuous numbers
\newtheorem{remark}{Remark}%
\newtheorem{lemma}{Lemma}
\newtheorem{definition}{Definition}%

\newtheorem{assumption}{Assumption}%
\newcommand*\circled[1]{\tikz[baseline=(char.base)]{
            \node[shape=circle,draw,inner sep=0.5pt] (char) {#1};}}
\begin{document}

\title[Article Title]{Efficient Private SCO for Heavy-Tailed Data via Averaged Clipping}

\author[1]{\fnm{Chenhan} \sur{Jin}}\email{chjin2@math.cuhk.edu.hk}

\author[1]{\fnm{Kaiwen} \sur{Zhou}}\email{kwzhou@cse.cuhk.edu.hk}

\author[2]{\fnm{Bo} \sur{Han}}\email{bhanml@comp.hkbu.edu.hk}

\author*[1]{\fnm{James} \sur{Cheng}}\email{jcheng@cse.cuhk.edu.hk}

\author*[1]{\fnm{Tieyong} \sur{Zeng}}\email{zeng@math.cuhk.edu.hk}

\affil*[1]{\orgname{The Chinese University of Hong Kong}, \city{Hong Kong}}

\affil[2]{\orgname{Hong Kong Baptist University}, \city{Hong Kong}}


\abstract{We consider stochastic convex optimization for heavy-tailed data with the guarantee of being differentially private (DP).
Most prior works on differentially private stochastic convex optimization for heavy-tailed data are either restricted to gradient descent (GD) or performed multi-times clipping on stochastic gradient descent (SGD), which is inefficient for large-scale problems. In this paper, we consider a one-time clipping strategy and provide principled analyses of its bias and private mean estimation. We establish new convergence results and improved complexity bounds for the proposed algorithm called AClipped-dpSGD for constrained and unconstrained convex problems. We also extend our convergent analysis to the strongly convex case and non-smooth case (which works for generalized smooth objectives with H$\ddot{\text{o}}$lder-continuous gradients). All the above results are guaranteed with a high probability for heavy-tailed data. Numerical experiments are conducted to justify the theoretical improvement.}




\maketitle

\section{Introduction}\label{sec1}
Stochastic Convex Optimization (SCO) \citep{1995The} and its empirical form, Empirical Risk Minimization (ERM) \citep{1995The}, have been widely used in areas such as medicine, finance, genomics, and social science. Today, machine learning tasks often involve sensitive data, which leads to privacy-preserving concerns. This means that a machine learning algorithm not only needs to learn effectively from data but also provides a certain level of privacy-preserving guarantee. As a widely accepted concept for privacy preservation,
differential privacy (DP) \citep{Dwork2006CalibratingNT} provides the provable guarantee that an algorithm learns statistical characteristics of the
population, but nothing about individuals. Differentially private algorithms have been widely studied and recently deployed in industry \citep{tang2017privacy,ding2017collecting}.

In this work, we focus on Differentially Private Stochastic Convex Optimization (DP-SCO), which is started by \cite{bassily2014private}. The problem of DP-SCO aims to find a $x^{priv}\in \mathbb{R}^d$ that minimizes the population risk, i.e.,
\begin{equation} \label{OP} 
\min_{x \in \mathbb{R}^{d}}f(x), \  f(x)=\mathbb{E}_{\xi}[f(x,\xi)], 
\end{equation}
with the guarantee of being differentially private. Here, $\xi $ is a random variable on the probability space $\Omega$ with some unknown distribution $\mathcal{P}$. The function $f(x)$ is a smooth convex loss function and takes the expectation over $\xi$. The convergence metric of an algorithm is measured by the so-called \textit{excess population risk}, that is $\left(f(x^{priv})-\min_{x \in \mathbb{R}^{d}}f(x)\right).$ Besides the population risk, we introduce the differentially private empirical risk minimization (DP-ERM) over a fixed dataset $D=\{\xi_{i}\}_{i=1}^{n}$, i.e., $\min_{x}\widehat{f}(x, D),\ \widehat{f}(x, D)=\frac{1}{n}\sum_{i=1}^{n}f(x,\xi_i).$ DP-ERM appears frequently in the literature \citep{bassily2014private,talwar2015nearly,kasiviswanathan2016efficient,wang2017differentially,zhang2017efficient,wang2021estimating}. Besides DP-ERM, \cite{bassily2014private} studied DP-SCO and achieved a sub-optimal rate. Later, \cite{bassily2019private} established a tight analysis of excess population risk. After their work, some works focus on reducing the gradient complexity and running time \citep{feldman2020private} and different geometries \citep{asi2021private,bassily2021non}.

Almost all the previous results tackle the Problem (\ref{OP}) or its empirical form by estimating the mean of a random variable (such as gradient) based on the empirical mean estimator or its variants, and then adding random noise to achieve ($\epsilon$,$\delta$)-differential privacy. They all assume that either the loss function is $O(1)$ or $O(L)$-Lipschitz or each data sample is bounded $\ell_2$ or $\ell_\infty$ norm explicitly \citep{wang2020differentially}, which
restricts the sensitivity (see Definition \ref{sensitivity}) in order to establish the DP guarantee of their algorithms. However, the quality of the empirical mean is sensitive to data outliers, especially when the data are heavy-tailed\footnote{Heavy-tailed data refers to the data where outliers can be sampled with higher probability than data called ``light-tailed'' \citep{gorbunov2020stochastic}. We do not assume the data to be ``light-tails'', i.e., sub-Gaussian distribution meaning that for random vector $\eta$ and $b\geq 0$, if there exists $\mathbb{E}[\eta]$ and variance $\sigma$, $\mathbb{P}\{\|\eta-\mathbb{E}[\eta]\|_2\geq b\}\leq 2\exp{(-b^2/(2\sigma^2))}$.} \citep{lugosi2019mean}. This situation widely exists in the real world such as in biomedical engineering and finance \citep{ibragimov2015heavy}. Heavy-tailed data could lead to a very large empirical mean and thus ruin the above assumptions. This could cause big oscillations in the convergence trajectories, which may even lead to divergence.  

Recently, to tackle the above problems, several works have been proposed to privately estimate the mean of a heavy-tailed distribution via scaling or truncation. \cite{bun2019average} proposed a private mean estimator with a trimming framework. \cite{wang2020differentially} adopted it in the private gradient descent (GD) under some strong assumptions. Later,  \cite{kamath2020private} studied private mean estimation under a general setting that the mean of the distribution has bounded $k$-th ($k \geq 2$) moments. They proposed a one-round mean estimator called CDPHDME and established its $\ell_2$-error guarantees, which hold with constant probability. By taking inspiration from the private version of the Median of Means method studied in \citep{kamath2022improved}, \cite{tao2022private} extended the private analysis to the cases where the gradient is bounded $k$-th moment, $k \in (1,2].$ It is worth noting that for the specific case where $k=2,$ their results align with each other.
Recently, \cite{das2023beyond} considered the Generalized Lipschitzness to be bounded $k$-th moments, and provided the \emph{expected} excess population risk bounds for DP-SGD algorithm in (non) convex cases. They achieved a fast running time but their results are built upon taking the expectation of all the randomness of DP-SGD. %Note that the above algorithms of DP-SCO all have guarantees that hold in expectation. 
It is more desirable to establish excess population risk bounds that hold \emph{with high probability} as the outliers may cause large variances in the gradient-based algorithms, which cannot be captured by guarantees that hold in expectation \citep{brownlees2015empirical,wang2020differentially,gorbunov2020stochastic}.  \cite{kamath2022improved} extended the CDPHDME to a multiple-step
framework (DP-GD) and provided the statistical guarantees on excess population risk bounds with high probability. By adopting the scaling and truncation mean estimator in  \cite{holland2019robust}, \cite{wang2020differentially} established the high probability guarantees for a DP-GD algorithm in (strongly) convex cases. Using this estimator, some works focus on gradient expectation maximization \citep{wang2020differentially1} and high dimensional space \citep{hu2022high}. However, due to the usage of this mean estimator, their algorithm needs to solve an additional correction function in each iteration, which leads to a high computational cost. Moreover, in view of the prevalence of large-scale datasets in the industry, their GD-based methods have a high gradient complexity (that is, the number of computations of the called first-order oracle $\nabla f(x, \xi_i)$), which usually results in a long running time. Then, it is natural to ask:
\textit{is there any private variant of SGD that can deal with the heavy-tailed data? Can we establish its statistical guarantee with high probability?}

To answer these questions, we propose new analyses for a new private SGD algorithm that is capable of handling heavy-tailed data with high probability. We list our main contributions and also summarize the theoretical results in Table \ref{tab1}.

\begin{itemize}
\item We consider a one-time clipping strategy, referred to as AClip that clips once per iteration. Compared with the per-sample clipping strategy of DP-SGD, we explain how the number of clip operations affects the bias of an estimator (see Lemma \ref{bias}). We introduce our private mean estimator and derive its mean estimation error in a novel accumulated form (see Theorem \ref{cccc}). We provide a principled clip norm and batch size tuning strategy for AClipped-dpSGD under Assumption \ref{assump1}, in which the tuning strategy of batch size is treated as a hyperparameter in our compared works.

\item 
We provide the excess risk bounds for AClipped-dpSGD under the assumption of bounded second moment of gradients. For private constrained convex and unconstrained convex objectives, we derive the risk bounds with the expected high probability forms: $\Tilde{O}\left(\frac{d^{1/7}\sqrt{\log (n /\beta d^{2})}}{(n\epsilon)^{2/7}}\right)$ and  $\Tilde{O}\left(\frac{d^{1/7}\log (n /\beta d^{2})}{(n\epsilon)^{2/7}}\right),$ respectively (see Theorem \ref{BC} and \ref{ee}). These bounds do not require the bounded mean of gradient and outperform the prior result in \cite{wang2020differentially}. Moreover, AClipped-dpSGD is much faster than DP-SGD and DP-GD methods in terms of gradient complexity (see Table \ref{tab1}). We also provide the tuning strategies of clip norm and batch size, making AClipped-dpSGD more practical.   

\item We extend AClipped-dpSGD to the $\mu$-strongly unconstrained convex setting. For this extension, we establish
an excess population risk bound $\Tilde{O}\big(\frac{d^{1/2}L^2}{\mu^3n\epsilon}\big),$ with high probability (see Theorem \ref{h}). The result improves on the bound of $\Tilde{O}\big(\frac{d^{2}L^4}{\mu^3n\epsilon^2}\big)$ in \cite{wang2020differentially}
and also nearly matches the expectation form bound $\Tilde{O}\big(\frac{d^{1/2}L}{\mu^2n\epsilon}\big)$ in \cite{kamath2022improved}. Moreover, we do not require the bounded convex set (see Assumption \ref{assump2}) as they need.

\item We also consider the gradient satisfies the generalized smoothness condition (i.e., H$\ddot{\text{o}}$lder continuous
gradient). For this case, we derive a new high probability bound for DP SCO with heavy-tailed data and propose a novel tuning strategy of step size (see Theorem \ref{NONS}). Moreover, H$\ddot{\text{o}}$lder continuity implies the problem can be non-smooth and our analysis does not require it to hold on $\mathbb{R}^n.$
\end{itemize}

\begin{table}[h]
\centering

\caption{Summary of the (expected) excess population risk bounds (with Assumption \ref{assump1}).
The \emph{high probability bound} term refers to the excess risk bounds hold with high probability $1-\beta,$ instead of taking the expectation of randomness of algorithms.   \cite{kamath2022improved} and \cite{das2023beyond} assume that the gradients are bounded $k$-th moment, we consider $k=2$ to keep the consistency of the comparison. The complexity represents the gradient complexity (w.r.t. running time). \footnotetext{Note: 
\cite{wang2020differentially} and \cite{kamath2022improved} assume that the gradient to be coordinate-wise bounded, i.e., $\mathbb{E}_{\xi}[\|\nabla_j f(x,\xi)-\mu_j\|_2^2] \leq v$ for each coordinate $j \in [d],$ while the others consider the weaker one: $\mathbb{E}_{\xi}[\|\nabla f(x,\xi)-\mu\|_2^2]$ is bounded by $\sigma^2.$ The quoted bounds are obtained with $v=\sigma^2/d,$ and the total gradient complexity also changes accordingly.}}\label{tab1}

\begin{minipage}{\textwidth}
%\resizebox{\textwidth}{!}{
\begin{tabular*}{\textwidth}{c|l|l|l|c}
\toprule%
%& \multicolumn{3}{c}{Element 1\footnotemark[1]} & \multicolumn{3}{@{}c@{}}{Element 2\footnotemark[2]} \\
%\cmidrule{2-4}\cmidrule{5-7}%
%\midrule
Reference  & Convex Setting & \multicolumn{1}{|c|}{risk bound} & complexity & high probability? \\
\midrule
\cite{wang2020differentially}\footnote{The original rate in \cite{wang2020differentially} shall be $\Tilde{O}\left(\frac{dv^{1/3}}{(n\epsilon^2)^{1/3}}\right)$ and $\Tilde{O}\big(\frac{d^3L^4v}{\mu^3n\epsilon^2}\big)$ for the convex and strongly convex cases, as confirmed in \cite{kamath2022improved}.}  &  Constrained & $\Tilde{O}\left(\frac{d^{2/3}\sqrt{\log (n /\beta d^{2})}}{(n\epsilon^2)^{1/3}}\right)$ & $\Tilde{O}(n^{\frac{4}{3}}d^{\frac{1}{3}})$ & Yes \\
\midrule
\cite{kamath2022improved} &Constrained & $\Tilde{O}\left(\frac{d^{1/4}\log (dn/\beta)}{(n\epsilon)^{1/2}}\right)$  &   $\Tilde{O}(n^{2}d^{-\frac{3}{2}})$& Yes \\
\midrule
This work \footnote{The three works all assume explicitly or implicitly that the gradient (i.e., $\|\mathbb{E}_{\xi}[\nabla f(x,\xi)]\|_2$) has a bounded mean which we do not. Such an assumption is not always satisfied in heavy-tailed data \citep{cohen2020heavy}.} & Constrained & $\Tilde{O}\left(\frac{d^{1/7}\sqrt{\log (n /\beta d^{2})}}{(n\epsilon)^{2/7}}\right)$  &   $\Tilde{O}\left(n^{\frac{6}{7}}d^{\frac{4}{7}}\right)$& Yes \\
\midrule
\cite{das2023beyond}\footnote{The quoted complexity bound is obtained by using the minimum batch size, i.e., batch size $m=1.$ \cite{das2023beyond} do not provide any theoretical insights into how to set the batch size $m$ but we do.} & Unconstrained & $\Tilde{O}\left(\frac{d^{1/6}}{\beta^{1/2}(n\epsilon)^{1/3}}\right)$  &  $\Tilde{O}\left(n^2\right)$ & No \\
\midrule
This work & Unconstrained & $\Tilde{O}\left(\frac{d^{1/7}\log (n /\beta d^{2})}{(n\epsilon)^{2/7}}\right)$  &   $\Tilde{O}\left(n^{\frac{6}{7}}d^{\frac{4}{7}}\right)$& Yes \\
\botrule
\end{tabular*}
\end{minipage}

%\footnotetext[1]{}
%\footnotetext[2]{   }
%\footnotetext[3]{}
% The analysis of m related terms can be avoided by taking the expectation of randomness of DP-SGD.
\end{table}

%The Introduction section, of referenced text \cite{bib1} expands on the background of the work (some overlap with the Abstract is acceptable). The introduction should not include subheadings.

%Springer Nature does not impose a strict layout as standard however authors are advised to check the individual requirements for the journal they are planning to submit to as there may be journal-level preferences. When preparing your text please also be aware that some stylistic choices are not supported in full-text XML (publication version), including coloured font. These will not be replicated in the typeset article if it is accepted. 

\section{Preliminaries}\label{pre}

In this section, we provide the necessary background for our analyses, including differential privacy and
basic assumptions in convex optimization.
\subsection{Settings}
%\begin{definition}
%	(Lipschitz Function \cite{nesterov2018lectures}). A function ff is called CC-Lipschitz on Rd\mathbb{R}^d if for all x,y∈Rdx,y \in \mathbb{R}^d, the following holds, |f(x)−f(y)|≤C‖|f(x)-f(y)| \leq C \|x-y\|_2.
%\end{definition}
\begin{definition}
	($L$-smoothness \citep{nesterov2018lectures}). A differentiable function $f$ is called $L$-smoothness on $\mathbb{R}^d$ if for all $x,y \in \mathbb{R}^d$, it holds that, $ f(y) \leq f(x) + \langle \nabla f(x), y-x \rangle + \frac{L}{2}\|y-x\|_2^2.$ 
\end{definition}
Additionally, if $f$ is convex, then for all $x,y \in \mathbb{R}^d,$ the following holds \citep{nesterov2018lectures}: 
\begin{equation*}
\|\nabla f(y) - \nabla f(x)\|_2^2 \leq 2L(f(x)-f(y)-\langle \nabla f(y), x-y \rangle).
\end{equation*}
\begin{definition}
	($\mu$-strongly convex \citep{nesterov2018lectures}). A differentiable function ff is called $\mu$-strongly convex on $\mathbb{R}^d$ if for all $x,y \in \mathbb{R}^d$, it holds that, $f(y) \geq f(x) + \langle \nabla f(x), y-x \rangle + \frac{\mu}{2}\|y-x\|_2^2.$ 
\end{definition}
\begin{definition}
    The Projection operator on a convex set $\mathcal{X}$ for any $\theta \in \mathbb{R}^{d}$ is defined as $\text{Proj}_{\mathcal{X}}(\theta) = \arg \min_{x\in \mathcal{X}}\|\theta -x\|_2 .$  
\end{definition}
\subsection*{Gradient clipping} 
We consider a clipping of the averaged gradients framework: 
\begin{equation} \label{clip}
\begin{split}
&\text{AClip}(\nabla f(x,\bm{\xi}),\lambda):\\
&\widehat{\nabla} f(x,\bm{\xi})=\begin{cases}
\nabla f(x,\bm{\xi}), &\text{if } \|\nabla f(x,\bm{\xi})\|_{2}\leq \lambda,\\    	
\frac{\lambda}{\|\nabla f(x,\bm{\xi})\|_{2}} \nabla f(x,\bm{\xi}), &\text{otherwise,}
\end{cases}
\end{split}
\end{equation}
where $\nabla f(x,\bm{\xi})$ is a mini-batched version of $\nabla f(x)$. 

We refer to this strategy as AClip (Averaged clipping). That is, in order to compute $\text{AClip}(\nabla f(x,\bm{\xi}),\lambda),$
one needs to get $m$ i.i.d. samples $\nabla f(x,\xi_{1}) \ldots,\nabla f(x,\xi_{m})$ and compute its average $\nabla f(x,\bm{\xi})=\frac{1}{m}\sum_{i=1}^{m}\nabla f(x,\xi_{i}).$

Note that this method is a strategy of averaging first. It is quite different from the one adopted in DP-SGD \citep{abadi2016deep,das2023beyond}, in which they first clip the per-sampled gradient and then average the clipped gradients, i.e., $\frac{1}{m}\sum_{i \in m_t}\min\{1,\frac{\lambda}{\|\nabla f(x,\xi)\|_{2}} \}\nabla f(x,\xi),$ where $m_t$ is a sampled mini-batch at $t$-th iteration.  The detailed discussion can be found in Section \ref{clip discuss}.
\begin{assumption} \label{assump1}
	For the loss function and the population risk, we assume the following:
	
	1. The loss function $f(x,\xi)$ is non-negative, differentiable and convex.
	
	2. The population risk $f(x)$ is $L$-smooth, and satisfies $\nabla f(x^*)= \mathbf{0}$ at the optimal solution $x^*.$

	%3.The constraint set XX contains the following ℓ2\ell_2-ball centered at x∗:{‖x^*:\{\|x-x^{*}\|_{2}\leq \|x^{0}-x^{*}\|_{2}\}..
	
	3. Throughout the paper, we assume that for all $x \in \mathbb{R}^d,$ the stochastic gradient $\nabla f(x,\xi)$ of function $f(x,\xi)$  satisfies:
	\begin{equation} \label{scg}
	    \mathbb{E}_{\xi}[\nabla f(x,\xi)]=\nabla f(x),\ \mathbb{E}_{\xi}\left[\|\nabla f(x,\xi)-\nabla f(x)\|_{2}^{2}\right] \leq \sigma^{2},
\end{equation}
	%\begin{gather}
	%\mathbb{E}_{\xi}[f(x,\xi)]=\nabla f(x),\\    	   	
	%	\mathbb{E}_{\xi}[\|f(x,\xi)-f(x)\|_{2}^{2}] \leq \sigma.   	
	%\end{gather}  
	where $\sigma$ is some non-negative known number.
\end{assumption}

We point out that the first two assumptions are standard in (private) SCO or ERM \citep{ghadimi2012optimal,wang2020differentially,gorbunov2020stochastic,kamath2022improved,tao2022private}. The third one assumes the stochastic gradient $\nabla f(x, \xi)$ (can be heavy-tailed) is unbiased and has a bounded second moment, which is also commonly used for heavy-tailed data. Moreover, most of the above works additionally assume that loss function $f(x, \xi)$ is convex within a bounded set ($x \in \mathcal{W}$) or 
$\mathbb{E}_{\xi}[\nabla f(x,\xi)]$ is known or bounded, while we focus on the unconstrained case, i.e., $ \mathcal{W} = \mathbb{R}^d,$ which is more practical and harder for analysis.

\subsection{Differential privacy}
\begin{definition}
	(Differential Privacy \citep{Dwork2006CalibratingNT}). We say that two datasets D and D' are neighbors if they differ by only one entry, denoted as $\sim D'.$ An algorithm $\mathcal{A}$ is $(\epsilon,\delta)$-differentially private if and only if for all neighboring datasets D, D' and for all events S in the output space of $\mathcal{A},$ we have	
	$P(\mathcal{A}(D)\in S)\leq e^{\epsilon}P(\mathcal{A}(D')\in S)+\delta,$
	where $\delta=0$ and $\mathcal{A}$ is $\epsilon$-differentially private. 	
\end{definition}

The additive term $\delta$ (preferably smaller than $1/|D|$) is called the broken probability of $\epsilon$-differential privacy \citep{Dwork2006CalibratingNT}. The core concept and fundamental tools used in DP are sensitivity and Gaussian mechanism, introduced as follows.

\begin{definition} \label{sensitivity}
	($\ell_2$-sensitivity \citep{dwork2014algorithmic}). The $\ell_2$-sensitivity of a deterministic query $q(\cdot)$ is defined as $\Delta_2 (q) = \sup_{D \sim D'}\|q(D)-q(D')\|_2.$
\end{definition}

\begin{definition}
	(Gaussian Mechanism \citep{dwork2014algorithmic}). Given any function $q: \Omega \to \mathbb{R}^d$, the Gaussian Mechanism is defined as $\mathcal{M}_G (D,q,\epsilon)=q(D)+Y,$
	where $Y$ is drawn from Gaussian Distribution $\mathcal{N}(0,\sigma'^2 I_p)$ with $\sigma' \geq \frac{\sqrt{2\ln(1.25/\delta)}\Delta_2 (q)}{\epsilon}.$ 
\end{definition}

Gaussian Mechanism preserves $(\epsilon,\delta)$-differential privacy and is popular in the study of DP-SCO and DP-ERM \citep{bassily2014private,wang2018empirical,bassily2019private,wang2020differentially}.

\section{Mean estimation oracle with clipping} \label{clip discuss}
In this section, we introduce our gradient estimator based on the AClip strategy, which privately estimates the mean of a heavy-tailed distribution with a high probability guarantee. Before presenting our result, we first discuss the bias of some simple clipped methods.
\subsection{Bias of clipped methods}
We first recall the clipping strategy \footnote{the strategies in \citep{wang2020differentially} and \citep{kamath2022improved} are complex and very different from DP-SGD. They both need to estimate $n$ times of gradient per iteration instead of a mini-batch, which usually results in a high gradient complexity.} of the famous DP-SGD algorithm \citep{abadi2016deep,das2023beyond} in updating: 
%employs a simple clipping  strategy that globally clips the norm of the gradient to threshold λk\lambda_k while updating:

%A common clipping strategy is to globally clip the norm of the gradient to threshold λk\lambda_k while updating:
\begin{equation} \label{clip_dpsgd}
    x^{k+1}:= x^k - \gamma\bigg(\underbrace{\frac{1}{m}\sum_{i \in m^k}\text{clip}(\nabla f(x^k,\xi^k),\lambda)}_{\text{gradient estimator:} f^k}+z^k \bigg),\ \lambda > 0,
\end{equation}
where the clipping operates on each random sample, i.e., $\text{clip}(\nabla f(x,\xi),\lambda):= \min\{1,\frac{\lambda}{\|\nabla f(x,\xi)\|_{2}} \}\nabla f(x,\xi),\ z^k$ is a Gaussian noise whose variance is proportional to $\lambda^2$ and the amount the privacy required.

\cite{abadi2016deep} analyze the privacy loss of DP-SGD by the so-called "Moment Account" technique. They did not analyze the choice of clipping level $\lambda$. In fact, The clipping makes $f^k$ the biased gradient estimator of $\nabla f(x^k),$ and the amount of bias depends on the clipping level $\lambda,$ the higher we set $\lambda,$ the lower is the bias, and vice-versa. Later on, \cite{das2023beyond} analyze the clip bias of $\text{clip}(\nabla f(x,\xi),\lambda)$ thus one can get the bias bound of the gradient estimator $f^k$ to be $O(\frac{\sigma^4}{\lambda^2})$.

Note that the updating (\ref{clip_dpsgd}) performs the strategy of per-sample gradient clipping, i.e., one needs to do $m$ times clipping each iteration. However, the amount of bias also depends on the complexity of clipping operations- the more times we clip, the higher is the bias. This inspires us to consider a one-time clipping strategy of AClip (\ref{clip}), in which we clip once per iteration. The bias bound of AClip strategy can be summarized as follows:
%the bias can be accumulated each time we do clipping, in which we consider a one-clipping strategy (
\begin{lemma} \label{bias}
    In the setting of Assumption \ref{assump1}, let $b(\nabla f(x,\bm{\xi}),\lambda)$ denotes the non-average of AClip, i.e.,
    $$b(\nabla f(x,\bm{\xi}),\lambda):=\mathbb{E}_{\bm{\xi}}\left[\min\{1,\frac{m\lambda}{\| \sum_{i=1}^{m}\nabla f(x,\xi_{i})\|_{2}} \} \left(\sum_{i=1}^{m}\nabla f(x,\xi_{i})\right)\right].$$
    If $\|\nabla f(x)\|_2 \leq \lambda / 2$ holds, then
$$ \left\|\mathbb{E}_{\bm{\xi}}\left[b(\nabla f(x,\bm{\xi}),\lambda)\right]-m\nabla f(x)\right\|_{2}^{2} \leq \frac{4\sigma^4}{ \lambda^2}\  \text{and}\ \left\|\mathbb{E}_{\bm{\xi}}\left[\widehat{\nabla}f(x,\bm{\xi})\right]-\nabla f(x)\right\|_{2}^{2} \leq \frac{4\sigma^{4}}{m^2\lambda^2}.$$
\end{lemma}
\begin{proof}
    We derive the upper bound of bias of AClip$(\nabla f(x,\bm{\xi}),\lambda),$ denoted as $\widehat{\nabla}f(x,\bm{\xi})$: 
    \begin{equation} \label{bias bound}
        \begin{split} 
        &\left\|\mathbb{E}_{\bm{\xi}}\left[\widehat{\nabla}f(x,\bm{\xi})\right]-\nabla f(x)\right\|_{2}=\left\|\mathbb{E}_{\bm{\xi}}\left[\widehat{\nabla}f(x,\bm{\xi}) - \nabla f(x,\bm{\xi})  \right]\right\|_{2} \\
        ={}& \mathbb{E}_{\bm{\xi}}\left[\left|\|\nabla f(x,\bm{\xi})\|_2 - \lambda \right|\mathbbm{1}_{\{\|\nabla f(x,\bm{\xi})\|_2 \geq \lambda\}} \right] \\
        \leq{}& \mathbb{E}_{\bm{\xi}}\left[\left|\|\nabla f(x,\bm{\xi})\|_2 - \lambda \right|\mathbbm{1}_{\{\|\nabla f(x,\bm{\xi}) - \nabla f(x)\|_2 \geq \lambda / 2\}} \right]\\ \leq{}&\mathbb{E}_{\bm{\xi}}\left[\|\nabla f(x,\bm{\xi}) - \nabla f(x) \|_2 \mathbbm{1}_{\{\|\nabla f(x,\bm{\xi}) - \nabla f(x)\|_2 \geq \lambda / 2\}}\right] \\
        \overset{(1)}{\leq}& \sqrt{\mathbb{E}_{\bm{\xi}}\left[\|\nabla f(x,\bm{\xi}) - \nabla f(x) \|_2^2\right] \mathbb{P}\left\{\|\nabla f(x,\bm{\xi}) - \nabla f(x)\|_2 \geq \lambda / 2 \right\}}\\
        \leq{}&  \sqrt{\frac{\sigma^2}{m} \mathbb{P}\left\{\|\nabla f(x,\bm{\xi}) - \nabla f(x)\|_2 \geq \lambda / 2 \right\}}
        \leq{}  \frac{2\sigma^2}{m \lambda}
        \end{split}
    \end{equation}
    where $(1)$ uses the Cauchy–Schwarz inequality.
    The last inequality follows by the Markov's inequality:
    $$\mathbb{P}\left\{\|\nabla f(x,\bm{\xi}) - \nabla f(x)\|_2 \geq \lambda / 2 \right\} \leq \frac{\mathbb{E}_{\bm{\xi}}\|\nabla f(x,\bm{\xi}) - \nabla f(x)\|_2^2}{(\lambda / 2)^2} \leq \frac{4\sigma^2}{m\lambda^2}.$$
    Moreover, note $\widehat{\nabla}f(x,\bm{\xi}):=\min\{1,\frac{\lambda}{\|\nabla f(x,\bm{\xi})\|_{2}} \}\nabla f(x,\bm{\xi}),$ we have the following by multiplying $m$ at the both side of inequality (\ref{bias bound}):
    $$\left\|\mathbb{E}_{\bm{\xi}}\left[\min\{m,\frac{m\lambda}{\|\nabla f(x,\bm{\xi})\|_{2}} \}\nabla f(x,\bm{\xi})\right]-m\nabla f(x)\right\|_{2} \leq \frac{2\sigma^2}{ \lambda}.$$
    This is equal to 
    $$\left\|\mathbb{E}_{\bm{\xi}}\left[\min\{1,\frac{m\lambda}{\| \sum_{i=1}^{m}\nabla f(x,\xi_{i})\|_{2}} \} \left(\sum_{i=1}^{m}\nabla f(x,\xi_{i})\right)\right]-m\nabla f(x)\right\|_{2} \leq \frac{2\sigma^2}{ \lambda}.$$
    We get the desired result. Similarly, we can also bound the variance of AClip$(\nabla f(x,\bm{\xi}),\lambda):$ 
    \begin{equation} \label{variance bound} 
        \begin{split} &{}\mathbb{E}_{\bm{\xi}}\left\|\widehat{\nabla}f(x,\bm{\xi})-\mathbb{E}_{\bm{\xi}}\left[\widehat{\nabla}f(x,\bm{\xi})\right]\right\|_{2}\leq{}\mathbb{E}_{\bm{\xi}}\left\|\widehat{\nabla}f(x,\bm{\xi})-\nabla f(x)\right\|_{2}\\
        ={}&\mathbb{E}_{\bm{\xi}}\left[\left\|\frac{\lambda\nabla f(x,\bm{\xi})}{\|\nabla f(x,\bm{\xi})\|_2} - \nabla f(x)\right\|_2 \mathbbm{1}_{\{\|\nabla f(x,\bm{\xi})\|_2 \geq \lambda\}}  \right] \\
        &{}+ \mathbb{E}_{\bm{\xi}}\left[\left\|f(x,\bm{\xi}) - \nabla f(x)\right\|_2 \mathbbm{1}_{\{\|\nabla f(x,\bm{\xi})\|_2 < \lambda\}}  \right]\\
        \leq{}& \mathbb{E}_{\bm{\xi}}\left[\left(\left\|\frac{\lambda\nabla f(x,\bm{\xi})}{\|\nabla f(x,\bm{\xi})\|_2}\right\|_2 + \left\|\nabla f(x)\right\|_2\right) \mathbbm{1}_{\{\|\nabla f(x,\bm{\xi})-\nabla f(x)\|_2 \geq \lambda /2\}} \right] \\
        &{}+ \mathbb{E}_{\bm{\xi}}\left[\left\|f(x,\bm{\xi}) - \nabla f(x)\right\|_2\right] \\
        \leq{}& \sqrt{\mathbb{E}_{\bm{\xi}}\left[2\left\|\frac{\lambda\nabla f(x,\bm{\xi})}{\|\nabla f(x,\bm{\xi})\|_2}\right\|_2^2 + 2\left\|\nabla f(x)\right\|_2^2\right] \mathbb{P}\left\{\|\nabla f(x,\bm{\xi}) - \nabla f(x)\|_2 \geq \lambda / 2 \right\}} \\
        &{}+ \frac{\sigma}{\sqrt{m}}\\
        \leq{}&  \sqrt{\left(2\lambda^2 + \frac{\lambda^2}{2}\right)\frac{4\sigma^2}{m\lambda^2} }  + \frac{\sigma}{\sqrt{m}}\\
        \leq{}&  \frac{\sqrt{18}\sigma}{\sqrt{m}},
        \end{split}
    \end{equation}
    where the third inequality follows by the Cauchy–Schwarz inequality and the fact $\|a+b\|_2^2 \leq 2\|a\|_2^2 + 2\|b\|_2^2, \forall a,b \in \mathbb{R}^n.$
    It is worth mentioning here that the results similar to the above proofs have been established in Lemma F.5 of \cite{gorbunov2020stochastic}.
\end{proof}
The results show that AClip strategy has a diminishing bias with batch size. Prior works on the bias of $f^k$ in DP-SGD \citep{das2023beyond} are proportional to $\frac{m\sigma^2}{\lambda}$ and $\frac{\sigma^2}{\lambda}$ for no averaged and averaged cases, which our results improved on.
%Compared with the per-sample clipping strategy in updating ()
An intuitive perspective of our advantage is with respect to Gradient Descent (GD). GD achieves a fast linear convergence rate by generating a step towards the direction of the averaged gradient. The AClip strategy has a similar behavior to GD.
%that the average of the clipped gradient is more similar to the behaviour of original gradient descent. 
That is, at the $k$-th iteration, the final direction of $f^k$ is much closer to the "effect direction" made by GD. 

\subsection{Private mean estimator} \label{TD}
%put the estimator here and give insight
 We employ the AClip strategy in the updating of Algorithm \ref{algorithm}, which can privately estimate the mean of the heavy-tailed data with high probability. Our mean estimator is given as follows:

\begin{equation} \label{estimator}
\widetilde{\nabla} f(x,\bm{\xi}) = \widehat{\nabla} f(x,\bm{\xi}) + z, 
\end{equation}
where $\widehat{\nabla} f(x,\bm{\xi})$ refers to AClip (\ref{clip}) and $z$ is the injected Gaussian random  noise. 
A key observation is that the clipping technique also makes the $\ell_2$-sensitivity to be bounded by $\lambda$. The privacy loss can also satisfy the Moment Account technique \citep{abadi2016deep}. That is, the mean estimator (\ref{estimator})
will be ($\epsilon, \delta$)-differential privacy if we set $z\sim \mathcal{N}(0,\sigma'^{2} I^d), \sigma'=O\left(\frac{ \lambda\sqrt{\ln(1/\delta)}}{\epsilon}\right),$ and by the amplification and composition theorems \citep{Dwork2006CalibratingNT}, the whole algorithm is still differentially private. 

The estimator (\ref{estimator}) has a lower per iteration cost compared with GD-based methods and lower bias compared with DP-SGD method, but its estimation error under high probability is hard to analyze.  \cite{gorbunov2020stochastic} implicitly analyze the behavior of AClip in non-private SGD. Inspired by their work, we provide a novel analysis to decouple and embed the estimator for private convergence analysis, which differs from theirs in two crucial parts:

Firstly, Lemma \ref{bias} relies on the condition of $\|\nabla f(\cdot)\|_2$, which implies that the boundedness should be satisfied at each step in a multi-step algorithm. However, bounding the injected noise $z^k$ and its related terms needs extra work. Specifically, if we naively follow their framework, we will get an extra inner product of Gaussian noise and $x^k$ which needs to be bounded, thus we will obtain a loose bound (an extra $O\left(\frac{N^{5/2}d}{\lambda^2n^2\epsilon^2}\right)$ term) on the excess population risk. To fit our needs, we decouple their analysis to exclude all noisy inner product terms and enable to embedding a private mean estimator.

Secondly, one can usually improve the convergence rate with some extra assumptions, e.g., loss function $f$ is strongly convex. Deriving an improved rate is non-trivial since the clipping operation prevents the bias bound of a mean estimator from converging to zero. A practical method is to use the restarts technique, as considered in \cite{gorbunov2020stochastic}. However, their analysis fails to work in the private setting. In particular, if we naively follow their analysis, the privacy term in the excess risk, i.e., $O\left(\frac{N^{5/2}d}{\lambda^2n^2\epsilon^2}\right)$ where $\lambda=4L\|x^0-x^*\|_2^2, \gamma = \frac{1}{70L\ln\frac{4N}{\beta}}$ required by their analysis, cannot be upper bounded by $\frac{2}{\mu}(f(x^0)-f(x^*)),$ %{\color{blue}construct the contraction sequence} 
which ruins the condition of using the restarts technique. Thus, we reconstruct the whole analysis to handle this issue.

Before we provide our convergence analysis, we first introduce the following private mean estimation error:

%as clipping introduces additional bias even in statistical estimation. This difficulty was not solved until recently.  Gorbunov et al. \cite{gorbunov2020stochastic} implicitly performed a mean error analysis of the clipping method () in a non-private way. Although we can privately use their results via adding Gaussian random noise to achieve a high probability excess risk for DP-SCO, such naive extension is insufficient to achieve the utility bounds in Theorem and cannot extend to strongly convex setting (see section Difficulties for details). This is where the intersection of privacy and heavy-tailed data gives rise to a new technical challenge: no unbiased mean estimation oracle for this setting is known to exist \cite{kamath2022improved}. We consider decoupling the bias and noise for application and explicitly derive a mean estimation error in the following theorem.

\begin{theorem} \label{cccc}
	If $\|\nabla f(x^{k})\|_{2} \leq \frac{\lambda}{2}$ holds for all iteration $k \geq 0,$ and for all $\beta \in (0,1), N \geq 1,$ and the injected Gaussian random noise variance $\widehat{\sigma}^2I_d,$ we have
	%1. For all 0≤k≤N−1,0\leq k \leq N-1, with probability at least 1−β,1-\beta, if batch size mm is set to be max\max\left\{1,\frac{162N \sigma^2}{\lambda^2\ln^2 \frac{4}{\beta}}\right\}, then
	%\begin{equation}
	%\begin{split}	
	%&{}\sum_{t=0}^{k}\|\widetilde{\nabla} f(x^{t},\bm{\xi}^{t})-\nabla f(x^{t})\|_2 \leq \\ &{}\lambda\left(4\sqrt{N}\ln\frac{4}{\beta}+\frac{\sqrt{N}\ln \frac{4}{\beta}}{3}+\frac{2\ln^2\frac{4}{\beta}}{81}+k\widehat{\sigma}\sqrt{16d \ln\frac{4k}{\beta}}\right).
	%\end{split} 	
	%\end{equation}
	for all $0\leq k \leq N-1,$ with probability  at least $1-\beta,$ if batch size $m$ is set to be $\max\left\{1,\frac{162N^2\sigma^2}{\lambda^2(\ln \frac{4}{\beta})^2}\right\},$ then
	\begin{equation}
	\begin{split}
	&{}\sum_{t=0}^{k}\|\widetilde{\nabla} f(x^{t},\bm{\xi}^{t})-\nabla f(x^{t})\|_2 \leq \\ &{}\lambda\left(4\ln\frac{4}{\beta}+\frac{\ln \frac{4}{\beta}}{3}+\frac{2\ln^2\frac{4}{\beta}}{81N
	}+\frac{k\widehat{\sigma}\sqrt{16d \ln\frac{4k}{\beta}}}{\lambda}\right).	
	\end{split}
	\end{equation}

\end{theorem}
\begin{proof}
    The proof of the above theorem is complicated and also requires an extra definition. In order to maintain logical coherence, we defer the proof of this theory to Appendix \ref{appendix for thm1}. 
\end{proof}
The constraint on $\nabla f(x^{k+1})$ comes from Lemma \ref{bias}. It can be automatically satisfied when embedding such mean estimation error in the analyses of Theorems \ref{BC} and \ref{ee}.
It is worth mentioning that the error bound is an accumulated form with respect to $k$ rounds. However, the accumulated bias from clipping prevents the above mean estimation error from converging to zero. This is why we consider the restarting technique for a strongly convex case. Note that the corresponding privacy can be guaranteed by the Advanced Composition Theorem \citep{dwork2014algorithmic}. %Unfortunately, the advanced composition theorem is not a tighter estimation of privacy loss compared with moments accountant \cite{abadi2016deep}. However, our restarting method does not overall fit the way of moments account. We leave how to tightly estimate the privacy loss for the restarting technique as an open problem.
\section{Convergence of algorithms for DP-SCO}
In this section, we provide our main Algorithm \ref{algorithm}, AClipped-dpSGD, and establish its convergence results (i.e., excess population risk bounds) under (strongly) convex and (non)-smooth objectives. 

\begin{algorithm}[!]
	\caption{Averaged Clipping DP-SGD (AClipped-dpSGD)}
	\label{algorithm}
	\textbf{Input}: data $\{\bm{\xi}_{i}\}_{i=1}^{n}$,  starting point $x^{0} \in \mathcal{X}$, number of iterations $N$, batch size $m$, stepsize $\gamma >0$, clipping level $\lambda >0.$
	\begin{algorithmic}[1] %[1] enables line numbers
		%\STATE Let t=0t=0.
		\FOR{$k=0$ {\bfseries to} $N-1$}
		\STATE Draw a random batch $m_k$ with independently sampling probability $m/n$ and compute $\nabla f(x^{k} ,\bm{\xi}^{k})$ according to $\nabla f(x,\bm{\xi})=\frac{1}{m}\sum_{m_k}\nabla f(x,\xi_{i}).$
		%\IF {conditional}
 		\STATE Compute $\widetilde{\nabla}\; f(x^{k},\bm{\xi}^{k})=\text{AClip}(\nabla f(x^{k},\bm{\xi}^{k}),\lambda)+z^{k}$, where $z^{k}\sim \mathcal{N}(0,\widehat{\sigma}^2 I_{d}).$
		%\ELSE
		\STATE  $x^{k+1}= \text{Proj}_{\mathcal{X}}(x^{k}-\gamma\widetilde{\nabla}
		f(x^{k},\bm{\xi}^{k})$). (Note that $\text{Proj}_{\mathbb{R}^d}(x) = x$ )
		%\ENDIF
		\ENDFOR
		\STATE \textbf{return} $\bar{x}^{N}=\frac{1}{N}\sum_{k=0}^{N-1}x^{k}.$
	\end{algorithmic}
\end{algorithm}

Algorithm \ref{algorithm} is a type of gradient perturbation mechanism different from DP-SGD in clip operation. We clip the averaged gradients instead of multi-times clipping per iteration. Privacy can be guaranteed in nearly the same way: 
\begin{theorem} \label{PG}
	(Privacy guarantee). For  $\epsilon \leq c_1 \frac{Nm^2}{n^2}$ with some constant $c_1$, Algorithm $\ref{algorithm}$ is $(\epsilon, \delta)$-differentially private for any $\delta > 0$ if
	\begin{equation} \label{777}
	\widehat{\sigma} = c \frac{\lambda m \sqrt{T\ln(1/\delta)}}{n\epsilon},
	\end{equation} 
	for some constant $c$.
\end{theorem}

\begin{proof}
    Suppose $D$ and $D'$ be the neighbouring datasets drawn from a distribution $\mathcal{P}.$  Let $M$ be a sample from $[n]$ and each $i\in [n]$ is chosen independently with probability $\frac{m}{n}.$ Then the mechanism $\mathcal{M}(D)=\widehat{\nabla}f(x,D)+\mathcal{N}(0,\sigma^2\bm{I}^d),$ where  $\widehat{\nabla}f(x,D)=\frac{1}{m}\sum_{i\in M}\nabla f(x,D_i)/ \max \left(1,\frac{\|\frac{1}{m}\sum_{i\in M}\nabla f(x,D_i)\|_2}{\lambda}\right).$ 

Without loss of generality, we assume $\lambda =1$, i.e., $\|\widehat{\nabla}f(x,\cdot)\|_2 \leq 1.$
Now, consider the $k$-th query,

$$\mathcal{M}_k=\widehat{\nabla}f(x^k,D^k)+\mathcal{N}(0,\sigma^2\bm{I}^d).$$

By Theorem 2.1 in \cite{abadi2016deep}, we have $\alpha_{\mathcal{M}}(\omega) \leq \sum_{k=0}^{N-1}\alpha_{\mathcal{M}_k}(\omega),$ where the positive integer $\omega \leq \widehat{\sigma}^2\ln m/(n\widehat{\sigma}).$ 
Now we bound $\alpha_{\mathcal{M}_i}(\omega).$ Note that the $\ell_2$-sensitivity of $\widehat{\nabla}f(x,\cdot)$ is $\|\widehat{\nabla}f(x,D)-\widehat{\nabla}f(x,D')\|_2 \leq 1.$ Thus we can follow the way of Lemma 3 in \cite{abadi2016deep} with $q=\frac{m}{n},$ and  $\nabla f(x,D_n)=m\bm{e}_1$ and $\sum_{i\in M \setminus \{n\}}f(x,D_i)=\bm{0},$ by fixing $D'$ and letting $D=D'\cup \{D_n\}.$ For some constant $c_1,$ we have $$\alpha_{\mathcal{M}_k}(\omega)\leq c_1\frac{m^2\omega^2}{n^2\widehat{\sigma}^2}+O\left(\frac{m^3\omega^3}{n^3\widehat{\sigma}^3}\right).$$

After $N$ iterations, we have that for some $c_1,$
$$\alpha_{\mathcal{M}}(\omega) \leq \sum_{k=0}^{N-1}\alpha_{\mathcal{M}_k}(\omega) \leq c_1\frac{Nm^2\omega^2}{n^2\widehat{\sigma}^2}.$$ 

To be ($\epsilon,\delta$)-differentially private, using Theorem 2.2 in \cite{abadi2016deep}, it suffices that 
$$c_1\frac{Nm^2\omega^2}{n^2\widehat{\sigma}^2} \leq \frac{\omega\epsilon}{2}$$
and $$\exp{\frac{-\omega\epsilon}{2}}\leq \delta.$$

In addition, we need $$ \omega \leq \widehat{\sigma}^2\ln m/(n\widehat{\sigma}).$$

It can be verified that when $\epsilon \leq c_2\frac{Nm^2}{n^2}$ for some $c_2,$ we can set
$$\widehat{\sigma}=c\frac{m\sqrt{N\ln(1/\delta)}}{n\epsilon}$$
to satisfy all the conditions for some $c.$
Therefore, $N$-fold queries
$$\mathcal{M}_k=\widehat{\nabla}f(x^k,D^k)+\mathcal{N}(0,\sigma^2\bm{I}^d).$$
 will guarantee ($\epsilon,\delta$)-differentially private for $\epsilon \leq c_2\frac{Nm^2}{n^2}.$
\end{proof}
The constraint on $\epsilon$ remains the same as DP-SGD. Note that the variance depends on the batch size $m$ while \cite{das2023beyond} treat it as a constant hyperparameter because the analysis of $m$ related terms can be avoided by taking the expectation of randomness of DP-SGD.

%can be removed by introducing an additional factor √logT/δ\sqrt{\log T/\delta} in ˆσ\widehat{\sigma}. However, there will be a factor of √logT/δ\sqrt{\log T/\delta} in the utility bounds in Theorems ref{BC}, ref{BC},ref{BC}, ref{BC},ef{ee} and √logT/δ\sqrt{\log T/\delta} in ˆσ\widehat{\sigma} accordingly \cite{wang2018empirical}. In this case, we reserve the constraint as Wang et al. \cite{wang2017differentially} did. %{\color{red} So you didn't remove the constraint? Specify it clear here.}  
\subsection{Convex case}
In this section, we consider the cases that  $f(x)$ is convex and $L$-smooth. We provide convergence results for AClipped-dpSGD with the preferred high probability. We will show that AClipped-dpSGD is faster than DP-GD and DP-SGD in terms of running time to achieve the excess population risk bounds.
We first consider the convex case within the constraint of a bounded set, i.e., $\mathcal{X}$ is bounded:
\begin{assumption} (Bounded set). $\label{assump2}$
	The convex set $\mathcal{X}$ is defined as the following $\ell_2$-ball centered at optimal solution $x^*\ \text{of}\  (\ref{OP}):\{\|x-x^{*}\|_{2}\leq \|x^{0}-x^{*}\|_{2}\}.$
\end{assumption} 

This assumption is also required in DP-SCO \citep{wang2020differentially,kamath2022improved,wang2020differentially1}. Under this assumption, we establish the excess population risk bound of Algorithm \ref{algorithm} in the following theorem.

\begin{theorem} (Constrained convex case). \label{BC}
	Under Assumption $\ref{assump1}$ and $\ref{assump2}$, if we take $N=\Tilde{O}\left(\frac{n\epsilon}{\sqrt{d\ln\frac{1}{\delta}}}\right)^{\frac{2}{7}},$ then for all $\beta \in (0,1)$, 
	we have that after $N$ iterations of AClipped-dpSGD with
	$$\lambda= 2LR,\ m=\max\left\{1,\frac{81N^2 \sigma^2}{2L^2R^2\left(\ln \frac{4}{\beta}\right)^2}\right\},$$
	where the known constant $R \propto \|x^{0}-x^{*}\|_{2}$, and stepsize $\gamma = \frac{1}{2L\ln \frac{4}{\beta}},$
	%\begin{equation} \label{888}
	%\gamma = \frac{1}{2L\ln \frac{4}{\beta}}.
	%\end{equation} 
	with probability at least $1-\beta$, the following holds
	%\begin{equation} \label{T2_1}
	%f(\bar{x}^{N})-f(x^{*}) \leq  O\left(\frac{LR_0^2\ln\frac{4}{\beta}}{\sqrt{N}}+\frac{ N^{\frac{3}{2}} \sigma^2\sqrt{d\ln\frac{4N}{\beta}\ln
	%		\frac{1}{\delta}}}{ n\epsilon L\ln^2\frac{4}{\beta}}\right).
	%\end{equation}
	%Moreover, if taking N=O(‖x0−x∗‖2nϵln24β√dln1δ)12,N=O\left(\frac{\|x^0-x^*\|_2n\epsilon\ln^2\frac{4}{\beta}}{\sqrt{d\ln\frac{1}{\delta}}}\right)^{\frac{1}{2}}, we have that with probability at least 1−β1-\beta,
	\begin{equation} \label{CB}
	f(\bar{x}^{N})-f(x^{*}) \leq  \Tilde{O}\left(\frac{ \left(d\ln\frac{1}{\delta}\right)^{\frac{1}{7}}\sqrt{\ln\frac{(n\epsilon)^{2}}{\beta d }}}{(n\epsilon)^{\frac{2}{7}}}\right),
	\end{equation}
	when $n \geq \Tilde{\Omega}\big(\frac{\sqrt{d}}{\epsilon}\big).$ The total gradient complexity is $\Tilde{O}\left(\max \left\{n^{\frac{2}{7}}d^{\frac{6}{7}},n^{\frac{6}{7}}d^{\frac{4}{7}}\right\}\right),$ and the Big-$\Tilde{O}$ notation here omits other logarithmic factors and the terms $\sigma,\ln\frac{1}{\beta}, L.$ 
	
\end{theorem}
\begin{proof}
    The proof of the above theorem is complicated. In order to maintain logical coherence, we defer the proof of this theory to Appendix \ref{appendix for thm3}.
\end{proof}
\begin{remark} \label{CRK}
    The excess population risk is bounded by $\Tilde{O}\left(\frac{d^{1/7}\sqrt{\log (n /\beta d^{2})}}{(n\epsilon)^{2/7}}\right)$  with the gradient complexity $\Tilde{O}\left(n^{\frac{6}{7}}d^{\frac{4}{7}}\right)$ for convex objectives within a bounded set with high probability. In comparison with similar settings, prior work of \cite{wang2020differentially} provide an excess risk $\Tilde{O}\left(\frac{d^{2/3}\sqrt{\log (n /\beta d^{2})}}{(n\epsilon^2)^{1/3}}\right)$  with the gradient complexity $\Tilde{O}(n^{\frac{4}{3}}d^{\frac{1}{3}}),$ which our result significantly improves on. \cite{kamath2022improved} achieve an improved rate, they need more running time, i.e., $\Tilde{O}(n^2d^{-\frac{3}{2}})$ gradient complexity.
\end{remark}

The constraint of the bounded set makes the overall optimization bias easy to analyze: one can bound the distance from current point $x^k$ and $x^*$ (i.e., $\|x^k-x^*\|_2$) by the diameter of the constraint set, which is one part of overall optimization bias \citep{das2023beyond}. However, analyzing the \emph{unconstrained} (i.e., $\mathcal{X} = \mathbb{R}^d$) convex case is practical, hard, and necessary.
As mentioned in Section \ref{TD}, a necessary way to extend our method to the strongly convex case is the restarts technique. However, the bounded set does not support that the sequence $\|x^k-x^*\|_2$ is contracted, which forces us to establish the convergence analysis with the unconstrained case.
% without Assumption ref{assump2} hinders this technique, as it is difficult to guarantee that the sequence  {‖xk−x∗‖2}Nk=0\{\|x^k-x^*\|_2\}_{k=0}^{N} is contracted.
%Therefore, we need to consider a new analysis in the convex case without a bounded domain. 
This has not been analyzed in previous private works, in terms of high probability bound. Now, we provide the following theorem to establish the excess population risk bound with the unconstrained case. 

\begin{theorem} \label{ee}
	(Unconstrained convex case). Under Assumption $\ref{assump1}$ and $\mathcal{X}=\mathbb{R}^d$, if we take $N=\Tilde{O}\left(\frac{n\epsilon}{\sqrt{d\ln\frac{1}{\delta}}}\right)^{\frac{2}{7}}$ in Algorithm \ref{algorithm}, then for all $\beta \in (0,1)$, 
	we have that after $N$ iterations of AClipped-dpSGD with
	$$\lambda= 4LR+2\sqrt{D},\  m=\max\left\{1,\frac{162N^2 \sigma^2}{\lambda^2(\ln \frac{4N}{\beta})^2}\right\},$$
	where the constants $R \propto \|x^{0}-x^{*}\|_{2}$, $D \propto \frac{\gamma L N^3 \sigma^2\sqrt{dN\ln\frac{4N^2}{\beta}\ln
			\frac{1}{\delta}}}{ n\epsilon\ln^2\frac{4N}{\beta}},$ and stepsize $\gamma = \frac{1}{24L\ln \frac{4N}{\beta}},$
	%\begin{equation} \label{888}
	%	\gamma = \frac{1}{24L\ln \frac{4N}{\beta}}.
	%	\end{equation} 
	with probability at least $1-\beta$, the following holds
	%\begin{equation} \label{T3_1}
	%f(\bar{x}^{N})-f(x^{*}) \leq  %O\left(\frac{LR_0^2\ln\frac{4}{\beta}}{N}+\frac{ N^{\frac{5}{2}} \sigma^2\sqrt{d
	%		\ln\frac{4N}{\beta}\ln
	%		\frac{1}{\delta}}}{ n\epsilon L^2\ln^2\frac{4}{\beta}}\right).
	%f(\bar{x}^{N})-f(x^{*}) \leq  O(\frac{LR_0^2\ln\frac{4N}{\beta}}{N}+\frac{ N^{\frac{5}{2}} \sigma^2\sqrt{d\ln\frac{4N^2}{\beta}\ln\frac{1}{\delta}}}{ n\epsilon L^2\ln^2\frac{4}{\beta}}).	
%	\end{equation}
	%Moreover, if taking  N=O(‖x0−x∗‖2nϵln24β√dln1δ)27,N=O\left(\frac{\|x^0-x^*\|_2n\epsilon\ln^2\frac{4}{\beta}}{\sqrt{d\ln\frac{1}{\delta}}}\right)^{\frac{2}{7}}, we have that with probability at least 1-N\beta,
	\begin{equation} \label{CU}
	f(\bar{x}^{N})-f(x^{*}) \leq  \Tilde{O}\left(\frac{ \left(d\ln\frac{1}{\delta}\right)^{\frac{1}{7}}\ln \frac{(n\epsilon)^{2}}{\beta d}}{(n\epsilon)^{\frac{2}{7}}}\right),
	%$$f(\bar{x}^{N})-f(x^{*}) \leq  O(\frac{ (d\ln\frac{1}{\delta})^{\frac{1}{7}}\sqrt{\ln\frac{(n\epsilon)^{\frac{4}{7}}\ln\frac{4}{\beta}}{\beta(d\ln\frac{1}{\delta})^{\frac{2}{7}}}}}{(n\epsilon)^{\frac{2}{7}}(\ln\frac{4}{\beta})^{\frac{4}{7}}})
	\end{equation}
	when $n \geq \Tilde{\Omega}\big(\frac{\sqrt{d}}{\epsilon}\big).$ The total gradient complexity is $\Tilde{O}\left(\max \left\{n^{\frac{2}{7}}d^{\frac{6}{7}},n^{\frac{6}{7}}d^{\frac{4}{7}}\right\}\right)$, and the Big-$\Tilde{O}$ notation here omits other logarithmic factors and the terms $\sigma,\ln\frac{1}{\beta}, L.$ %The total running time is O(n67d47).O\left(n^{\frac{6}{7}}d^{\frac{4}{7}}\right).
	
\end{theorem}
\begin{proof}
    Unlike the constraint case, the error sequence $\{\|x^k - x^*\|_2\}_k$ is bounded by the initial error. For the unconstrained case, we derive similar results by constructing a sequence of events bounded by probability. The proof of the above theorem is more complicated. In order to maintain logical coherence, we defer the proof of this theory to Appendix \ref{appendix for thm4}.
\end{proof}
\begin{remark}
	The convergence analysis is more complicated without the bounded set assumption. But the result is only slightly worse than Theorem \ref{BC}, up to the square root of logarithmic factor $\sqrt{\ln{(\cdot)}}$. This is because we are able to bound $\|x^k-x^*\|_2, \forall k \in 0,\cdots,N$ with high probability. 
	Note that this theorem also enjoys the similar advantages of Theorem \ref{BC} over the GD-based method of \cite{wang2020differentially,kamath2022improved}, in terms of excess population risk bound and total gradient complexity. In comparison with the analysis of unconstrained convex case for DP-SGD in \cite{das2023beyond}, we provide the high probability bound and principled tuning strategy of batch size $m$ which they do not. Even with the minimum batch size, i.e., $m=1,$ their complexity becomes $\Tilde{O}(n^2)$, which our result significantly improves on. 
\end{remark}

\subsection{Strongly convex case}
 Based on the analysis of Theorem \ref{ee}, we can consider, in this section, the restarts technique under the case that $f(x)$ is additionally $\mu$-strongly convex for the unconstrained smooth case. For this case, we modify Algorithm \ref{algorithm} to a restarted version called Restarted Averaged Clipping DP-SGD (R-AClipped-dpSGD), as shown in Algorithm \ref{algorithm1}. In each iteration, R-AClipped-dpSGD runs AClipped-dpSGD for $N_0$ iterations from the current point $\widehat{x}^t$ and uses its output $\widehat{x}^{t+1}$ as the starting point for the next iteration. This strategy is known as the restarting technique \citep{gorbunov2020stochastic,juditsky2011first,ghadimi2013optimal,dvurechensky2016stochastic}.

\begin{algorithm}[!]
	\caption{Restarted Averaged Clipping DP-SGD (R-AClipped-dpSGD)}
	\label{algorithm1}
	\textbf{Input}: data $\{\bm{\xi}_{i}\}_{i=1}^{n}$,  starting point $x^{0}$, number of iterations $N_0$ of AClipped-dpSGD, number $\tau$ of AClipped-dpSGD runs, batch size $\{m_{t}\}_{t=0}^{\tau-1}.$
	\begin{algorithmic}[1] %[1] enables line numbers
		%\STATE Let t=0t=0.
		\FOR{$t=0$ {\bfseries to} $\tau-1$}
		\STATE 		 Run AClipped-dpSGD (Algorithm \ref{algorithm}) for $N_0$ iterations with  constant batch size $m^t$, stepsize $\gamma$,  Gaussian random noise $z^t \sim \mathcal{N}(0,\bm{\widehat{\sigma}_t^2 I_{d}})$, $\widehat{\sigma}_t = \frac{\lambda_t m_t \sqrt{N_0\ln(1/\widehat{\delta})}}{n\widehat{\epsilon}}$ and starting point $\widehat{x}^t$. Define the output of AClipped-dpSGD by $\widehat{x}^{t+1}.$ 
		\ENDFOR
		\STATE \textbf{return} $\widehat{x}^{\tau}.$
	\end{algorithmic}
\end{algorithm} 
Privacy can be guaranteed by the Advanced Composition Theorem \citep{dwork2014algorithmic}:
\begin{theorem} \label{pg2}
	(Privacy guarantee). Algorithm $\ref{algorithm1}$ is overall $(\epsilon,\delta)$-differentially private after $\tau$ runs, where each run of AClipped-dpSGD is $\Big(\widehat{\epsilon},\widehat{\delta}\Big)$-differentially private with $\widehat{\epsilon}=\frac{\epsilon}{2\sqrt{2\tau\ln\frac{2}{\delta}}}, \widehat{\delta}=\frac{\delta}{2\tau}.$
\end{theorem}

\begin{proof}
    Let each run of Algorithm \ref{algorithm} to be ($\widehat{\epsilon},\widehat{\delta}$)-DP with $\widehat{\epsilon}=\frac{\epsilon}{2\sqrt{2\tau\ln\frac{2}{\delta}}}, \widehat{\delta}=\frac{\delta}{2\tau}.$
    Then the ($\epsilon,\delta$)-DP of Algorithm \ref{algorithm1} can be guaranteed by the following Advanced Composition Theorem \citep{dwork2014algorithmic}:
    
    For all $\epsilon, \delta, \delta'\geq 0$, the class of ($\epsilon,\delta$)-differentially private mechanisms satisﬁes ($\epsilon',k\delta+\delta'$)-differential
	privacy under $k$-fold adaptive composition for:
	$\epsilon'=\sqrt{2k\ln\frac{1}{\delta'}\epsilon}+k\epsilon(e^{\epsilon}-1).$
	Typically,  it suffices that each mechanism is ($\epsilon,\delta$)-differentially private, where
	$\epsilon=\frac{\epsilon'}{2\sqrt{2k\ln\frac{1}{\delta'}}}.$
\end{proof}
Finally, we establish the excess population risk bound of Algorithm \ref{algorithm1} in the following theorem.

\begin{theorem}\label{h}
	(Unconstrained strongly convex case). Assume that the function $f$ is $\mu$-strongly convex and $\mathcal{X}=\mathbb{R}^d$. If we choose $\beta \in (0,1), \tau$ and $N_0,\hat{c} \geq 1, $ such that $\frac{N_0}{\ln\frac{4N_0}{\beta}}\geq \frac{768\hat{c}^2L}{\mu},$ together with 
	$$\lambda_t= 4LR^t+2\sqrt{D},\  m_t=\max\left\{1,\frac{162N_0^2 \sigma^2}{\lambda_t^2(\ln \frac{4N_0}{\beta})^2}\right\},$$
	where $R^t \propto \|\widehat{x}^{t}-x^{*}\|_{2}$, $D \propto \frac{\gamma L N_0^3 \sigma^2\sqrt{dN_0\ln\frac{4N_0^2}{\beta}\ln
			\frac{1}{\widehat{\delta}}}}{ n\widehat{\epsilon}\ln^2\frac{4N_0}{\beta}},$ and stepsize $\gamma = \frac{1}{24L\ln \frac{4N_0}{\beta}},$
	we have that with probability at least $1-\tau\beta,$
	$$f(\widehat{x}^{\tau})-f(x^*) \leq \frac{1}{2^{\tau}}(f(x^0) - f(x^*))+O\left(\frac{ N_0^{\frac{5}{2}} \sigma^2\sqrt{d\ln\frac{4N_0^2}{\beta}\ln
			\frac{1}{\widehat{\delta}}}}{ n\widehat{\epsilon} L\ln^2\frac{4N_0}{\beta}}\right).$$
	Moreover, if taking $\tau = O\left(\frac{L}{\mu}\log n\right),$ %\tau = O(\frac{L}{\mu}\log n),\tau = O(\frac{L}{\mu}\log n),
	and $N_0=O\left(\frac{L}{\mu}\ln \frac{L^2}{(\mu\beta)^2}\right),$
	%N0=O(LμlnL2lognμ2β),N_0=O(\frac{L}{\mu}\ln \frac{L^2\log n}{\mu^2\beta}), 
	with probability at least $1-\tau\beta,$ the output $\widehat{x}^{\tau}$ satisfies
	\begin{equation}
	f(\widehat{x}^{\tau})-f(x^*) \leq  \Tilde{O}\left(\frac{d^{\frac{1}{2}}L^2 \sqrt{\log n}\ln\frac{L}{\mu\beta} \ln\frac{L}{\delta \mu}}{ \mu^3n\epsilon}\right).
	\end{equation}
	The total gradient complexity is $\Tilde{O}\left(\max\left\{d(\frac{L}{\mu})^2,nd^{\frac{1}{2}}\right\}\right),$ and the Big-$\Tilde{O}$ notation here omits other logarithmic factors and the terms $\sigma,\ln \frac{1}{\beta}.$ 
\end{theorem}
\begin{proof}
    The proof mainly depends on the Theorem \ref{ee}.  In order to maintain logical coherence, we defer the proof of this theory to Appendix \ref{appendix for thm6}.
\end{proof}
\begin{remark} \label{SCRK}
	R-AClipped-dpSGD yields an  $\Tilde{O}\big(\frac{d^{\frac{1}{2}}L^2}{\mu^3n\epsilon}\big)$ rate on the excess population risk, which is much faster than the high probability result $\Tilde{O}\big(\frac{d^{2}L^4}{\mu^3n\epsilon^2}\big)$ in \cite{wang2020differentially} and also nearly matches the expectation form bound $\Tilde{O}\big(\frac{d^{\frac{1}{2}}L}{\mu^2n\epsilon}\big)$ in \cite{kamath2022improved}. Moreover, R-AClipped-dpSGD does not require assuming the constraint of bounded set $\mathcal{X}.$ It also has a faster running time in terms of gradient complexity. 
 %The total gradient complexity \TildeO(max{d(Lμ)2,nd12})\Tilde{O}\left(\max\left\{d(\frac{L}{\mu})^2,nd^{\frac{1}{2}}\right\}\right) is better than GD based method \TildeO(ndLμ)\Tilde{O}\big(nd\frac{L}{\mu}\big) from Theorem 7 in \cite{wang2020differentially}. (Note that the above gradient complexity all require n≥Lμn \geq \frac{L}{\mu}.) 
\end{remark}

\subsection{Non-smooth case}
In this section, we extend our unconstrained convex analysis of AClipped-dpSGD to the non-smooth case. Typically, we consider the gradient of loss function $f$ satisfies H$\ddot{\text{o}}$lder continuity and establish a new excess risk bound. We first introduce the H$\ddot{\text{o}}$lder continuity.

\textbf{Level of smoothness.} We assume that the function $f$ has $(\nu,M_\nu)$-H$\ddot{\text{o}}$lder continuous gradients on a compact set $Q \in \mathbb{R}^n$ for some $\nu \in [0,1],$ $M_\nu > 0$ meaning that
\begin{equation} \label{nonsmooth}
    \|\nabla f(x)-\nabla f(y)\|_2 \leq M_\nu\|x-y\|_2^{\nu}, \forall x,y 
    \in Q.
\end{equation}
H$\ddot{\text{o}}$lder continuous covers the $L$-smoothness of $f$. When $\nu=1,$ inequality (\ref{nonsmooth}) is equivalent to $M_1$-smoothness of $f,$ and when $\nu=0,$ $f$ is non-smooth and has a uniformly boundedness of $\nabla f(x).$ Although the boundedness of $\nabla f(x)$ is not preferred with heavy-tailed data, one can assume the boundedness of subgradients of $f$ for this special case \citep{Gorbunov2021near,lowy2023private}. Example with $\nu \in (0,1)$ can be founded in \cite{chaux2007variational}. Moreover, the inequality (\ref{nonsmooth}) only needs to hold in a compact set $Q.$ %i.e., we do not need it to hold on \mathbb{R}^n.\mathbb{R}^n. 
Explicitly, as we show in the following result, $Q$ should contain a ball that centers at $x^*$ with radius $7R: R \propto \|x^0-x^*\|_2.$    

\begin{theorem} (Unconstrained non-smooth case). \label{NONS}
	Assuming the function $f$ is convex and $\mathcal{X}=\mathbb{R}^d,$  and its gradients satisfy (\ref{nonsmooth}) with $\nu \in [0,1],$ $M_\nu$ on $Q=B_{7R}=\{x \in \mathbb{R}^n:\|x-x^*\|_2\leq 7R \},$ 
	then for all $\beta \in (0,1)$, $\alpha > 0,$ $N=\Tilde{O}\left(\frac{n\epsilon}{\sqrt{d\ln\frac{1}{\delta}}}\right)^{\frac{2}{7}},$
	we have that after $N$ iterations of Algorithm \ref{algorithm} with
	$$\lambda=2M_\nu C^\nu R^\nu,\ m=\max\left\{1,\frac{27N\sigma^2}{\lambda^2\ln\frac{8}{\beta}}\right\}
	,$$ and stepsize $$\gamma = \min\left\{\frac{\alpha^{\frac{1-\nu}{1+\nu}}}{8M_\nu^{\frac{2}{1+\nu}}},\frac{R}{\sqrt{2N}\alpha^{\frac{\nu}{1+\nu}}M_\nu^{\frac{1}{1+\nu}}},\frac{R}{2\lambda\ln\frac{8}{\beta}},\frac{\lambda R}{2DN}\right\},$$
	%\begin{equation} \label{888}
	%\gamma = \frac{1}{2L\ln \frac{4}{\beta}}.
	%\end{equation} 
	where $D  \propto \frac{N^{1.5}\sigma^2\sqrt{d\ln\frac{8N}{\beta}\ln\frac{1}{\delta}}}{n\epsilon\ln\frac{8}{\beta}}$, $C=7,$
	with probability at least $1-N\beta$, the following holds
    \begin{equation}\label{nsr}
    \begin{split}
        &f(\bar{x}^{N})-f(x^{*})\leq{}\Tilde{O}\Biggl(\max\Biggl\{\frac{M_\nu^{\frac{2}{1+\nu}}(d\ln\frac{1}{\delta})^{\frac{1}{7}}}{\alpha^{\frac{1-\nu}{1+\nu}}(n\epsilon)^{\frac{2}{7}}},\frac{M_\nu (d\ln\frac{1}{\delta})^{\frac{1}{7}}}{(n\epsilon)^{\frac{2}{7}}}\\
        &\frac{M_\nu^{\frac{1}{1+\nu}}\alpha^{\frac{\nu}{1+\nu}}(d\ln\frac{1}{\delta})^{\frac{1}{14}}}{(n\epsilon)^{\frac{1}{7}}},\frac{(d\ln\frac{1}{\delta})^{\frac{2}{7}}\sqrt{\ln\frac{(n\epsilon)^2}{\beta d}}}{(n\epsilon)^{\frac{4}{7}} M_\nu}\Biggl\}\Biggl)
    \end{split}
    \end{equation}
	when $n \geq \Tilde{\Omega}\big(\frac{\sqrt{d}}{\epsilon}\big).$ The total gradient complexity is $\Tilde{O}\left(\max \left\{n^{\frac{2}{7}}d^{\frac{6}{7}},n^{\frac{4}{7}}d^{\frac{5}{7}}\right\}\right),$ and the Big-$\Tilde{O}$ notation here omits other logarithmic factors and the terms $\sigma,\alpha, R.$
	
\end{theorem}
\begin{proof} 
    The proof is different and more complicated than other theorems when considering the H$\ddot{\text{o}}$lder continuous gradient. We defer it to Appendix \ref{appendix for thm7} to keep logical coherence.
\end{proof}
\begin{remark}
    This result has an important feature that H$\ddot{\text{o}}$lder continuity is required only on the domain of $B_{7R_0}$ ball centered at the optimal solution. To the best of our knowledge, this is the first high-probability result on private convex and non-smooth problems for heavy-tailed data. For the non-smooth case, i.e., $\nu=0,$ prior results in \cite{lowy2023private}, where authors establish an optimal excess risk up to a logarithmic factor, i.e., $\frac{1}{\sqrt{n}}+\left(\frac{\sqrt{d\ln(n)}}{m\epsilon}\right)^{\frac{1}{2}}.$ Their result is in the form of its expectation and also requires the convex problem under the constraint of a bounded set, which our result does not depend on. Moreover, since the proposed algorithm is a two-loop structure, it needs more running time to attain the excess risk, i.e., $\Tilde{O}(n^2d)$ (sub) gradient complexity.
    
\end{remark}

\section{Experiments}
In this section, we conduct numerical experiments to corroborate the efficacy of AClipped-dpSGD (in terms of Table \ref{tab1}). We apply our method and other private methods on the base machine learning tasks, i.e., solving ridge regression (RR) with $f(x,\xi) = (\langle x, \xi \rangle - y )^2$ objectives and logistic regression (LR) task with $f(x,\xi) = \log(1+\text{exp}(1+y\langle x, \xi \rangle ))$ objectives. We use the non-private AClipped-dpSGD method as our baseline method.

We test our proposed methods on synthetic data. Specifically, we generate the synthetic datasets from three heavy-tailed distributions, including \textit{Student's t-distribution}, \textit{Laplace distribution}, and \textit{Chi-squared distribution}. Each dataset has a size of 100,000 and each dataset is generated by the following models: 
\begin{equation*}
    \textit{RR}: \ y_i:=\langle x^*,\xi_i \rangle + e_i \ \text{and}\  \textit{LR}:\
     y_i = 2 * \text{sign}\left[\frac{1}{\exp{(\langle x^*,\xi_i \rangle+e_i)}}-\frac{1}{2}\right]-1,
\end{equation*}
where $x_i \in \mathbb{R}^{10}$ and $y_i \in \mathbb{R}.$ We denote $z_i$ as the zero mean noise that is generated from the above three heavy-tailed distributions. As an example, we generate the noise $z_i$ from \textit{Student's t-distribution}, i.e., $\mathbb{P}(z_i = t) = \frac{\Gamma(\frac{v+1}{2})}{\sqrt{\pi v}\Gamma(\frac{v}{2})}(1+\frac{t^2}{v})^{-(v+1)/2}$ where $v$ is the number of degrees of freedom and $\Gamma$ is the \textit{gamma function}, and then let $e_i = z_i - \mathbb{E}[z_i].$

We also test our methods on real-world data. To be consistent with the heavy-tailed setting, the datasets we chose from the LIBSVM library \citep{chang2011libsvm} are  \emph{Adult}, \emph{Diabetes}, on which the SGD is struggling especially on the \emph{Diabetes} data \citep{gorbunov2020stochastic}.  We compare AClipped-dpSGD with DP-GD of \cite{wang2020differentially}, DP-SGD of \cite{das2023beyond}. For a fair comparison, the batch size $m$ is fixed to $\sqrt{n}$ for stochastic methods and other model parameters are tuned according to the corresponding theoretical strategies. We consider four privacy levels - $(0.5, n^{-1})$-DP, $(0.75, n^{-1})$-DP, $(1, n^{-1})$-DP, and $(2, n^{-1})$-DP, with step size in the range $10^{-4} \sim 10^{-2}.$ More descriptions and results are presented in Appendix \ref{secB1}. 

\begin{figure*}[t]
	\centering
    %\captionsetup[subfigure]{labelformat=empty}
	\subfigure{
		\includegraphics[width=0.23\textwidth]{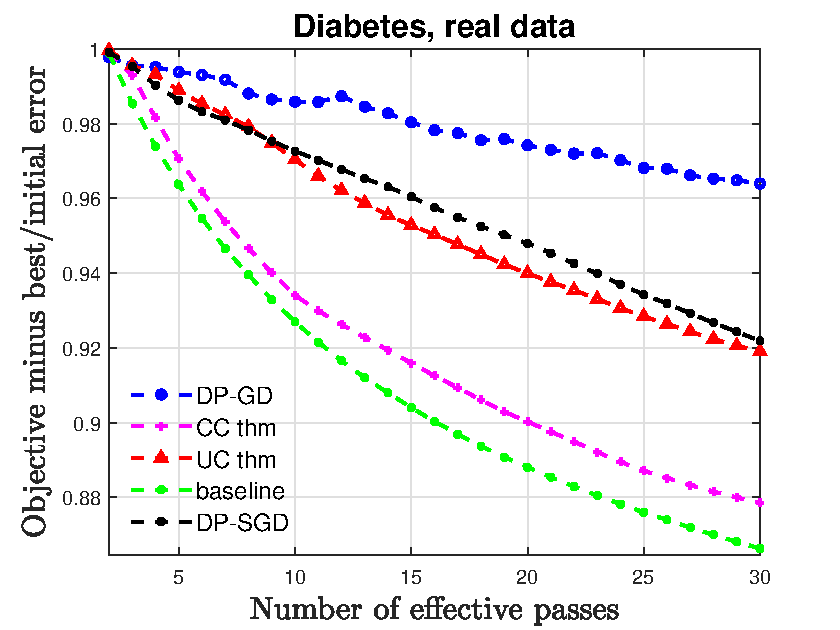}}
	\subfigure{
		\includegraphics[width=0.23\textwidth]{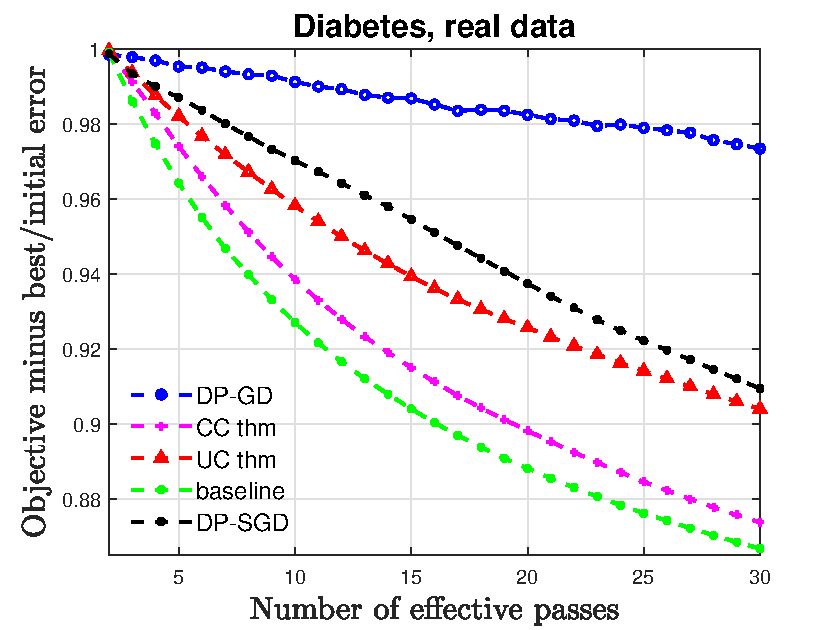}}
	\subfigure{
		\includegraphics[width=0.23\textwidth]{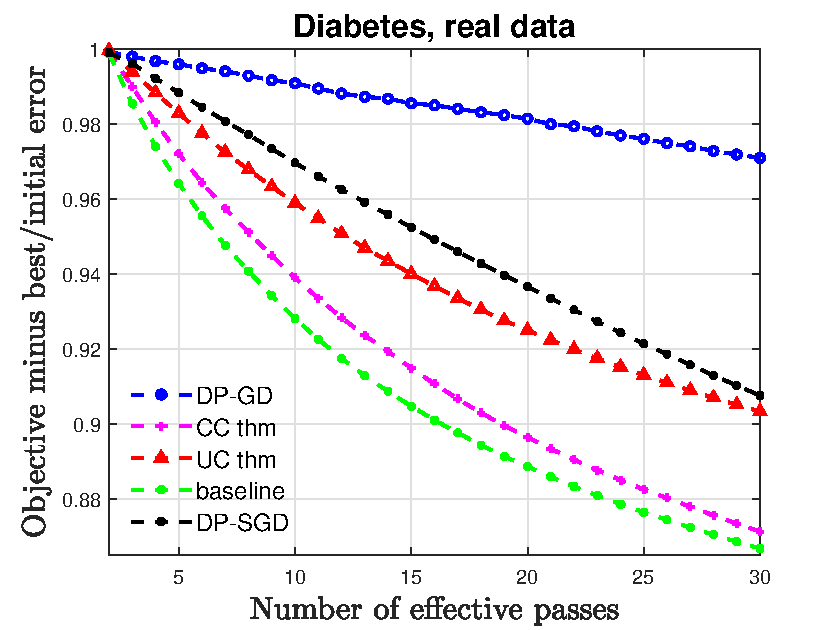}}
	\subfigure{
		\includegraphics[width=0.23\textwidth]{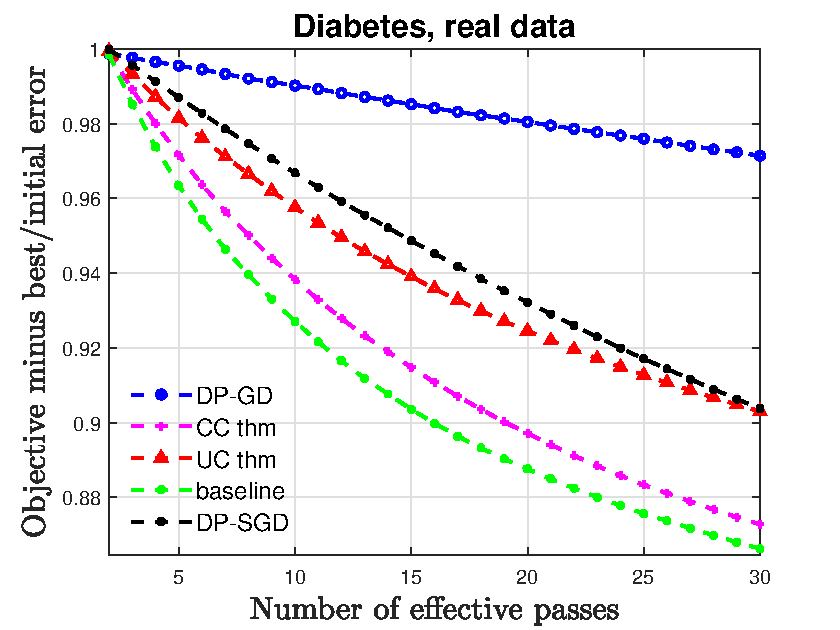}}
  \newline
   \setcounter{subfigure}{0}
  	\subfigure[$\epsilon = 0.5$]{
		\includegraphics[width=0.23\textwidth]{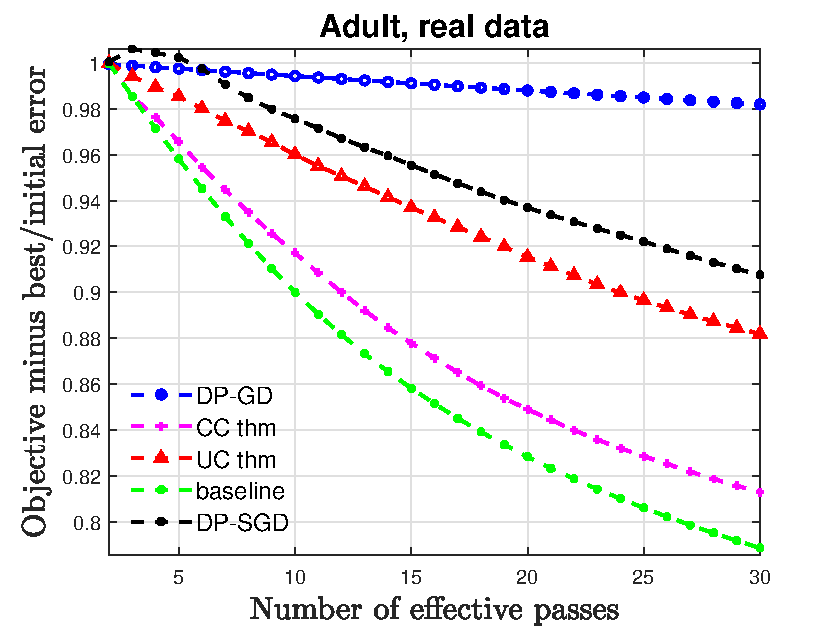}}
	\subfigure[$\epsilon = 0.75$]{
		\includegraphics[width=0.23\textwidth]{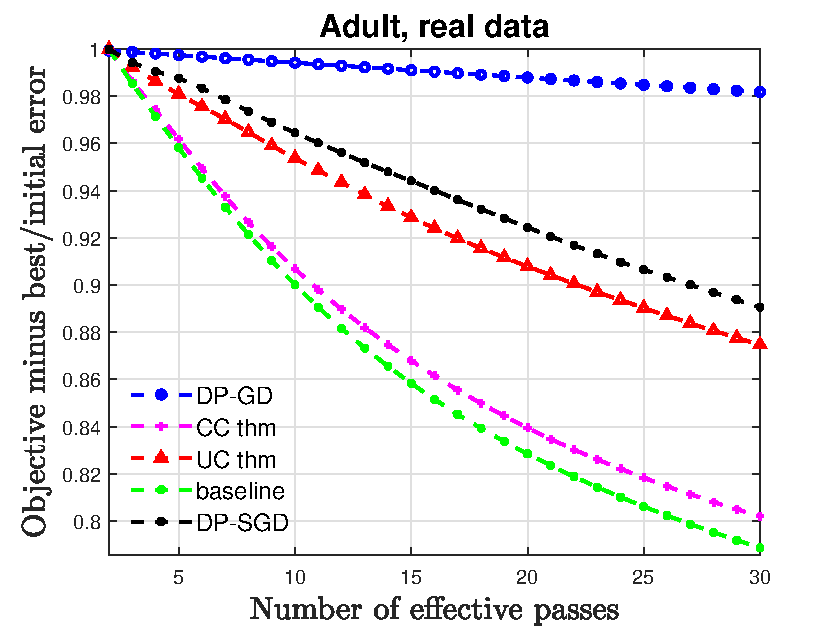}}
	\subfigure[$\epsilon = 1.0$]{
		\includegraphics[width=0.23\textwidth]{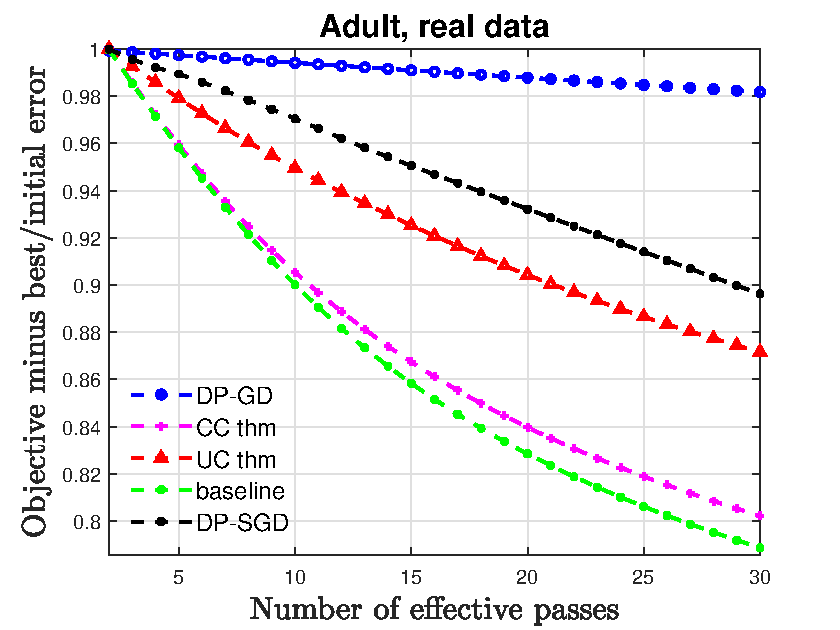}}
	\subfigure[$\epsilon = 2.0$]{
		\includegraphics[width=0.23\textwidth]{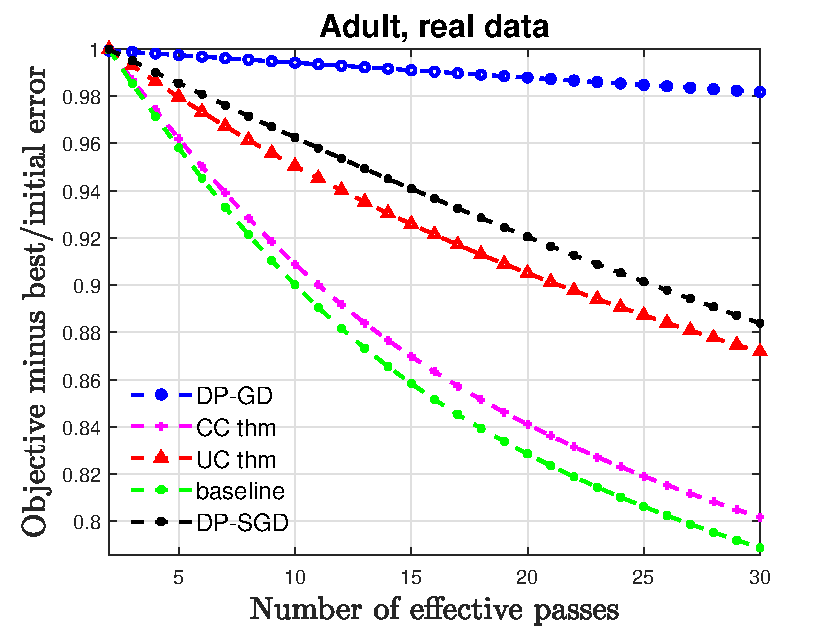}}
	\caption{Trajectories of the logistic regression model for the real-world data. The top and bottom rows correspond to the \textit{Diabetes} and \textit{Adult} datasets, respectively.}
    \label{fig1}
\end{figure*}

\begin{figure*}[t]
	\centering
    %\captionsetup[subfigure]{labelformat=empty}
	\subfigure[$\epsilon = 2$]{
		\includegraphics[width=0.31\textwidth]{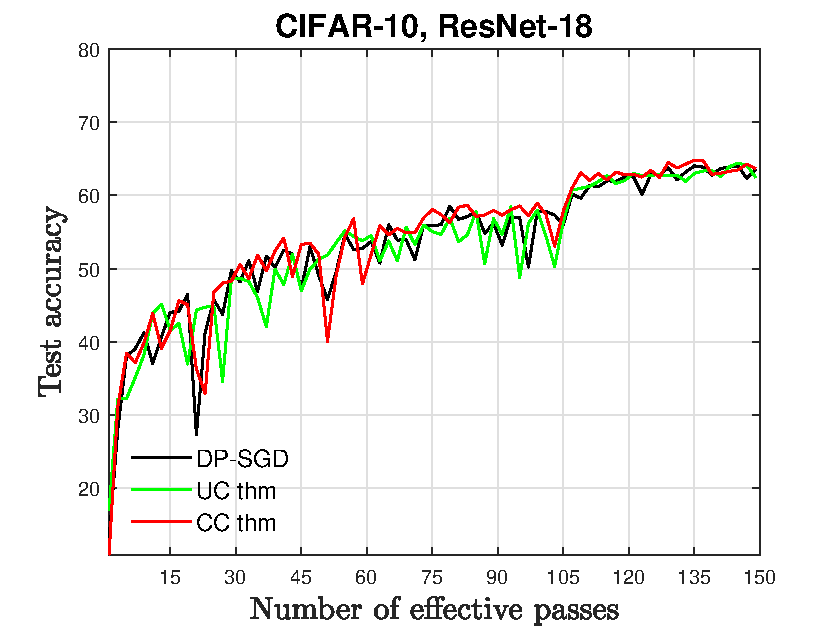}}
	\subfigure[$\epsilon = 4$]{
		\includegraphics[width=0.31\textwidth]{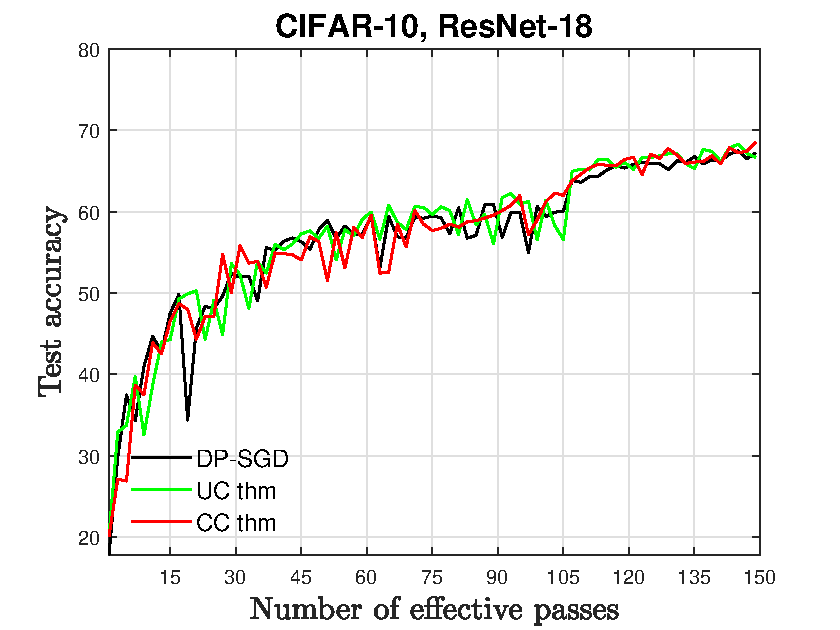}}
	\subfigure[$\epsilon = 6$]{
		\includegraphics[width=0.31\textwidth]{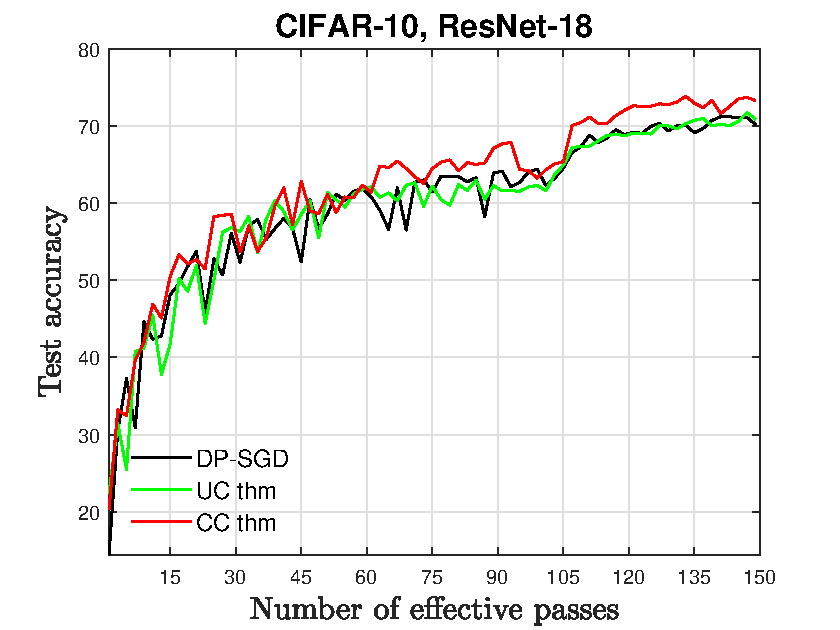}}
	\caption{Trajectories of test accuracy for different noise levels on \textit{CIFAR-10} dataset. We set the $\delta$ to $10^{-5},$ clipping level $\lambda$ to 30, and batch size $m$ to 500. All experiments are conducted on the NVIDIA RTX A6000 platform.}
    \label{figdp}
\end{figure*}

In Figure \ref{fig1} and \ref{fig2}, we plot the behavior of the loss function averaged over 300 times. Each time we 
run 30 and 400 epochs through the real-world datasets and synthetic datasets, respectively. We define the alias for AClipped-dpSGD, DP-SGD, DP-GD, baseline methods as the CC them (Theorem \ref{BC}), UC thm (Theorem \ref{ee}), DP-SGD, DP-GD, and baseline, respectively. Table \ref{table2} illuminates the corresponding estimation errors (divided
by initial error) and running times after being averaged over 300 times.

In addition to numerical tests, we further conduct experiments using a popular deep learning benchmark, that is, the \textit{CIFAR-10} dataset. This dataset consists of color images of the shape 32x32x3, categorized into 10 classes, with 50,000 training examples and 10,000 test examples. In our study, we measure the performances of DP-SGD and the proposed AClipped-dpSGD method by training the ResNet-18 model (without use of additional data or a pre-trained model) and reporting the trajectory of test accuracy. To ensure a fair comparison, we consider three private settings: $(2, 10^{-5})$-DP, $(4, 10^{-5})$-DP and $(6, 10^{-5})$-DP, with fixed batch size $m$ to 500, consistent with specifications in \cite{das2023beyond}. We tune the stepsize within the range $0.001 \sim 0.01.$ In Fig \ref{figdp}, we provide the test accuracy averaged over 5 independent runs and across 150 epochs. In Table \ref{tabledp}, we calculate the mean and standard deviation of the last 5 epochs across 5 independent runs. 

\begin{table*}[t]
    \caption{Summary of experimental results of private methods on the logistic model. Bold values mark the results of our methods.}
	\centering	
    \resizebox{\textwidth}{!}{
	\begin{tabular}{ccccccc|cccc}
		\toprule
		\multirow{2}{*}{Model} &\multirow{2}{*}{Data} &\multirow{2}{*}{$\epsilon$} & \multicolumn{4}{c}{Estimation Error}&\multicolumn{4}{|c}{Runtime (wall clock time): $s$}\\
		\cmidrule{4-7} \cmidrule{8-11}
		
		& & & CC & UC &DP-SGD & DP-GD & CC & UC &DP-SGD& DP-GD \\
		\hline
		
		%Teacher &4.12G	&0.87841	&0.89512	&0.9014	&0.83587	&0.85736	&0.86297\\ 
		%\hline
		
		%0.25x-----------------
		\multirow{20}{*}{LR} &\multirow{4}{*}{\emph{Diabetes}}& 0.5  &
		\textbf{0.8772} & \textbf{0.9174} &	0.9195 &0.9724	 &\textbf{6.18E-4} & \textbf{5.87E-4} &7.14E-4&5.32E-3 \\
		%\cmidrule{2-4} \cmidrule{5-6} \cmidrule{7-8}

        \multirow{20}{*}{} &\multirow{4}{*}{} &0.75	&\textbf{0.8718}	&\textbf{0.9021}	&0.9070 &0.9706 &\textbf{5.91E-4} &\textbf{5.87E-4}	&7.00E-4	&5.27E-3\\

		\multirow{20}{*}{} &\multirow{4}{*}{} &1.0	&\textbf{0.8693}	&\textbf{0.9017}	&0.9051	&0.9702	&\textbf{5.63E-4}	&\textbf{5.69E-4} &6.98E-4 &5.61E-3\\

        \multirow{20}{*}{} &\multirow{4}{*}{} &2.0	&\textbf{0.8691}	&\textbf{0.9012}	&0.9012 &0.9632 &\textbf{6.06E-4} &\textbf{5.05E-4}	&7.49E-4	&5.49E-3\\
		\cmidrule{2-4} \cmidrule{5-6} \cmidrule{7-11}
		
		\multirow{20}{*}{}&\multirow{4}{*}{\emph{Adult}}&0.5	&\textbf{0.8103}	&\textbf{0.8791}	&0.9051	&0.9813	&\textbf{0.128}	&\textbf{0.129} &0.138 &8.388\\
		%\cmidrule{3-4} \cmidrule{5-6} \cmidrule{7-8}
		\multirow{20}{*}{}&\multirow{4}{*}{}&0.75	&\textbf{0.7993}	&\textbf{0.8718}	& 0.8875 	&0.9811	&\textbf{0.125}	&\textbf{0.126} &0.133 &8.297\\
  
		\multirow{20}{*}{}&\multirow{4}{*}{}&1.0	&\textbf{0.7992}	&\textbf{0.8689}	&0.8859	&0.9809	&\textbf{0.585}	&\textbf{0.599} &0.605 &9.264\\

		\multirow{20}{*}{}&\multirow{4}{*}{}&2.0	&\textbf{0.7986}	&\textbf{0.8679}	& 0.8805 	&0.9806	&\textbf{0.126}	&\textbf{0.125} &0.134 &8.317\\
        \cmidrule{2-4} \cmidrule{5-6} \cmidrule{7-11}

		\multirow{20}{*}{}&\multirow{4}{*}{\emph{Student's t}}&0.5	&\textbf{0.8517}	&\textbf{0.8672}	& 0.8796	&0.9058	&\textbf{1.380}	&\textbf{1.381} &1.580 &32.071\\
		%\cmidrule{3-4} \cmidrule{5-6} \cmidrule{7-8}
		\multirow{20}{*}{}&\multirow{4}{*}{}&0.75	&\textbf{0.8443}	&\textbf{0.8526}	& 0.8765 	&0.9055	&\textbf{1.396}	&\textbf{1.384} &1.610 &32.011\\
  
		\multirow{20}{*}{}&\multirow{4}{*}{}&1.0	&\textbf{0.8428}	&\textbf{0.8505}	& 0.8754 	&0.9053	&\textbf{1.396}	&\textbf{1.396} &1.611 &33.011\\

		\multirow{20}{*}{}&\multirow{4}{*}{}&2.0	&\textbf{0.8420}	&\textbf{0.8496}	& 0.8752 	&0.9049	&\textbf{1.395}	&\textbf{1.396} &1.613 &33.015\\
        \cmidrule{2-4} \cmidrule{5-6} \cmidrule{7-11}
        
		\multirow{20}{*}{}&\multirow{4}{*}{\emph{Laplace}}&0.5	&\textbf{0.5767}	&\textbf{0.6128}	& 0.6980	&0.7742	&\textbf{1.415}	&\textbf{1.398} &1.631 &36.645\\
		%\cmidrule{3-4} \cmidrule{5-6} \cmidrule{7-8}
		\multirow{20}{*}{}&\multirow{4}{*}{}&0.75	&\textbf{0.5709}	&\textbf{0.6098}	& 0.6965 	&0.7720	&\textbf{1.290}	&\textbf{1.289} &1.506 &31.339\\
  
		\multirow{20}{*}{}&\multirow{4}{*}{}&1.0	&\textbf{0.5702}	&\textbf{0.6056}	& 0.6960 	&0.7715	&\textbf{1.434}	&\textbf{1.390} &1.727 &32.231\\

		\multirow{20}{*}{}&\multirow{4}{*}{}&2.0	&\textbf{0.5679}	&\textbf{0.6042}	& 0.6957 	&0.7711	&\textbf{1.687}	&\textbf{1.673} &2.024 &32.501\\
        \cmidrule{2-4} \cmidrule{5-6} \cmidrule{7-11}
        
		\multirow{20}{*}{}&\multirow{4}{*}{\emph{$\chi^2$}}&0.5	&\textbf{0.6270}	&\textbf{0.6521}	& 0.7245	&0.7926	&\textbf{1.411}	&\textbf{1.418} &1.561 &32.011\\
		%\cmidrule{3-4} \cmidrule{5-6} \cmidrule{7-8}
		\multirow{20}{*}{}&\multirow{4}{*}{}&0.75	&\textbf{0.6165}	&\textbf{0.6448}	& 0.7230 	&0.7920	&\textbf{1.419}	&\textbf{1.428} &1.567 &32.384\\
  
		\multirow{20}{*}{}&\multirow{4}{*}{}&1.0	&\textbf{0.6118}	&\textbf{0.6432}	& 0.7226 	&0.7916	&\textbf{1.382}	&\textbf{1.392} &1.646 &32.805\\

		\multirow{20}{*}{}&\multirow{4}{*}{}&2.0	&\textbf{0.6113}	&\textbf{0.6421}	& 0.7221 	&0.7911	&\textbf{1.390}	&\textbf{1.401} &1.582 &32.443\\
		\botrule
	\end{tabular}}
	\label{table2}
\end{table*}

\begin{table*}[t]
    \caption{Summary of experimental results of private methods on the ResNet-18 model: average test accuracy $\pm$ standard deviation of the last 5 epochs across 5 independent runs, non-private baseline: $83.28\pm0.02$\% . Bold values mark the results of our methods.}
	\centering	
    \resizebox{\textwidth}{!}{
	\begin{tabular}{cccccc|ccc}
		\toprule
		\multirow{2}{*}{Model} &\multirow{2}{*}{Data} &\multirow{2}{*}{$\epsilon$} & \multicolumn{3}{c}{Test Accuracy}&\multicolumn{3}{|c}{GPU Runtime (per epoch): $s$}\\
		\cmidrule{4-6} \cmidrule{7-9}
		
		& & & CC & UC &DP-SGD  & CC & UC &DP-SGD \\
		\hline
		
		%Teacher &4.12G	&0.87841	&0.89512	&0.9014	&0.83587	&0.85736	&0.86297\\ 
		%\hline
		
		%0.25x-----------------
		\multirow{3}{*}{ResNet-18} &\multirow{3}{*}{\emph{CIFAR-10}}& 2  &
		\textbf{64.23$\pm$0.51}\% & \textbf{63.69$\pm$0.59}\% &	63.09$\pm$0.46\% 	 &\textbf{27.38} & \textbf{27.63} &57.64\\
		%\cmidrule{2-4} \cmidrule{5-6} \cmidrule{7-8}

        \multirow{3}{*}{} &\multirow{3}{*}{} &4	&\textbf{67.78$\pm$0.32}\%	&\textbf{67.19$\pm$0.44}\%	&66.64$\pm$0.37\%  &\textbf{27.02} &\textbf{27.21}	&57.99	\\

		\multirow{3}{*}{} &\multirow{3}{*}{} &6	&\textbf{73.26$\pm$0.16}\%	&\textbf{71.12$\pm$0.18}\%	&70.99$\pm$0.22\%		&\textbf{27.12}	&\textbf{27.17} &58.46 \\
		\bottomrule
	\end{tabular}}
	\label{tabledp}
\end{table*}

\begin{figure*}[t]
	\centering
    %\captionsetup[subfigure]{labelformat=empty}
	\subfigure{
		\includegraphics[width=0.23\textwidth]{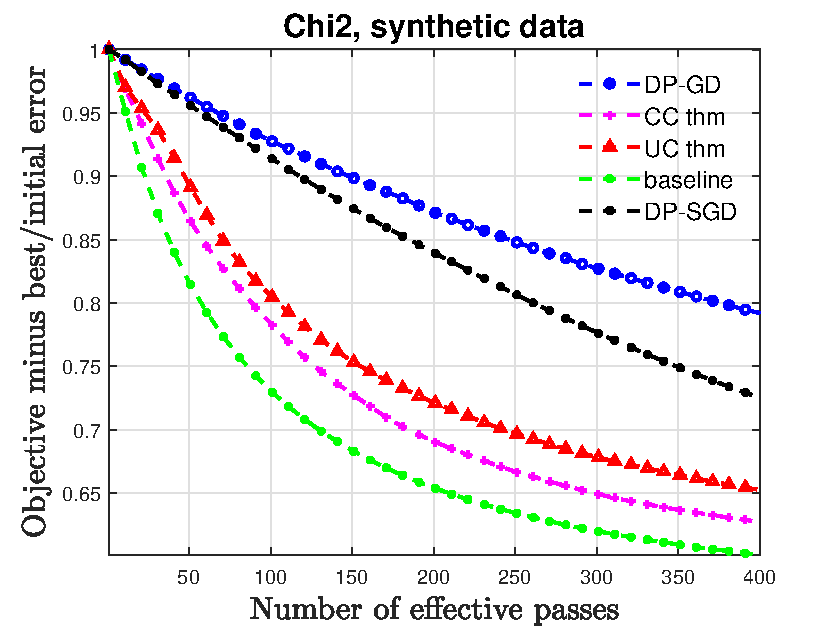}}
	\subfigure{
		\includegraphics[width=0.23\textwidth]{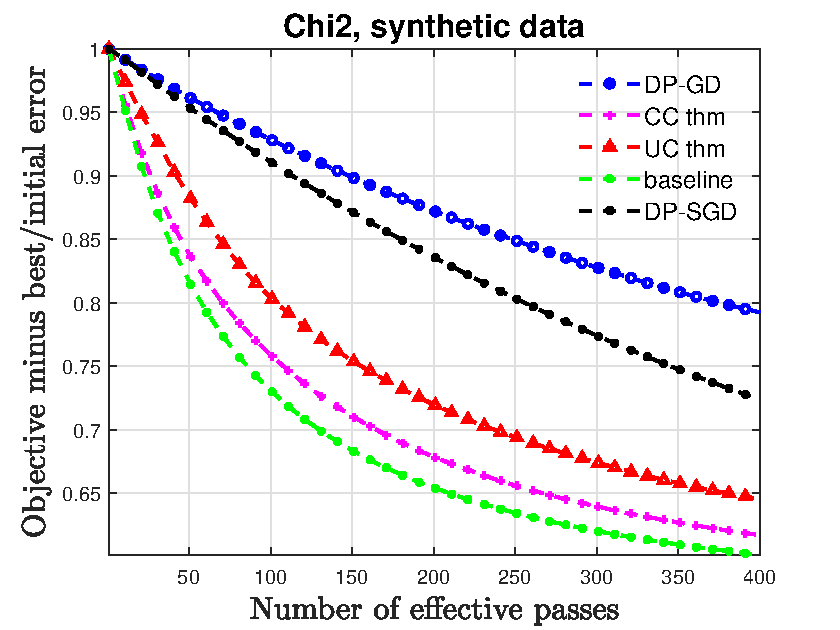}}
	\subfigure{
		\includegraphics[width=0.23\textwidth]{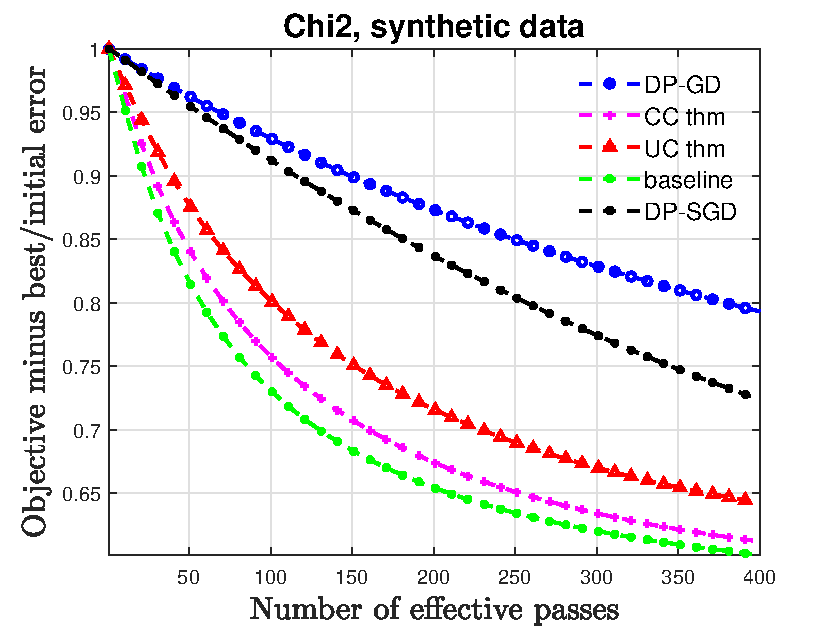}}
	\subfigure{
		\includegraphics[width=0.23\textwidth]{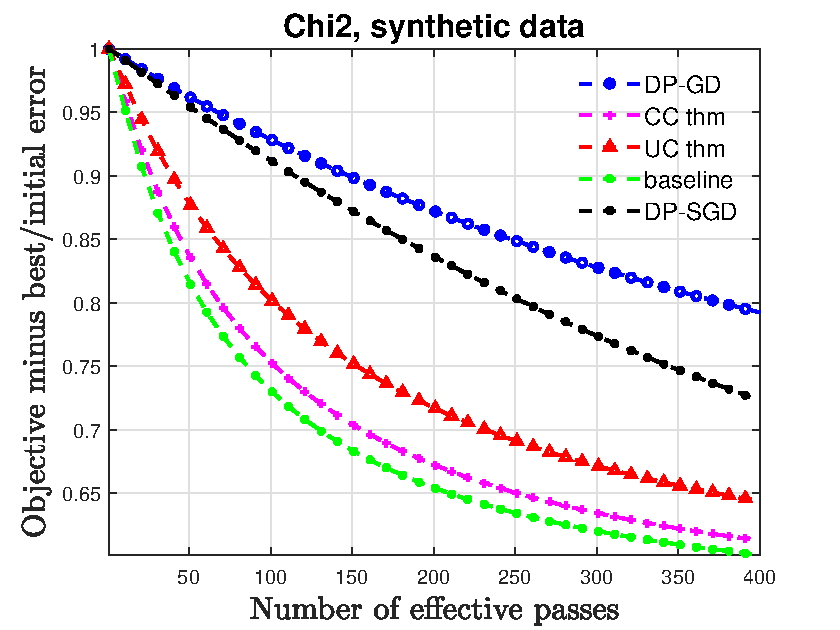}}
  \newline
  	\subfigure{
		\includegraphics[width=0.23\textwidth]{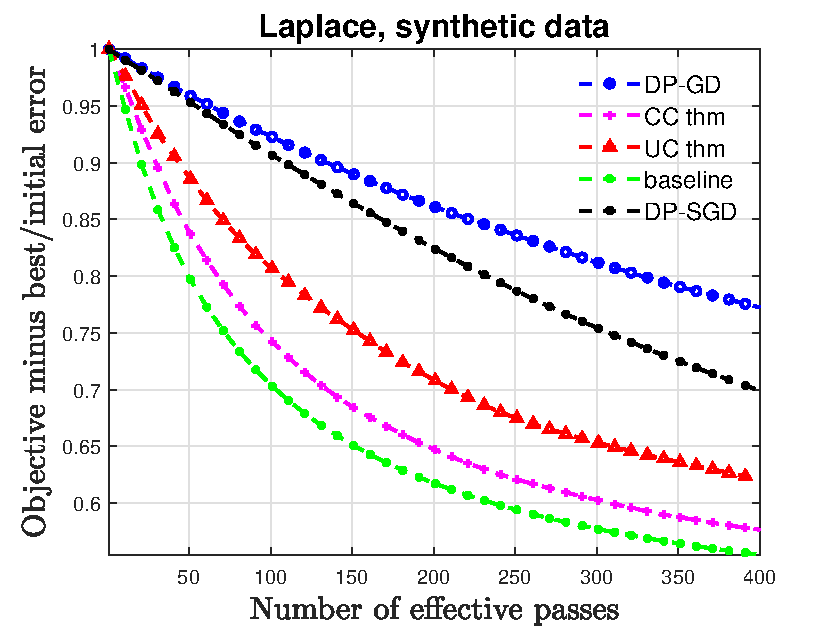}}
	\subfigure{
		\includegraphics[width=0.23\textwidth]{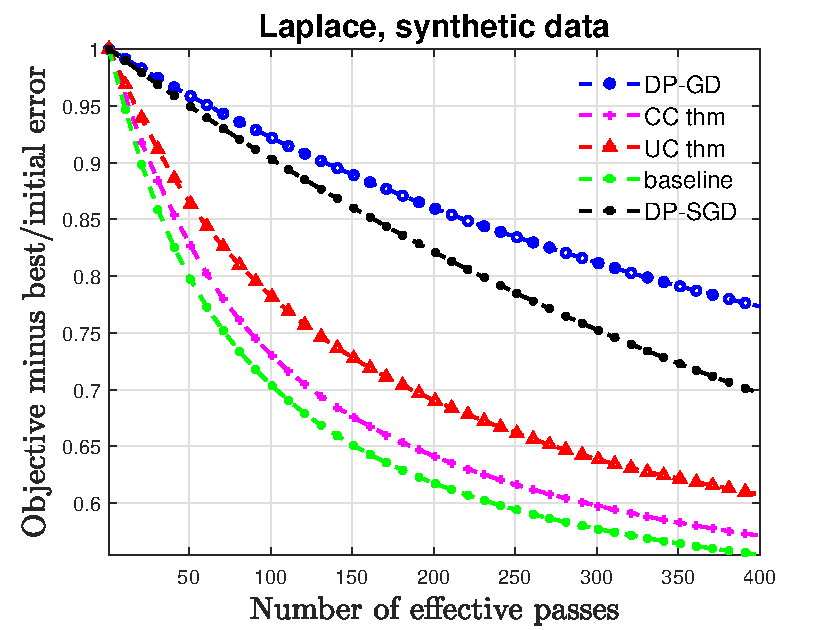}}
	\subfigure{
		\includegraphics[width=0.23\textwidth]{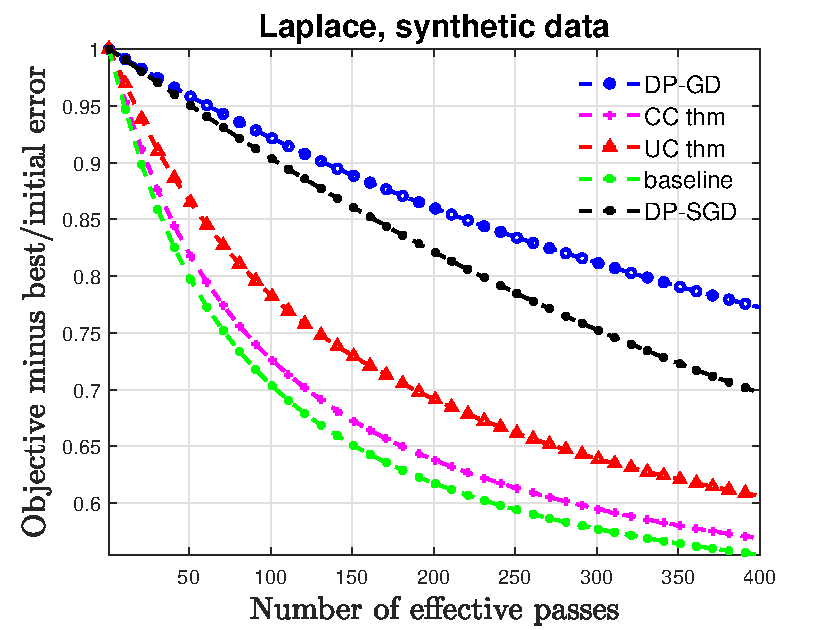}}
	\subfigure{
		\includegraphics[width=0.23\textwidth]{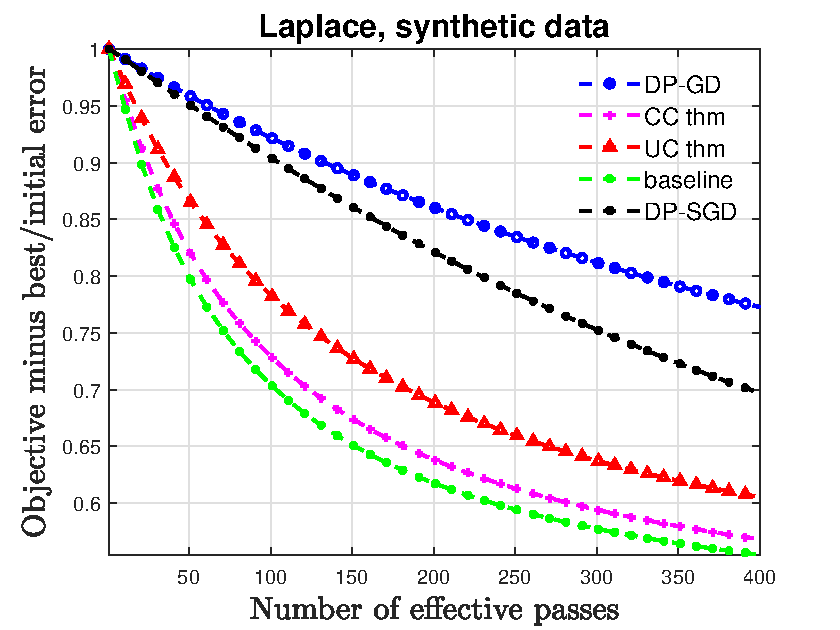}}
  \newline
   \setcounter{subfigure}{0}
  	\subfigure[$\epsilon = 0.5$]{
		\includegraphics[width=0.23\textwidth]{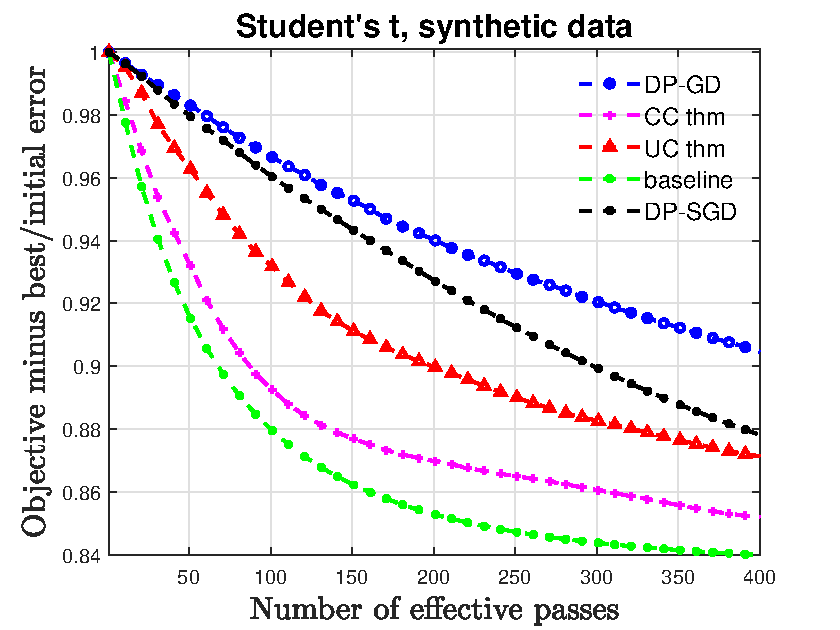}}
	\subfigure[$\epsilon = 0.75$]{
		\includegraphics[width=0.23\textwidth]{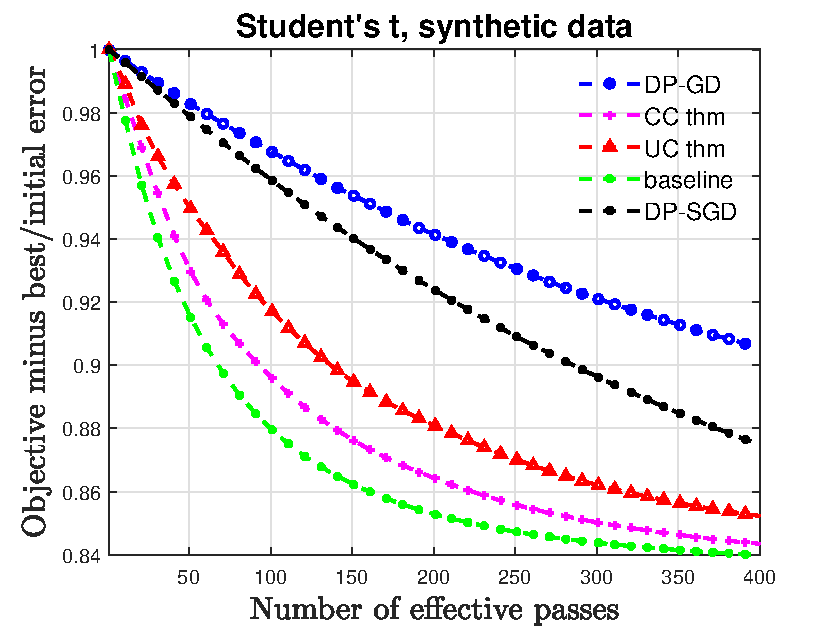}}
	\subfigure[$\epsilon = 1.0$]{
		\includegraphics[width=0.23\textwidth]{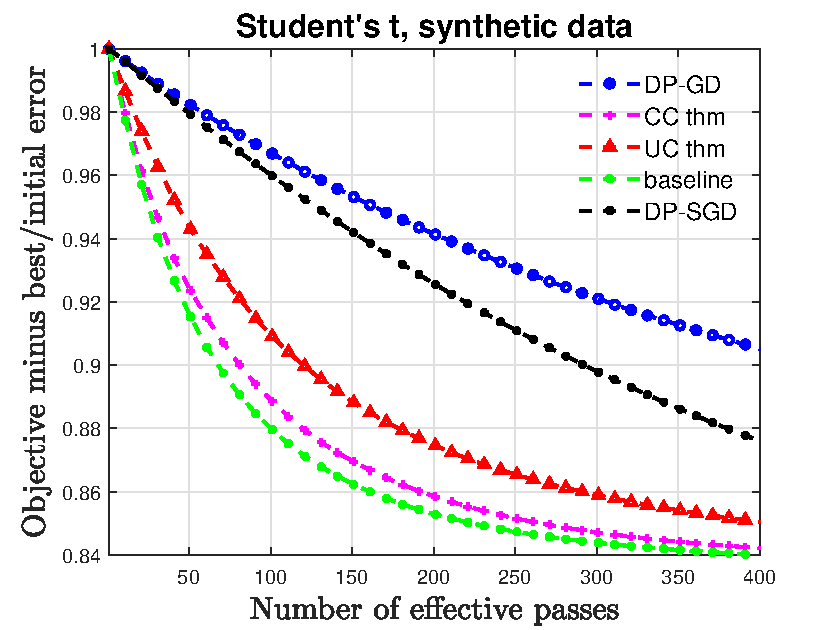}}
	\subfigure[$\epsilon = 2.0$]{
		\includegraphics[width=0.23\textwidth]{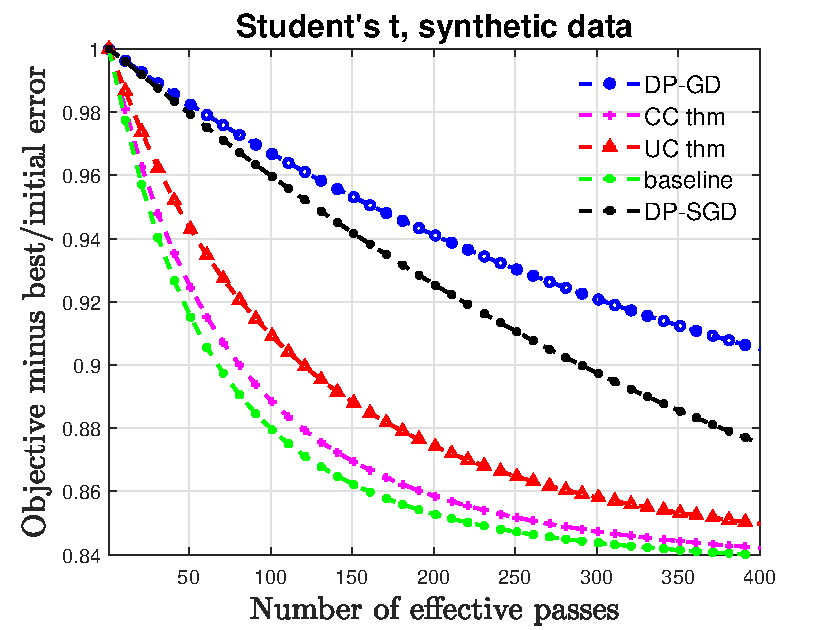}}
    \caption{Trajectories of the logistic regression model for the synthetic data. The three rows correspond to the \textit{Chi-squared distribution}, \textit{Laplace distribution}, and \textit{Student's t-distribution}, respectively.}
 \label{fig2}
\end{figure*}

From the experimental results, we observe that the stochastic methods consistently outperform the deterministic method. We also verify that the unconstrained AClipped-dpSGD performs worse than its constrained counterpart, but remains comparable to DP-SGD.
These results align with our previous theoretical analyses. Additionally, we note a significant improvement in the running time of AClipped-dpSGD compared to DP-SGD. This enhancement is due to AClipped-dpSGD requiring fewer clipping operations, as indicated by fewer clipping function calls during the training process in our source code. Specifically, the pre-iteration cost is decreased from  $O(m)$ to $O(1).$ Now that we have validated our insights in Table \ref{tab1} 
in terms of the risk bound and complexity based on empirical results. Moreover, as we increase the privacy budget $\epsilon,$ we observe overall improvements across all methods in terms of performance, smoothness, and stability. These findings confirm our previous theoretical analysis.

\section{Conclusions}
In this paper, we conduct a comprehensive study of DP-SCO for heavy-tailed data for the (strongly) convex and (non) smooth objectives. We first analyze the advantages of the one-clip strategy, AClip, and establish the expected high probability form of excess risk bounds and gradient complexities for the proposed AClipped-dpSGD method under constrained or unconstrained convex cases.  We show that for (strongly) convex and (non) smooth objectives, our method achieves better risk bounds or runs faster compared with previous works. Moreover, %by using the restarting technique, we extend our results to the strongly convex case and prove new utility bound  without assuming the boundedness of parameter domain. 
as the clipping method is widely used in (private) optimization areas \citep{papernot2021tempered,song2021evading},
our study of the averaged clipping strategy provides a new perspective on choices of clipping method in optimization algorithms and our theoretical analyses of AClipped-dpSGD, an application based on this strategy, are of great help to related studies. We also conduct a numerical study of the considered methods. The results confirm a better performance of our method.
%Although it is interesting to

Admittedly, our approach has some limitations. For example, we do not provide lower bounds for the convex and (non) smooth cases and we do not consider regularized problems.  It would also be interesting to generalize our approach to the general non-convex problems. 

\begin{appendices}

\section{Missing proofs}\label{secA1}
\subsection{Proof of Theorem \ref{cccc}} \label{appendix for thm1}
We start by introducing the following definition:
\begin{definition}	
	(Sub-Gaussian vector \citep{jin2019short}). A random vector $Z \in \mathbb{R}^d$ is said to be sub-Gaussian with variance $\sigma^2$ if it is centered and for any $u \in \mathbb{R}^d$ such that $|u|=1$, the real random variable $u^T Z$ is sub-Gaussian with variance $\sigma^2$. We write $Z \sim $  subG($\sigma^2$).
	
\end{definition}
\begin{proof}
    For all iteration $0 \leq k \leq N-1$, 
	\begin{align*}
	\sum_{t=0}^{k}\|\widetilde{\nabla} f(x^{t},\bm{\xi}^{t})-\nabla f(x^{t})\|_2 
	\leq{}& \sum_{t=0}^{k}\left(\|\widehat{\nabla}f(x^{t},\bm{\xi}^{t})-\nabla f(x^{t})\|_2+\|z^t\|_2\right)\\
	={}&\sum_{t=0}^{k}\|\widehat{\nabla}f(x^{t},\bm{\xi}^{t})-\nabla f(x^{t})\|_2+\sum_{t=0}^{k}\|z^t\|_2\\
    ={}&\underbrace{\sum_{t=0}^{k}\left(\|\theta_{t}^a\|_{2}-\mathbb{E}_{\bm{\xi}^{t}}[\|\theta_{t}^a\|_{2}]\right)}_{\circled{1}}+\underbrace{\sum_{t=0}^{k}\mathbb{E}_{\bm{\xi}^{t}}[\|\theta_{t}^a\|_{2}]}_{\circled{2}}\\	{}&+\underbrace{\sum_{t=0}^{k}\|\theta_{t}^b\|_{2}}_{\circled{3}}+\underbrace{\sum_{t=0}^{k}\|z^t\|_2}_{\circled{4}},
\end{align*}
where we introduce new notations:
	$$\theta_{t}^a \overset{\text{def}}{=}\widehat{\nabla}f(x^{t},\bm{\xi}^{t})-\mathbb{E}_{\bm{\xi}^{t}}\left[\widehat{\nabla}f(x^{t},\bm{\xi}^{t})\right],\       	
	\theta_{t}^b \overset{\text{def}}{=} \mathbb{E}_{\bm{\xi}^{t}}\left[\widehat{\nabla}f(x^{t},\bm{\xi}^{t})\right]-\nabla f(x^{t}).$$ 

Upper bound for $\circled{1}$. 
First of all,  we notice that the terms in $\circled{1}$ are conditionally unbiased:
	$$\mathbb{E}_{\bm{\xi}^{t}}\Bigl[\|\theta_{t}^a\|_{2}-\mathbb{E}_{\bm{\xi}^{t}}[\|\theta_{t}^a\|_{2}]\Bigl]=0.$$
	
	Secondly, the terms are bounded with probability 1:
	\begin{equation*}
	\bigl|\|\theta_{t}^a\|_{2}-\mathbb{E}_{\bm{\xi}^{t}}[\|\theta_{t}^a\|_{2}]\bigl| \leq \|\theta_{t}^a\|_{2}+\mathbb{E}_{\bm{\xi}^{t}}[\|\theta_{t}^a\|_{2}]\leq 4\lambda\overset{\text{def}}{=} c.
	\end{equation*}
	
	Finally, we can bound the conditional variances $\bar{\sigma}_{t}^2= \mathbb{E}_{\bm{\xi}^{t}}\left[\bigl|\|\theta_{t}^a\|_{2}-\mathbb{E}_{\bm{\xi}^{t}}[\|\theta_{t}^a\|_{2}]\bigl|^2\right]$ as follows:
	\begin{equation*}
	\bar{\sigma}_t^2 \leq   c\mathbb{E}_{\bm{\xi}^{t}}\left[\bigl|\|\theta_{t}^a\|_{2}-\mathbb{E}_{\bm{\xi}^{t}}[\|\theta_{t}^a\|_{2}]\bigl|\right]
	\leq c\mathbb{E}_{\bm{\xi}^{t}}\left[\|\theta_{t}^a\|_{2}+\mathbb{E}_{\bm{\xi}^{t}}[\|\theta_{t}^a\|_{2}\right]
	= 2c\mathbb{E}_{\bm{\xi}^{t}}[\|\theta_{t}^a\|_{2}].
	\end{equation*} 

Thus, the sequence $\left\{ \|\theta_{t}^a\|_{2}-\mathbb{E}_{\bm{\xi}^{t}}[\|\theta_{t}^a\|_{2}]  \right\}_{t\geq 0}$ is a bounded martingale differences sequence with bounded conditional variances $\{\bar{\sigma}_t^2\}_{t\geq0}$. Therefore, we can apply Bernstein's inequality \citep{freedman1975tail,dzhaparidze2001bernstein} with $X_t = \|\theta_{t}^a\|_{2}-\mathbb{E}_{\bm{\xi}^{t}}[\|\theta_{t}^a\|_{2}]$, $c=4\lambda$ \text{and} $F=\frac{c^2\ln\frac{4}{\beta}}{6}$ and get for all $b>0,$ it holds that
	$$\mathbb{P}\Bigl\{\sum_{t=0}^{k}\bigl|X_t\bigl| >b \  \text{and} \ \sum_{t=0}^{k}\bar{\sigma}_t^2 \leq F\Bigl\} \leq 2e^{-\frac{b^2}{2F+2cb/3}}.$$ 
	
	Thus, with probability at least $1-2\text{exp}\left(-\frac{b^2}{2F+2cb/3}\right),$ we have
	\begin{center}
		either \ $\sum_{t=0}^{k}\bigl|X_t\bigl| \leq b$ \ or \  $\sum_{t=0}^{k}\bar{\sigma}_t^2 > F.$
	\end{center}
 
	Here we choose $b$ in a way such that $2\text{exp}\left(-\frac{b^2}{2F+2cb/3}\right)\leq\frac{\beta}{2}.$ This implies that $$b^2-\frac{2c\ln\frac{4}{\beta}}{3}b-2F\ln\frac{4}{\beta}\geq 0.$$ 
	
	Hence, $$b \geq \frac{c\ln\frac{4}{\beta}}{3}+\sqrt{\frac{c^2\ln^2\frac{4}{\beta}}{9}+2F\ln\frac{4}{\beta}}.$$
	
    Next, in order to bound $\sum_{t=0}^{k}\bar{\sigma}_t^2$ with probability 1, we have the following inequality for $F$
	\begin{align*}
	    \sum_{t=0}^{k}\bar{\sigma}_t^2 \leq{}&  2c\sum_{t=0}^{k}\mathbb{E}_{\bm{\xi}^{t}}[\|\theta_{t}^a\|_{2}]
	    \overset{(\ref{variance bound})}{\leq}{}6\sqrt{2}c\sigma\sum_{t=0}^{k}\frac{1}{\sqrt{m}}\\
	    ={}& 6\sqrt{2}c\sigma\frac{k}{\sqrt{m}}
	    \overset{k \leq N}{\leq} \frac{c^2\ln\frac{4}{\beta}}{6}\\ ={}& F.
	\end{align*}
Therefore, we have shown that with probability at least $1-\frac{\beta}{2}$, $\sum_{t=0}^{k}\bigl|X_t\bigl| \leq b,$ i.e.,
	$$\sum_{t=0}^{k}\bigl|\|\theta_{t}^a\|_{2}-\mathbb{E}_{\bm{\xi}^{t}}\left[\|\theta_{t}^a\|_2\right]\bigl| \leq b,$$ where	
	$b=c\ln\frac{4}{\beta}=4\lambda\ln\frac{4}{\beta}$ as desired.

    Upper bound for $\circled{2}$.
	\begin{equation*}
	\begin{split}
	\circled{2} =&{} \sum_{t=0}^{k}\mathbb{E}_{\bm{\xi}^{t}}\left[\|\theta_{t}^a\|_{2}\right]
	\leq{} 3\sqrt{2}\sigma\sum_{t=0}^{k}\frac{1}{\sqrt{m}}\\
	=&{}3\sqrt{2}\sigma\frac{k}{\sqrt{m}} \overset{k\leq N}{\leq} 3\sqrt{2}\sigma\frac{N}{\sqrt{m}}\\
	=&{}\frac{\lambda\ln \frac{4}{\beta}}{3},
	\end{split} 	
	\end{equation*}
	where the first inequality also follows the variance bound (\ref{variance bound}) in Lemma \ref{bias}. 

    Upper bound for $\circled{3}$.
	\begin{equation*}
	\begin{split}
	\circled{3} =&{}\sum_{t=0}^{k}\|\theta_{t}^b\|_{2}
	\overset{(\ref{bias})}{\leq}{} 4\sigma^2\sum_{t=0}^{k}\frac{1}{m\lambda}\\
	=&{} 4\sigma^2\frac{k}{m\lambda} \overset{k\leq N}{\leq} 4\sigma^2\frac{N}{m\lambda}\\
	=&{} \frac{2\lambda \ln^2\frac{4}{\beta}}{81N}.
	\end{split} 	
	\end{equation*}

    Upper bound for $\circled{4}.$ 
	It is easy to verify that the injected Gaussian random noise $z \in \mathbb{R}^d, z \sim \mathcal{N}(0,\widehat{\sigma}^2 I^d)$ is subG($\widehat{\sigma}^2$). Thus the tail bound of the vector can be achieved by the Lemma 1 in \cite{jin2019short}:
	\begin{equation*}
	\mathbb{P}\left(\|z\|_2 \geq \widehat{t}\right) \leq 2e^{-\frac{\widehat{t}^2}{2\widehat{c}^2 d\widehat{\sigma}^2}},
	\end{equation*}
	where constant $\widehat{c}$  is selected as $2\sqrt{2}$ as the derivations in \cite{jin2019short}. We choose $\widehat{t}$ such that $2\text{exp}\left[-\frac{\widehat{t}^2}{2\widehat{c}^2d \widehat{\sigma}^2}\right]=\frac{\beta}{2k}.$

That is, we choose $\widehat{t}= \widehat{\sigma}\sqrt{16d \ln\frac{4k}{\beta}}$ so that with the probability at least $1-\frac{\beta}{2}$,
	$$\sum_{t=0}^{k}\|z^t\|_2 \leq k\widehat{\sigma}\sqrt{16d \ln\frac{4k}{\beta}}.$$

Finally, summarizing all the above bounds we have derived
	\begin{align*}
	&{}\sum_{t=0}^{k}\|\widetilde{\nabla} f(x^{t},\bm{\xi}^{t})-\nabla f(x^{t})\|_2 \\
	\leq&{}\underbrace{\sum_{t=0}^{k}\left(\|\theta_{t}^a\|_{2}-\mathbb{E}_{\bm{\xi}^{t}}\left[\|\theta_{t}^a\|_{2}\right]\right)}_{\circled{1}}+\underbrace{\sum_{t=0}^{k}\mathbb{E}_{\bm{\xi}^{t}}\left[\|\theta_{t}^a\|_{2}\right]}_{\circled{2}}+\underbrace{\sum_{t=0}^{k}\|\theta_{t}^b\|_{2}}_{\circled{3}}+\underbrace{\sum_{t=0}^{k}\|z^t\|_2}_{\circled{4}},\\
	&{} E_1:\mathbb{P}\left\{\circled{1} \leq 4\lambda\ln\frac{4}{\beta}\right\} \geq 1-\frac{\beta}{2},\\
	&{} E_2:\circled{2} \leq\frac{\lambda\ln \frac{4}{\beta}}{3},\ E_3:\circled{3} \leq \frac{2\lambda \ln^2\frac{4}{\beta}}{81N},\\
	&{}E_4:\mathbb{P}\left\{\circled{4} \leq  k\widehat{\sigma}\sqrt{16d \ln\frac{4k}{\beta}}\right\} \geq 1-\frac{\beta}{2}.
	\end{align*}

Then, we have for each event $E_i,\ i=1,2,3,4,\   \bigcap_{i=1}^{4}\mathbb{P}(E_{i})=1-\bigcup_{i=1}^{4}\mathbb{P}(\bar{E_{t}})\geq 1-\beta$, i.e., with probability at least $1-\beta$,
	\begin{equation*} \label{case1}
	\begin{split}
		&{}\sum_{t=0}^{k}\|\widetilde{\nabla} f(x^{t},\bm{\xi}^{t})-\nabla f(x^{t})\|_2 \leq\\ &{}\lambda\left(4\ln\frac{4}{\beta}+\frac{\ln \frac{4}{\beta}}{3}+\frac{2\ln^2\frac{4}{\beta}}{81N}+\frac{k\widehat{\sigma}\sqrt{16d \ln\frac{4k}{\beta}}}{\lambda}\right). 
	\end{split}
	\end{equation*}
\end{proof}

\subsection{Proof of Theorem \ref{BC}}  \label{appendix for thm3}
\begin{proof}
    For any $k$-th iteration, let  $\hat{x}^k:= x^{k}-\gamma \widetilde{\nabla} f(x^{k},\bm{\xi}^{k}),$ we have $\|x^{k+1}-x^*\|_2 \leq \|\hat{x}^k-x^*\|_2$ by the property of projection.
	Then with the convexity and $L$-smoothness of $f(x)$, we have the following inequality:
	\begin{align*}
	\|x^{k+1}-x^{*}\|_{2}\leq{}&\|\hat{x}^k-x^*\|_2\\
	={}&\|x^{k}-\gamma \widetilde{\nabla} f(x^{k},\bm{\xi}^{k}) -x^{*}\|_{2}\\
	\leq{}&\|x^{k}-\gamma \left(\widetilde{\nabla}f(x^{k},\bm{\xi}^{k})-\nabla f(x^k)+\nabla f(x^{k})\right) -x^{*}\|_{2}\\
	\leq{}&\|x^{k}-\gamma \nabla f(x^{k}) -x^{*}\|_{2}+\gamma \|\widetilde{\nabla}f(x^{k},\bm{\xi}^{k})-\nabla f(x^k)\|_{2}.		
	%	\leq{}&\|x^{k}-\gamma{\nabla} f(x^{k},\bm{\xi}^{k}) -x^{*}\|_{2}+\gamma\|\theta_{k}\|_{2}+\gamma \|z_{k}\|_{2}
	\end{align*}
	%where θk=ˆ∇f(xk,\bmξk)−f(xk).\theta_{k}=\widehat{\nabla} f(x^{k},\bm{\xi}^{k})-f(x^{k}).
	
	For the  first term, we can expand it as
	\begin{align*}
	\|x^{k}-\gamma{\nabla} f(x^{k}) -x^{*}\|_{2}^2={}&\|x^{k}-x^{*}\|_{2}^{2}+\gamma ^{2}\|\nabla f(x^{k})\|_{2}^{2}-2\gamma\left\langle x^{k}-x^{*},{\nabla} f(x^{k})\right\rangle\\
	\leq{}&\|x^{k}-x^{*}\|_{2}^{2}+\gamma ^{2}\|\nabla f(x^{k})\|_{2}^{2}-2\gamma(f(x^{k})-f(x^{*}))\\
	\leq{}&\|x^{k}-x^{*}\|_{2}^{2}+2\gamma ^{2}L(f(x^{k})-f(x^{*}))-2\gamma(f(x^{k})-f(x^{*}))\\
	\leq{}&\|x^{k}-x^{*}\|_{2}^{2}+2\gamma(\gamma L-1) \left(f(x^{k})-f(x^{*})\right)\\
	\leq{}&\|x^{k}-x^{*}\|_{2}^{2}\left(1-\frac{-2\gamma(\gamma L-1)}{\|x^{k}-x^{*}\|_{2}^{2}}(f(x^{k})-f(x^{*}))\right).
	\end{align*}

    Choosing $\gamma$ such that $1-\gamma L>0$ and using the inequality  $\sqrt{1-x}\leq 1-\frac{x}{2}$, we obtain
	\begin{align*}
	\|x^{k}-\gamma{\nabla} f(x^{k}) -x^{*}\|_{2}={}&\sqrt{\|x^{k}-\gamma{\nabla} f(x^{k}) -x^{*}\|_{2}^2}\\
	\leq{}&\|x^{k}-x^{*}\|_{2}\left(1-\frac{-\gamma(\gamma L-1)}{\|x^{k}-x^{*}\|_{2}^{2}}(f(x^{k})-f(x^{*}))\right)\\
	\leq{}&\|x^{k}-x^{*}\|_{2}+\frac{\gamma(\gamma L-1)}{\|x^{k}-x^{*}\|_{2}}(f(x^{k})-f(x^{*})).
	\end{align*}
	
	Thus, the following holds
	\begin{equation} \label{999}
    \begin{split}
        \|x^{k+1}-x^{*}\|_{2}\leq{}& \|x^{k}-x^{*}\|_{2}+\frac{\gamma(\gamma L-1)}{\|x^{k}-x^{*}\|_{2}}(f(x^{k})-f(x^{*}))\\
        &{}+\gamma \|\widetilde{\nabla}f(x^{k},\bm{\xi}^{k})-\nabla f(x^k)\|_{2}.
    \end{split}
	\end{equation}
 
    Using the notation $R_k \overset{\text{def}}{=}\|x^k-x^*\|_2,$    and summing up ($\ref{999}$) for $k=0,1,\ldots,N-1$,
	we derive that 
	$$\sum_{k=0}^{N-1}\frac{\gamma(1-\gamma L)}{R_k}(f(x^{k})-f(x^{*}) \leq R_0 - R_N + \gamma\sum_{k=0}^{N-1}\|\widetilde{\nabla}f(x^{k},\bm{\xi}^{k})-\nabla f(x^k)\|_{2}.$$

    Since $R_k \geq 0,$ it holds that
	$$\gamma(1-\gamma L)\sum_{k=0}^{N-1}\frac{(f(x^{k})-f(x^{*}))}{R_k}\leq R_0 + \gamma\sum_{k=0}^{N-1}\|\widetilde{\nabla}f(x^{k},\bm{\xi}^{k})-\nabla f(x^k)\|_{2}.$$

    By Assumption \ref{assump2}, we have $R_k \leq R_0$ for all $k=0,\ldots,N-1.$ Let us define $A=\frac{\gamma(1-\gamma L)}{R_0}>0,$ then we obtain
	\begin{align*}
	A\sum_{k=0}^{N-1}(f({x}^{k})-f(x^{*}))
	\leq{}& \gamma(1-\gamma L)\sum_{k=0}^{N-1}\frac{(f(x^{k})-f(x^{*}))}{R_k}\\
	\leq{}& R_0 + \gamma\sum_{k=0}^{N-1}\|\widetilde{\nabla}f(x^{k},\bm{\xi}^{k})-\nabla f(x^k)\|_{2}.
	\end{align*}

    Noting that $\bar{x}^{N}=\frac{1}{N}\sum_{k=0}^{N-1} x^{k}$, then with the Jensen's inequality $f(\bar{x}^{N}) = f\left(\frac{1}{N}\sum_{k = 0}^{N-1}x^k\right) \leq \frac{1}{N}\sum_{k = 0}^{N-1}f(x^k),$ we have
	\begin{equation} \label{T21}
	AN(f(\bar{x}^{N})-f(x^{*}))\leq R_{0}+ \gamma\sum_{k=0}^{N-1}\|\widetilde{\nabla}f(x^{k},\bm{\xi}^{k})-\nabla f(x^k)\|_{2}.
	\end{equation}
 
	Since $f$ is $L$-smooth, this implies
	\begin{equation*}
	\|\nabla f(x^{k})\|_{2} \leq L\|x^k-x^{*}\|_2 \leq LR_{0} \leq LR =\frac{\lambda}{2},
	\end{equation*}
	for $k=0,1,\ldots,N-1$ and the known constant $R \propto R_0,$ we implicitly let $\propto$ is a directly proportional operator with the constant $\hat{c} \geq 1$ (to enable a large range when tuning).  Then the clipping level $\lambda$ can be chosen as
	\begin{equation*}
	\lambda= 2LR,
	\end{equation*}
	which implies that Lemma $\ref{bias}$ holds for all $k$.

    Thus, for $k=N$ and $\widehat{\sigma} = \frac{\lambda m \sqrt{N\ln(1/\delta)}}{n\epsilon},$ using Theorem \ref{cccc} we have that with probability at least $1-\beta$,
	\begin{equation} \label{9999}
	\begin{split}
	&{}AN(f(\bar{x}^{N})-f(x^{*}))\\
	&{}\leq R_{0}+ 4\gamma\lambda\ln\frac{4}{\beta}+\gamma\lambda\frac{\ln \frac{4}{\beta}}{3}
	+\gamma\lambda\frac{2\ln^2\frac{4}{\beta}}{81N}+\gamma N\widehat{\sigma}\sqrt{16d \ln\frac{4N}{\beta}} \\
	&{}\leq R_{0}+ 4\gamma\lambda\ln\frac{4}{\beta}+\gamma\lambda\frac{\ln \frac{4}{\beta}}{3}
	+\gamma\lambda\frac{2\ln^2\frac{4}{\beta}}{81N}+\frac{\gamma\lambda Nm \sqrt{16dN \ln \frac{1}{\delta} \ln\frac{4N}{\beta}}}{n\epsilon}\\
	&{}\leq R_{0}+12\gamma LR\ln\frac{4}{\beta}+\frac{2\gamma LR Nm \sqrt{16dN \ln \frac{1}{\delta}\ln\frac{4N}{\beta}}}{n\epsilon}.	
	\end{split}
	\end{equation}
    where the first inequality follows by Theorem \ref{cccc} and inequality (\ref{T21}).

	Since $A=\frac{\gamma(1-\gamma L)}{R_0}$ and $1-\gamma L \geq \frac{1}{2}$, we get that with $\gamma = \frac{1}{2L\ln \frac{4}{\beta}}$ and probability at least $1-\beta$,
	\begin{equation*}
	\begin{split}
	f(\bar{x}^{N})-f(x^{*})
	\overset{(\ref{9999})}{\leq}{}& \frac{28LR^2 \ln\frac{4}{\beta}}{N}+\frac{4 LR^2 m \sqrt{16dN \ln \frac{1}{\delta} \ln\frac{4N}{\beta}}}{n\epsilon}\\
	\leq{}& O\left(\frac{LR^2\ln\frac{4}{\beta}}{N}+\frac{ N^{\frac{5}{2}} \sigma^2\sqrt{d\ln\frac{4N}{\beta}\ln
			\frac{1}{\delta}}}{ n\epsilon L\ln^2\frac{4}{\beta}}\right )\\
	\leq{}& O\left(\frac{LR^2\ln\frac{4}{\beta}}{N}+\frac{ N^{\frac{5}{2}} \sigma^2\sqrt{d\ln\frac{4N}{\beta}\ln
			\frac{1}{\delta}}}{ n\epsilon L}\right )		
	\end{split}
	\end{equation*}
 
	Thus if we take $N$ such that 
	$O\left(\frac{LR^2\ln\frac{4}{\beta}}{N}\right) \leq \left(\frac{ N^{\frac{5}{2}} \sigma^2\sqrt{d\ln
			\frac{1}{\delta}}}{ n\epsilon L}  \right) $ 
	we have with probability at least $1-\beta$, the following holds  for $n \geq \Omega(\frac{\sqrt{d}}{\epsilon}),$ 
	$$f(\bar{x}^{N})-f(x^{*}) \leq  \Tilde{O}\left(\frac{ (d\ln\frac{1}{\delta})^{\frac{1}{7}}\sqrt{\ln\frac{(n\epsilon)^{2}}{\beta d}}}{(n\epsilon)^{\frac{2}{7}}(\ln\frac{4}{\beta})^{\frac{1}{2}}}\right)$$
	where $N=\Tilde{O}\left(\frac{R^2n\epsilon}{\sqrt{d}}\right)^{\frac{2}{7}}.$ The total gradient complexity is computed as $\Tilde{O}(mdN),$ i.e., $\Tilde{O}\left(\max \left\{n^{\frac{2}{7}}d^{\frac{6}{7}},n^{\frac{6}{7}}d^{\frac{4}{7}}\right\}\right).$
\end{proof}
\subsection{Proof of Theorem \ref{ee}}  \label{appendix for thm4}
\begin{proof}
    With the parameter domain $\mathcal{X}$ being $\mathbb{R}^d,$ the convexity and smoothness of $f,$ we have the following inequality directly:
    \begin{align*}
	\|x^{k+1}-x^{*}\|_{2}={}&\|x^{k}-\gamma \widetilde{\nabla} f(x^{k},\bm{\xi}^{k}) -x^{*}\|_{2}\\
	\leq{}&\|x^{k}-\gamma \left(\widetilde{\nabla}f(x^{k},\bm{\xi}^{k})-\nabla f(x^k)+\nabla f(x^{k})\right) -x^{*}\|_{2}\\
	\leq{}&\|x^{k}-\gamma \nabla f(x^{k}) -x^{*}\|_{2}+\gamma \|\widetilde{\nabla}f(x^{k},\bm{\xi}^{k})-\nabla f(x^k)\|_{2}.	
	\end{align*}

	For the  first term on the right-hand side, we can expand it as
	\begin{align*}
	\|x^{k}-\gamma{\nabla} f(x^{k}) -x^{*}\|_{2}^2={}&\|x^{k}-x^{*}\|_{2}^{2}+\gamma ^{2}\|\nabla f(x^{k})\|_{2}^{2}-2\gamma\langle x^{k}-x^{*},{\nabla} f(x^{k})\rangle\\
	\leq{}&\|x^{k}-x^{*}\|_{2}^{2}+\gamma ^{2}\|\nabla f(x^{k})\|_{2}^{2}-2\gamma\left(f(x^{k})-f(x^{*})\right)\\
	\leq{}&\|x^{k}-x^{*}\|_{2}^{2}+2\gamma ^{2}L(f(x^{k})-f(x^{*}))-2\gamma(f(x^{k})-f(x^{*}))\\
	\leq{}&\|x^{k}-x^{*}\|_{2}^{2}+2\gamma(\gamma L-1) (f(x^{k})-f(x^{*}))\\
	\leq{}&\|x^{k}-x^{*}\|_{2}^{2}\left(1-\frac{-2\gamma(\gamma L-1)}{\|x^{k}-x^{*}\|_{2}^{2}}(f(x^{k})-f(x^{*}))\right).
	\end{align*}

	Choosing $\gamma$ such that $1-\gamma L>0$ and using the inequality $\sqrt{1-x}\leq 1-\frac{x}{2}$, we obtain
	\begin{align*}
	\|x^{k}-\gamma{\nabla} f(x^{k}) -x^{*}\|_{2}={}&\sqrt{\|x^{k}-\gamma{\nabla} f(x^{k}) -x^{*}\|_{2}^2}\\
	\leq{}&\|x^{k}-x^{*}\|_{2}\left(1-\frac{-\gamma(\gamma L-1)}{\|x^{k}-x^{*}\|_{2}^{2}}(f(x^{k})-f(x^{*}))\right)\\
	\leq{}&\|x^{k}-x^{*}\|_{2}+\frac{\gamma(\gamma L-1)}{\|x^{k}-x^{*}\|_{2}}(f(x^{k})-f(x^{*})).
	\end{align*}

	Thus, the following holds
	\begin{equation} \label{555}
    \begin{split}
        \|x^{k+1}-x^{*}\|_{2}\leq&{} \|x^{k}-x^{*}\|_{2}+\frac{\gamma(\gamma L-1)}{\|x^{k}-x^{*}\|_{2}}(f(x^{k})-f(x^{*}))\\
        &{}+\gamma \|\widetilde{\nabla}f(x^{k},\bm{\xi}^{k})-\nabla f(x^k)\|_{2}.
    \end{split}	
	\end{equation}

    Using the notation $R_k \overset{\text{def}}{=}\|x^k-x^*\|_2$, and summing up ($\ref{555}$) for $k=0,1,\ldots,N-1$,
	we derive that 
	$$\sum_{k=0}^{N-1}\frac{\gamma(1-\gamma L)}{R_k}(f(x^{k})-f(x^{*}) \leq R_0 - R_N + \gamma\sum_{k=0}^{N-1}\|\widetilde{\nabla}f(x^{k},\bm{\xi}^{k})-\nabla f(x^k)\|_{2}.$$

    Let us define $A=\gamma(1-\gamma L)>0,$ then
	\begin{equation} \label{cx888}
	A\sum_{k=0}^{N-1}\frac{(f(x^{k})-f(x^{*}))}{R_k}\leq R_0 -R_N+ \gamma\sum_{k=0}^{N-1}\|\widetilde{\nabla}f(x^{k},\bm{\xi}^{k})-\nabla f(x^k)\|_{2}.
	\end{equation}

    Taking into account that $$A\sum_{k=0}^{N-1}\frac{(f(x^{k})-f(x^{*}))}{R_k} \geq 0,$$ and changing the indices in (\ref{cx888}), we get that for all $T \geq 0$,
	\begin{equation} \label{RT}
	R_T \leq R_0 + \gamma\sum_{k=0}^{T-1}\|\widetilde{\nabla}f(x^{k},\bm{\xi}^{k})-\nabla f(x^k)\|_{2}.
	\end{equation}

	The remaining part of the proof is based on the analysis of inequality ($\ref{RT}$). In particular, via induction we prove that for all $T=0,\ldots,N$ with probability at least $1-T\beta,$  the following statement holds
	\begin{equation} \label{RTT}
	\begin{split}
	R_t \leq{}& R_0 + \gamma\sum_{k=0}^{t-1}\|\widetilde{\nabla}f(x^{k},\bm{\xi}^{k})-\nabla f(x^k)\|_{2}\\
	\leq{}& R_0+ \gamma\lambda 6\ln\frac{4}{\beta}+\frac{648\gamma N^3 \sigma^2\sqrt{dN\ln\frac{4N}{\beta}\ln
			\frac{1}{\delta}}}{\lambda n\epsilon\ln^2\frac{4}{\beta}},
	\end{split}	
	\end{equation}
	for all $t=0,\ldots,T$ simultaneously where $\gamma,\ \lambda$ will be defined further. Let us define $D:=\frac{648\gamma L N^3 \sigma^2\sqrt{dN\ln\frac{4N}{\beta}\ln
			\frac{1}{\delta}}}{ n\epsilon\ln^2\frac{4}{\beta}},$ and the probability event when  statement ($\ref{RTT}$) holds as $E_T.$ Then, our goal is to show that $\mathbb{P}(E_{T})\geq 1-T\beta$ for all $T=0,\ldots,N.$ Clearly, when $T=0,$ inequality ($\ref{RTT}$) holds with probability 1. Next, assuming
	that for $T\leq N-1,$ we have $\mathbb{P}(E_{T})\geq 1-T\beta.$ Let us prove that $\mathbb{P}(E_{T+1})\geq 1-(T+1)\beta.$ 

    First of all, probability event $E_T$ implies that
	\begin{align*}
	R_t \overset{(\ref{RT})}{\leq}{}& R_0 + \gamma\sum_{k=0}^{t-1}\|\widetilde{\nabla}f(x^{k},\bm{\xi}^{k})-\nabla f(x^k)\|_{2}\\
	\leq{}\ & R_0+ \gamma\lambda 6\ln\frac{4}{\beta}+\frac{D}{L\lambda},
	\end{align*}
	holds for $t=0,\ldots,T.$ Since $f(\cdot)$ is $L$-smooth, this implies
	\begin{equation*}
	\|\nabla f(x^{t})\|_{2} \leq L\|x^t-x^{*}\|_2 \leq L\left(R_0+ \gamma\lambda 6\ln\frac{4}{\beta}+\frac{D}{L\lambda}\right)\leq\frac{\lambda}{2},
	\end{equation*}
	for $t=0,\ldots,T.$ With $\gamma = \frac{1}{24L\ln\frac{4}{\beta}},$ we have 
	$$\frac{\lambda^2}{4}-LR_0\lambda-D \geq 0.$$

    Hence,
	$$\lambda \geq 2\left(LR_0+\sqrt{L^2R_0^2+D}\right).$$
	
	Then the clipping level $\lambda$ can be chosen as
	\begin{equation} \label{T31}
	\lambda= 4LR+2\sqrt{D}
	\end{equation}
	where a constant $R \propto \|x^0 - x^*\|_2$ with the proportional constant $\hat{c} \geq 1$ implicitly.

    Secondly, since event $E_T$ implies $\|\nabla f(x^{t})\|_{2} \leq \frac{\lambda}{2}$ holds for $t=0,\cdots,T,$ using Theorem \ref{cccc} with $\widehat{\sigma} = \frac{\lambda m \sqrt{N\ln(1/\delta)}}{n\epsilon},$ we have
	the following probability event, defined as  $E_{\circled{1}}$, holds with probability at least $1-\beta$,

	\begin{align*}
	{}&\sum_{k=0}^{T}\|\widetilde{\nabla} f(x^{k},\bm{\xi}^{k})-\nabla f(x^{k})\|_2\\ 
	&{}\leq \lambda\left(4\ln\frac{4}{\beta}+\frac{\ln \frac{4}{\beta}}{3}+\frac{2\ln^2\frac{4}{\beta}}{81N
	}+\frac{(T+1)\widehat{\sigma}\sqrt{16d \ln\frac{4(T+1)}{\beta}}}{\lambda}\right)\\
	&{}\leq \lambda\left(4\ln\frac{4}{\beta}+\frac{\ln \frac{4}{\beta}}{3}+\frac{2\ln^2\frac{4}{\beta}}{81N
	}+\frac{N\widehat{\sigma}\sqrt{16d \ln\frac{4N}{\beta}}}{\lambda}\right)\\
	&{}\leq \lambda 6\ln\frac{4}{\beta}+\frac{D}{L\lambda\gamma}.
	\end{align*}

	Finally, event $E_{T+1}$ :
	\begin{align*}
	R_{T+1} \leq R_0 + \gamma\sum_{k=0}^{T}\|\widetilde{\nabla}f(x^{k},\bm{\xi}^{k})-\nabla f(x^k)\|_{2}
	\leq R_0+ \gamma\lambda 6\ln\frac{4}{\beta}+\frac{D}{L\lambda},
	\end{align*}
	holds with probability $\mathbb{P}\{E_{T+1}\}=\mathbb{P}\{E_T \bigcap E_{\circled{1}} \}=1-\mathbb{P}\{\bar{E}_T\bigcup \bar{E}_{\circled{1}}\}\geq 1-(T+1)\beta.$ 

	Thus, we have proved that for all $k=0,\ldots,N,$ we have $\mathbb{P}\{E_k\}\geq 1-k\beta,$ which implies $R_k \leq R_0+ \gamma\lambda 6\ln\frac{4N}{\beta}+\frac{D}{L\lambda} \leq \frac{\lambda}{2L}$ with probability at least $1
	-k\beta.$

	Then, using (\ref{cx888}), for $T = N$ we have that with probability at least $1-N\beta,$
	\begin{align*}
	&2LA\sum_{k=0}^{N-1}\frac{(f(x^{k})-f(x^{*}))}{\lambda}\leq A\sum_{k=0}^{N-1}\frac{(f(x^{k})-f(x^{*}))}{R_k}\\
	&\overset{(\ref{cx888})}{\leq} R_0+ \gamma\sum_{k=0}^{N-1}\|\widetilde{\nabla}f(x^{k},\bm{\xi}^{k})-\nabla f(x^k)\|_{2}\\
	&\ \leq \ R_0+ \gamma\lambda 6\ln\frac{4}{\beta}+\frac{D}{L\lambda} \leq \frac{\lambda}{2L}.
	\end{align*}

	Noting that $\bar{x}^{N}=\frac{1}{N}\sum_{k=0}^{N-1} x^{k}$, then with the Jensen’s inequality, we have
	\begin{equation} \label{T32}
	\frac{2ANL}{\lambda}(f(\bar{x}^{N})-f(x^{*}))\leq \frac{\lambda}{2L}.	
	\end{equation}
	
	Since $A=\gamma(1-\gamma L)$ and $1-\gamma L \geq \frac{1}{2}$, we get that with probability at least $1-\beta$,
	\begin{align*}
	f(\bar{x}^{N})-f(x^{*}) \overset{(\ref{T32})}{\leq}&{} \frac{\lambda^2}{4AL^2N} \leq \frac{\lambda^2}{2\gamma L^2N} \overset{(\ref{T31})}{\leq} \frac{16L^2R^2+4D}{2\gamma L^2N}\\
	\leq{}& O\left(\frac{LR^2\ln\frac{4N}{\beta}}{N}+\frac{ N^{\frac{5}{2}} \sigma^2\sqrt{d
			\ln\frac{4N^2}{\beta}\ln
			\frac{1}{\delta}}}{ n\epsilon L\ln^2\frac{4N}{\beta}}\right)\\
	\leq{}& O\left(\frac{LR^2\ln\frac{4N}{\beta}}{N}+\frac{ N^{\frac{5}{2}} \sigma^2\ln\frac{4N}{\beta}\sqrt{d
			\ln
			\frac{1}{\delta}}}{ n\epsilon L}\right)
	\end{align*}

	Thus if we take $N$ such that
	$O\left(\frac{LR^2\ln\frac{4N}{\beta}}{N}\right) \leq \left(\frac{ N^{\frac{5}{2}} \sigma^2\ln\frac{4N}{\beta}\sqrt{d\ln
			\frac{1}{\delta}}}{ n\epsilon L}  \right),$
	we have with probability at least $1-\beta$, the following holds for $n \geq \Omega(\frac{\sqrt{d}}{\epsilon}),$ 
	$$f(\bar{x}^{N})-f(x^{*}) \leq  \Tilde{O}\left(\frac{ (d\ln\frac{1}{\delta})^{\frac{1}{7}}\ln\frac{(n\epsilon)^2}{\beta d}}{(n\epsilon)^{\frac{2}{7}}}\right),$$
	where $N=\Tilde{O}\left(\frac{R^2n\epsilon}{\sqrt{d\ln\frac{1}{\delta}}}\right)^{\frac{2}{7}}.$ The total gradient complexity is $\Tilde{O}\left(\max \left\{n^{\frac{2}{7}}d^{\frac{6}{7}},n^{\frac{6}{7}}d^{\frac{4}{7}}\right\}\right).$
\end{proof}

\subsection{Proof of Theorem \ref{h}} \label{appendix for thm6}
\begin{proof}
	Consider the first run of AClipped-dpSGD (Algorithm \ref{algorithm}). Observed
	that the proof of Theorem 3 still holds if we substitute $R_0$ everywhere by its upper bound, using $\mu$-strongly convexity of $f(\cdot)$ and the fact that $R \propto \|x^k-x^*\|_2$ with the proportional constant $\hat{c} \geq 1,$ we have
	$$R^2 = \hat{c}^2 R_0^2=\hat{c}^2\|x^0-x^*\|_2^2\leq \frac{2\hat{c}^2}{\mu}(f(x^0)-f(x^*)).$$ 

	It implies that after $N_0$ iterations of AClipped-dpSGD, we have
	$$f(\bar{x}^{N_0})-f(x^*) \leq \frac{384\hat{c}^2L\ln\frac{4N_0}{\beta}}{N_0\mu}(f(x^0)-f(x^*))+O\left(\frac{ N_0^{\frac{5}{2}} \sigma^2\sqrt{d\ln\frac{4N_0^2}{\beta}\ln
			\frac{1}{\widehat{\delta}}}}{ n\widehat{\epsilon} L\ln^2\frac{4N_0}{\beta}}\right),$$  with probability at least $1-\beta.$

	Thus, taking $\frac{N_0}{\ln\frac{4N_0}{\beta}}\geq \frac{768\hat{c}^2L}{\mu}$, we have 
	$$f(\bar{x}^{N_0})-f(x^*) \leq \frac{1}{2}(f(x^0)-f(x^*))+O\left(\frac{ N_0^{\frac{5}{2}} \sigma^2\sqrt{d\ln\frac{4N_0^2}{\beta}\ln
			\frac{1}{\widehat{\delta}}}}{ n\widehat{\epsilon} L\ln^2\frac{4N_0}{\beta}}\right).$$

	Then, by induction, we can show that for arbitrary $k=0,1,\ldots,\tau-1$, the inequality 
	$$f(\widehat{x}^{k+1})-f(x^*) \leq \frac{1}{2}(f(\widehat{x}^{k})-f(x^*))+O\left(\frac{ N_0^{\frac{5}{2}} \sigma^2\sqrt{d\ln\frac{4N_0^2}{\beta}\ln
			\frac{1}{\widehat{\delta}}}}{ n\widehat{\epsilon} L\ln^2\frac{4N_0}{\beta}}\right),$$
	holds with probability at least $1-\beta.$
	Therefore, we have
	\begin{align*}
	f(\widehat{x}^{\tau})-f(x^*) \leq{}& \frac{1}{2^\tau}(f(x^{0})-f(x^*))+O\left(\frac{ N_0^{\frac{5}{2}} \sigma^2\sqrt{d\ln\frac{4N_0^2}{\beta}\ln
			\frac{1}{\widehat{\delta}}}}{ n\widehat{\epsilon} L\ln^2\frac{4N_0}{\beta}}\right)
	\end{align*} 
	holds with probability at least $ 1-\tau\beta.$

	Hence, taking $\tau = O(\frac{L}{\mu}\log n)$ and $N_0=O(\frac{L}{\mu}\ln \frac{L^2}{(\mu\beta)^2}),$ we have
	\begin{align*}
	f(\widehat{x}^{\tau})-f(x^*)
	\leq{}& \Tilde{O}\left(\frac{d^{\frac{1}{2}}L^2 \sqrt{\log n}(\ln\frac{L^2}{\mu^2\beta^2 })^3 \ln\frac{L\log n}{\delta \mu}}{ \mu^3n\epsilon (\ln\frac{L}{\mu\beta})^2}\right)+\frac{f(x^0)-f(x^*)}{n^{\frac{L}{\mu}}}\\
	\leq{}&  \Tilde{O}\left(\frac{d^{\frac{1}{2}}L^2 \sqrt{\log n}\ln\frac{L}{\mu\beta} \ln\frac{L}{\delta \mu}}{ \mu^3n\epsilon}\right)	,
	\end{align*}
	holds with probability at least $1-\tau\beta.$ 
 
    The total gradient complexity is $\Tilde{O}\left(\max\left\{d(\frac{L}{\mu})^2,nd^{\frac{1}{2}}\right\}\right).$
\end{proof}

\subsection{Proof of Theorem \ref{NONS}} \label{appendix for thm7}
\begin{proof}
    The proof is basically different from the above analysis since H$\ddot{\text{o}}$lder continuity introduces other constant terms that need to be bounded. We start it from scratch. Since $f(x)$ is convex and its gradient satisfies inequality (\ref{nonsmooth}), we have the following inequality with assuming $x^k \in B_{7R}$:
	\begin{align*}
	&\|x^{k+1}-x^{*}\|_{2}^2=\|x^{k}-\gamma \widetilde{\nabla} f(x^{k},\bm{\xi}^{k}) -x^{*}\|_{2}^2\\
	\leq{}&\|x^{k}-x^*\|_2^2+\gamma^2\|\widetilde{\nabla}f(x^{k},\bm{\xi}^{k})\|_2^2 - 2\gamma\left\langle x^k-x^*,\widetilde{\nabla}f(x^{k},\bm{\xi}^{k})\right\rangle\\
	\leq{}&\|x^{k}-x^*\|_2^2+2\gamma^2\|\widehat{\nabla}f(x^{k},\bm{\xi}^{k})\|_2^2+2\gamma^2\|z^k\|_2^2 - 2\gamma\left\langle x^k-x^*,\widehat{\nabla}f(x^{k},\bm{\xi}^{k})+z^k\right\rangle\\
	={}&\|x^{k}-x^*\|_2^2+2\gamma^2\|\theta_k+\nabla f(x^k)\|_2^2+2\gamma^2\|z^k\|_2^2 - 2\gamma\left\langle x^k-x^*,\theta_k+\nabla f(x^k)+z^k\right\rangle\\
	\leq{}&\|x^{k}-x^*\|_2^2+4\gamma^2\|\theta_k\|_2^2+4\gamma^2\|\nabla f(x^{k})\|_2^2+2\gamma^2\|z^k\|_2^2\\
	&{}-2\gamma\left\langle x^k-x^*, \theta_k\right\rangle-2\gamma\left\langle x^k-x^*,\nabla f(x^{k})\right\rangle-2\gamma\left\langle x^k-x^*,z^k\right\rangle\\
	\leq{}&\|x^{k}-x^{*}\|_{2}^{2}+2\gamma\left(4\gamma \left(\frac{1}{\alpha}\right)^{\frac{1-\nu}{1+\nu}}M_\nu^{\frac{2}{1+\nu}}-1\right) (f(x^{k})-f(x^{*}))+4\gamma^2\|\theta_k\|_2^2\\
	&{}+2\gamma^2\|z^k\|_2^2-2\gamma\left\langle x^k-x^*,\theta_k\right\rangle-2\gamma\left\langle x^k-x^*,z^k\right\rangle+
	4\gamma^2\alpha^{\frac{2\nu}{1+\nu}}M_\nu^{\frac{2}{1+\nu}},	
	%	\leq{}&\|x^{k}-\gamma{\nabla} f(x^{k},\bm{\xi}^{k}) -x^{*}\|_{2}+\gamma\|\theta_{k}\|_{2}+\gamma \|z_{k}\|_{2}
	\end{align*}
    where $\theta_k = \widehat{\nabla}f(x^{k},\bm{\xi}^{k})-\nabla f(x^{k})$ and the last inequality follows from the convexity of $f$ and Lemma A.5 from \cite{Gorbunov2021near}.

	Letting the constant $\mathcal{M}:= \left(\frac{1}{\alpha}\right)^{\frac{1-\nu}{1+\nu}}M_\nu^{\frac{2}{1+\nu}}$ and  choosing $\gamma$ such that $1-4\gamma \mathcal{M}>0,$ together with the notation $R_k \overset{\text{def}}{=}\|x^k-x^*\|_2,k>0$, we obtain that for all $k\geq0$
	\begin{equation}\label{n555}	\begin{split}
	R_{k+1}^2\leq{}& R_k^{2}+2\gamma\left(4\gamma\mathcal{M}-1\right) (f(x^{k})-f(x^{*}))+4\gamma^2\|\theta_k\|_2^2+2\gamma^2\|z^k\|_2^2\\
	&{}-2\gamma\left\langle x^k-x^*,\theta_k\right\rangle-2\gamma\left\langle x^k-x^*,z^k\right\rangle+
	4\gamma^2\alpha^{\frac{2\nu}{1+\nu}}M_\nu^{\frac{2}{1+\nu}}
	\end{split}
	\end{equation}
    under the assumption $x^k \in B_{7R}$. 	Let us define $A=2\gamma(1-4\gamma \mathcal{M})>0,$ and summing up ($\ref{n555}$) for $k=0,1,\ldots,N-1$,
	we derive that under the assumption $x^k \in B_{7R}$ for $k=0,\cdots,N-1,$

	\begin{equation} \label{888}
	\begin{split}
	    &A\sum_{k=0}^{N-1}(f(x^{k})-f(x^{*}))\\
     \leq{}& R_0^2 -R_N^2+ 4\gamma^2\sum_{k=0}^{N-1}\|\theta_k\|_{2}^2+2\gamma^2\sum_{k=0}^{N-1}\|z^k\|_{2}^2+4\gamma^2N\alpha^{\frac{2\nu}{1+\nu}}M_\nu^{\frac{2}{1+\nu}}\\
	&{}-2\gamma\sum_{k=0}^{N-1}\left\langle x^k-x^*,\theta_k\right\rangle-2\gamma\sum_{k=0}^{N-1}\left\langle x^k-x^*,z^k\right\rangle.
	\end{split}
	\end{equation}

	Taking into account that $A\sum_{k=0}^{N-1}(f(x^{k})-f(x^{*})) \geq 0,$ and changing the indices in (\ref{888}), we get that for all $T=0,\cdots,N$, 

	\begin{equation} \label{RT1}
	\begin{split}
	    R_T^2\leq{}& R_0^2 + 4\gamma^2\sum_{t=0}^{T-1}\|\theta_t\|_{2}^2+2\gamma^2\sum_{t=0}^{T-1}\|z^t\|_{2}^2+4\gamma^2T\alpha^{\frac{2\nu}{1+\nu}}M_\nu^{\frac{2}{1+\nu}}\\
	   &{}-2\gamma\sum_{t=0}^{T-1}\left\langle x^t-x^*,\theta_t\right\rangle-2\gamma\sum_{t=0}^{T-1}\left\langle x^t-x^*,z^t\right\rangle
	\end{split}
	\end{equation}
	under assumption that $x^t \in B_{7R}$ for $t=0,\cdots,T-1.$
	The remaining part of the proof is based on the analysis of inequality ($\ref{RT1}$). In particular, via induction we prove that for all $T=0,\ldots,N$ with probability at least $1-T\beta,$  the following statement holds
	\begin{equation} \label{RT2}
	\begin{split}
	 R_k^2\leq{}& R_0^2 + 4\gamma^2\sum_{t=0}^{k-1}\|\theta_t\|_{2}^2+2\gamma^2\sum_{t=0}^{k-1}\|z^t\|_{2}^2+4\gamma^2k\alpha^{\frac{2\nu}{1+\nu}}M_\nu^{\frac{2}{1+\nu}}\\
	&{}-2\gamma\sum_{t=0}^{k-1}\left\langle x^t-x^*,\theta_t\right\rangle-2\gamma\sum_{t=0}^{k-1}\left\langle x^t-x^*,z^t\right\rangle\\
	\leq{}& C^2R^2 
	\end{split}	
	\end{equation}
	for all $k=0,\ldots,T$ simultaneously where $C$ is defined as $7$. Let us define %$D:=\frac{648 N^3 \sigma^2\sqrt{dN\ln\frac{4N}{\beta}\ln
	%		\frac{1}{\delta}}}{ n\epsilon\ln^2\frac{4}{\beta}},$ and 
	the probability event when  statement ($\ref{RT2}$) holds as $E_T.$ Then, our goal is to show that $\mathbb{P}(E_{T})\geq 1-T\beta$ for all $T=0,\ldots,N.$ Clearly, when $T=0,$ inequality ($\ref{RT2}$) holds with probability 1. Next, assuming
	that for $T\leq N-1,$ we have $\mathbb{P}(E_{T})\geq 1-T\beta.$ Let us prove that $\mathbb{P}(E_{T+1})\geq 1-(T+1)\beta.$

	First of all, the probability event
	$E_T$ implies that $x^t \in B_{7R}$ for $t=0,\cdots,T.$
	Since $f(\cdot)$ is $(\nu,M_\nu)$-H$\ddot{\text{o}}$lder continuous on $B_{7R}(x^*)$, probability event $E_T$ implies that
	\begin{equation*}
	\|\nabla f(x^{t})\|_{2} \leq M_\nu\|x^t-x^{*}\|_2^\nu  \leq M_\nu C^\nu R^\nu \leq \frac{\lambda}{2},
	\end{equation*}
	holds for $t=0,\ldots,T.$

	Next, we introduce new random variables:
	\begin{equation}
	    \eta_t = \begin{cases}
	    x^*-x^t, &\text{if}\ \|x^*-x^t\|_2 \leq CR\\
	    0, & \text{otherwise}
	    \end{cases}
	\end{equation}
	for $t=0,\ldots,T.$ Note that these random variables are bounded with probability 1, i.e. with probability 1 we have
    \begin{equation} \label{inner_nons}
        \|\eta_t\|_2 \leq CR.
    \end{equation}

	Using the introduced notation, let $\gamma \leq \frac{R}{2\sqrt{N}\alpha^{\frac{\nu}{1+\nu}}M_\nu^{\frac{1}{1+\nu}}},$ we obtain that event $E_{T+1}$ implies
	\begin{equation} 
	 R_{T+1}^2\leq{} 2R_0^2 + 4\gamma^2\sum_{t=0}^{T}\|\theta_t\|_{2}^2+2\gamma^2\sum_{t=0}^{T}\|z^t\|_{2}^2+
	2\gamma\sum_{t=0}^{T}\left\langle \eta_t,\theta_t\right\rangle+2\gamma\sum_{t=0}^{T-1}\left\langle \eta_t,z^t\right\rangle.
	\end{equation}

	Finally, we do some preliminaries in order to apply Bernstein's inequality and rewrite:
	\begin{equation}
	\begin{split}
	 &R_{T+1}^2\\
    \leq{}& 2R_0^2 + 8\gamma^2\underbrace{\sum_{t=0}^{T}\left(\|\theta_t^a\|_{2}^2-\mathbb{E}_{\bm{\xi}^{t}}[\|\theta_{t}^a\|_{2}^2]\right)}_{\circled{1}}+8\gamma^2\underbrace{\sum_{t=0}^{T}\left(\mathbb{E}_{\bm{\xi}^{t}}[\|\theta_{t}^a\|_{2}^2]\right)}_{\circled{2}}+8\gamma^2\underbrace{\sum_{t=0}^{T}\|\theta_{t}^b\|_{2}^2}_{\circled{3}}\\
	&{}+2\gamma\underbrace{\sum_{t=0}^{T}\left\langle \eta_t,\theta_t^a\right\rangle}_{\circled{4}}+2\gamma\underbrace{\sum_{t=0}^{T}\left\langle \eta_t,\theta_t^b\right\rangle}_{\circled{5}}+2\gamma\underbrace{\sum_{t=0}^{T}\left\langle \eta_t,z^t\right\rangle}_{\circled{6}}+2\gamma^2\underbrace{\sum_{t=0}^{T}\|z^t\|_{2}^2}_{\circled{7}}\\
	\end{split}	
	\end{equation}
    where we introduce new notations:
	$$\theta_{t}^a \overset{\text{def}}{=}\widehat{\nabla}f(x^{t},\bm{\xi}^{t})-\mathbb{E}_{\bm{\xi}^{t}}\left[\widehat{\nabla}f(x^{t},\bm{\xi}^{t})\right],\       	
	\theta_{t}^b \overset{\text{def}}{=} \mathbb{E}_{\bm{\xi}^{t}}\left[\widehat{\nabla}f(x^{t},\bm{\xi}^{t})\right]-\nabla f(x^{t}).$$

    Upper bound for $\circled{1}$. First of all,  we notice that the terms in $\circled{1}$ are conditionally unbiased:
	$$\mathbb{E}_{\bm{\xi}^{t}}\Bigl[\|\theta_{t}^a\|_{2}^2-\mathbb{E}_{\bm{\xi}^{t}}[\|\theta_{t}^a\|_{2}^2]\Bigl]=0.$$
	
	Secondly, the terms are bounded with probability 1:
	\begin{equation*}
	\bigl|\|\theta_{t}^a\|_{2}^2-\mathbb{E}_{\bm{\xi}^{t}}[\|\theta_{t}^a\|_{2}^2]\bigl| \leq \|\theta_{t}^a\|_{2}^2+\mathbb{E}_{\bm{\xi}^{t}}[\|\theta_{t}^a\|_{2}^2]\leq8\lambda^2\overset{\text{def}}{=} c.
	\end{equation*}

	Finally, we can bound the conditional variances $\bar{\sigma}_{t}^2= \mathbb{E}_{\bm{\xi}^{t}}\left[\bigl|\|\theta_{t}^a\|_{2}^2-\mathbb{E}_{\bm{\xi}^{t}}[\|\theta_{t}^a\|_{2}^2]\bigl|^2\right]$ as follows:
	\begin{equation*}
	\bar{\sigma}_t^2 \leq   c\mathbb{E}_{\bm{\xi}^{t}}\left[\bigl|\|\theta_{t}^a\|_{2}^2-\mathbb{E}_{\bm{\xi}^{t}}[\|\theta_{t}^a\|_{2}^2]\bigl|\right]
	\leq c\mathbb{E}_{\bm{\xi}^{t}}\left[\|\theta_{t}^a\|_{2}^2+\mathbb{E}_{\bm{\xi}^{t}}[\|\theta_{t}^a\|_{2}^2\right]
	= 2c\mathbb{E}_{\bm{\xi}^{t}}[\|\theta_{t}^a\|_{2}^2].
	\end{equation*}

	Thus, the sequence $\left\{ \|\theta_{t}^a\|_{2}^2-\mathbb{E}_{\bm{\xi}^{t}}[\|\theta_{t}^a\|_{2}^2]  \right\}_{t\geq 0}$ is a bounded martingale differences sequence with bounded conditional variances $\{\bar{\sigma}_t^2\}_{t\geq0}$. Therefore, we apply Bernstein's inequality \citep{freedman1975tail,dzhaparidze2001bernstein} with $X_t = \|\theta_{t}^a\|_{2}^2-\mathbb{E}_{\bm{\xi}^{t}}[\|\theta_{t}^a\|_{2}^2]$, $c=8\lambda^2$ \text{and} $F=\frac{c^2\ln\frac{4}{\beta}}{6}$ and get for all $b>0,$ it holds that
	$$\mathbb{P}\Bigl\{\sum_{t=0}^{T}\bigl|X_t\bigl| >b \  \text{and} \ \sum_{t=0}^{T}\bar{\sigma}_t^2 \leq F\Bigl\} \leq 2\text{exp}\left(-\frac{b^2}{2F+2cb/3}\right).$$

	Thus, with probability at least $1-2\text{exp}\left(-\frac{b^2}{2F+2cb/3}\right),$ we have
	\begin{center}
		either \ $\sum_{t=0}^{T}\bigl|X_t\bigl| \leq b$ \ or \  $\sum_{t=0}^{T}\bar{\sigma}_t^2 > F.$
	\end{center}

	Here we choose $b$ in a way such that $2\text{exp}\left(-\frac{b^2}{2F+2cb/3}\right)\leq\frac{\beta}{4}.$ This implies that $$b^2-\frac{2c\ln\frac{8}{\beta}}{3}b-2F\ln\frac{8}{\beta}\geq 0.$$
	
	Hence, $$b \geq \frac{c\ln\frac{8}{\beta}}{3}+\sqrt{\frac{c^2\ln^2\frac{8}{\beta}}{9}+2F\ln\frac{8}{\beta}}.$$

    Next, in order to bound $\sum_{t=0}^{T}\bar{\sigma}_t^2$ with probability 1, we have the following inequality for $F$
	\begin{align*}
	    \sum_{t=0}^{T}\bar{\sigma}_t^2 \leq{}&  2c\sum_{t=0}^{T}\mathbb{E}_{\bm{\xi}^{t}}[\|\theta_{t}^a\|_{2}^2]\overset{(\ref{variance bound})}{\leq}{}36c\sigma^2\sum_{t=0}^{T}\frac{1}{m}\\
	    ={}& 36c\sigma^2\frac{T+1}{m}
	     \overset{T \leq N-1}{\leq} \frac{c^2\ln\frac{8}{\beta}}{6}\\ ={}& F,
	\end{align*}
    where the quoted inequality (\ref{variance bound}) still works, i.e., $\mathbb{E}_{\bm{\xi}}\left[\left\|\widehat{\nabla}f(x,\bm{\xi})-\mathbb{E}_{\bm{\xi}}\left[\widehat{\nabla}f(x,\bm{\xi})\right]\right\|_{2}^2\right] \leq \frac{18\sigma^2}{m}.$
	Therefore, we have shown that with probability at least $1-\frac{\beta}{4}$, $\sum_{t=0}^{T}\bigl|X_t\bigl| \leq b,$ i.e.,
	$$\sum_{t=0}^{T}\bigl|\|\theta_{t}^a\|_{2}^2-\mathbb{E}_{\bm{\xi}^{t}}\left[\|\theta_{t}^a\|_2^2\right]\bigl| \leq b,$$ where	
	$b=c\ln\frac{8}{\beta}=8\lambda^2\ln\frac{8}{\beta}$ as desired.

	Upper bound for $\circled{2}$.
	\begin{equation*}
	\begin{split}
	\circled{2} =&{} \sum_{t=0}^{T}\mathbb{E}_{\bm{\xi}^{t}}\left[\|\theta_{t}^a\|_{2}^2\right]
	\leq{}18\sigma^2\sum_{t=0}^{T}\frac{1}{m}\\
	=&{}18\sigma^2\frac{T+1}{m} \overset{T\leq N-1}{\leq} 18\sigma^2\frac{N}{m}\\
	=&{}\frac{2\lambda^2\ln \frac{8}{\beta}}{3}.
	\end{split} 	
	\end{equation*}

	Upper bound for $\circled{3}$.
	\begin{equation*}
	\begin{split}
	\circled{3} =&{}\sum_{t=0}^{T}\|\theta_{t}^b\|_{2}^2
	\leq{} 16\sigma^4\sum_{t=0}^{T}\frac{1}{m^2\lambda^2}\\
	=&{} 16\sigma^4\frac{T+1}{m^2\lambda^2} \overset{T\leq N-1}{\leq} 16\sigma^4\frac{N}{m^2\lambda^2}\\
	=&{} \frac{16\lambda^2 \ln^2\frac{8}{\beta}}{729N}.
	\end{split} 	
	\end{equation*}

	Upper bound for $\circled{4}$. First of all,  we notice that the $\mathbb{E}_{\bm{\xi}^{t}}[\theta_t^a]=0$ summands in $\circled{4}$ are conditionally unbiased:
	$$\mathbb{E}_{\bm{\xi}^{t}}\Bigl[\left\langle \eta_t,\theta_t^a\right\rangle\Bigl]=0.$$
	
	Secondly, the terms are bounded with probability 1:
	\begin{equation*}
	\bigl|\left\langle \eta_t,\theta_t^a\right\rangle\bigl| \leq \|\eta_t\|_2\|\theta_{t}^a\|_{2}\overset{(\ref{inner_nons})}{\leq}2\lambda CR.
	\end{equation*}
	
	Finally, we can bound the conditional variances $\tilde{\sigma}_{t}^2= \mathbb{E}_{\bm{\xi}^{t}}\left[\left\langle \eta_t,\theta_t^a\right\rangle^2\right]$ as follows:
	\begin{equation*}
	\tilde{\sigma}_t^2 \leq   \mathbb{E}_{\bm{\xi}^{t}}\left[\|\eta_t\|_2^2\|\theta_{t}^a\|_{2}^2\right]
	\leq  (CR)^2\mathbb{E}_{\bm{\xi}^{t}}\left[\|\theta_{t}^a\|_{2}^2\right].
	\end{equation*}

	Thus, the sequence $\left\{ \left\langle \eta_t,\theta_t^a\right\rangle  \right\}_{t\geq 0}$ is a bounded martingale differences sequence with bounded conditional variances $\{\tilde{\sigma}_t^2\}_{t\geq0}$. Therefore, we can apply Bernstein's inequality with $X_t = \left\langle \eta_t,\theta_t^a\right\rangle$, $c_1=2\lambda CR$ \text{and} $F=\frac{c_1^2\ln\frac{4}{\beta}}{6}$ and get for all $b>0,$ it holds that
	$$\mathbb{P}\Bigl\{\sum_{t=0}^{T}\bigl|X_t\bigl| >b \  \text{and} \ \sum_{t=0}^{T}\tilde{\sigma}_t^2 \leq F\Bigl\} \leq 2\text{exp}\left(-\frac{b^2}{2F+2c_1b/3}\right).$$

	Thus, with probability at least $1-2\text{exp}\left(-\frac{b^2}{2F+2c_1b/3}\right),$ we have
	\begin{center}
		either \ $\sum_{t=0}^{T}\bigl|X_t\bigl| \leq b$ \ or \  $\sum_{t=0}^{T}\tilde{\sigma}_t^2 > F.$
	\end{center}

	Here we choose $b$ in a way such that $2\text{exp}\left(-\frac{b^2}{2F+2c_1b/3}\right)\leq\frac{\beta}{4}.$ This implies that $$b^2-\frac{2c_1\ln\frac{8}{\beta}}{3}b-2F\ln\frac{8}{\beta}\geq 0.$$ 

	Hence, 
	$$b \geq \frac{c_1\ln\frac{8}{\beta}}{3}+\sqrt{\frac{c_1^2\ln^2\frac{8}{\beta}}{9}+2F\ln\frac{8}{\beta}}.$$

    Next, in order to bound $\sum_{t=0}^{T}\tilde{\sigma}_t^2$ with probability 1, we have the following inequality for $F$
	\begin{align*}	    \sum_{t=0}^{T}\tilde{\sigma}_t^2 \leq{}&  (CR)^2\sum_{t=0}^{T}\mathbb{E}_{\bm{\xi}^{t}}[\|\theta_{t}^a\|_{2}^2]
	    \leq{}18(CR)^2\sigma^2\sum_{t=0}^{T}\frac{1}{m}\\
	    ={}& 18(CR)^2\sigma^2\frac{T+1}{m}
	     \overset{T \leq N-1}{\leq} \frac{c_1^2\ln\frac{8}{\beta}}{6}\\ ={}& F,
	\end{align*}
	where $c_1 = 2\lambda CR.$

	Therefore, we have shown that with probability at least $1-\frac{\beta}{4}$, $\sum_{t=0}^{T}\bigl|X_t\bigl| \leq b,$ i.e.,
	$$\sum_{t=0}^{T}\bigl|\left\langle \eta_t,\theta_t^a\right\rangle\bigl| \leq b,$$ where	
	$b=c_1\ln\frac{8}{\beta}=2\lambda CR\ln\frac{8}{\beta}$ as desired.

	Upper bound for $\circled{5}.$ 
	\begin{equation*}
	\bigl|\left\langle \eta_t,\theta_t^b\right\rangle\bigl| \leq \|\eta_t\|_2\|\theta_{t}^b\|_{2}\overset{(\ref{bias bound})}{\leq} \frac{4\sigma^2}{m\lambda} CR_0.
	\end{equation*}

	Thus 
	\begin{equation*}
	   \circled{5} = \sum_{t=0}^{T}\left\langle \eta_t,\theta_t^b\right\rangle \leq \sum_{t=0}^{T}\bigl|\left\langle \eta_t,\theta_t^b\right\rangle\bigl| \overset{T\leq N-1}{\leq }\frac{4\sigma^2 CRN}{m\lambda} \leq \frac{4\lambda CR \ln\frac{8}{\beta}}{27}.
	\end{equation*}

	Upper bound for $\circled{6}.$ 
	\begin{equation*}
	\bigl|\left\langle \eta_t,z^t\right\rangle\bigl| \leq \|\eta_t\|_2\|z^t\|_{2} \leq  CR\|z^t\|_{2}.
	\end{equation*}
	
	Thus 
	\begin{equation*}
	   \circled{6} = \sum_{t=0}^{T}\left\langle \eta_t,z^t\right\rangle \leq \sum_{t=0}^{T}\bigl|\left\langle \eta_t,z^t\right\rangle\bigl| \leq CR\sum_{t=0}^{T}\|z^t\|_2
	\end{equation*}

	Next we bound the $\sum_{t=0}^{T}\|z^t\|_2.$
	It is easy to verify that the injected Gaussian random noise $z \in \mathbb{R}^d, z \sim \mathcal{N}(0,\widehat{\sigma}^2 I^d)$ is subG($\widehat{\sigma}^2$). Thus the tail bound of the vector can be achieved by the Lemma 1 in \cite{jin2019short}:
	\begin{equation*}
	\mathbb{P}\left(\|z\|_2 \geq \widehat{t}\right) \leq 2e^{-\frac{\widehat{t}^2}{2\widehat{c}^2 d\widehat{\sigma}^2}},
	\end{equation*}
	where constant $\widehat{c}$  is selected as $2\sqrt{2}$ as the derivations in \cite{jin2019short}.

    To fit in our case, we choose $\widehat{t}$ such that $2\text{exp}\left[-\frac{\widehat{t}^2}{2\widehat{c}^2 d\widehat{\sigma}^2}\right]=\frac{\beta}{4(T+1)}$.
	That is, we choose $\widehat{t}= \widehat{\sigma}\sqrt{16d \ln\frac{8(T+1)}{\beta}}$ so that with the probability at least $1-\frac{\beta}{4}$,
	$$\sum_{t=0}^{T}\|z^t\|_2 \leq  (T+1)\widehat{\sigma}\sqrt{16d \ln\frac{8(T+1)}{\beta}}.$$

	Therefore, we bound the \circled{6} with the probability at least $1-\frac{\beta}{4},$
	\begin{equation*}
	\begin{split}
	   \circled{6} =& \sum_{t=0}^{T}\left\langle \eta_t,z^t\right\rangle \leq \sum_{t=0}^{T}\bigl|\left\langle \eta_t,z^t\right\rangle\bigl| \leq CR\sum_{t=0}^{T}\|z^t\|_2\\ 
	   \leq{}& CR (T+1)\widehat{\sigma}\sqrt{16d \ln\frac{8T}{\beta}} \overset{T\leq N-1}{\leq} \frac{108CRN^{2.5}\sigma^2\sqrt{d\ln\frac{8N}{\beta}\ln\frac{1}{\delta}}}{n\epsilon\lambda\ln\frac{8}{\beta}}\\
	   ={}& \frac{CRD N}{\lambda}
	\end{split}
	\end{equation*}
	where $D \overset{\text{def}}{=}\frac{108N^{1.5}\sigma^2\sqrt{d\ln\frac{8N}{\beta}\ln\frac{1}{\delta}}}{n\epsilon\ln\frac{8}{\beta}}.$

	Upper bound for \circled{7}.
	Recall we have shown that with probability at least $1-\frac{\beta}{4(T+1)},$
	$$\|z^t\|_2 \leq  \widehat{\sigma}\sqrt{16d \ln\frac{8(T+1)}{\beta}}.$$
 
	Therefore, we have that with probability at least $1-\frac{\beta}{4},$
	$$\sum_{t=0}^{T}\|z^t\|_2^2 \leq (T+1)\widehat{\sigma}^216d \ln\frac{8(T+1)}{\beta}\overset{T\leq N-1}{\leq} \frac{27^2*16N^4\sigma^4d\ln\frac{8N}{\beta}\ln\frac{1}{\beta}}{n^2\epsilon^2\lambda^2\ln^2\frac{8}{\beta}}=\frac{D^2N}{\lambda^2}.$$

	Finally, letting $\gamma = \min\left\{\frac{\alpha^{\frac{1-\nu}{1+\nu}}}{8M_\nu^{\frac{2}{1+\nu}}},\frac{R}{\sqrt{2N}\alpha^{\frac{\nu}{1+\nu}}M_\nu^{\frac{1}{1+\nu}}},\frac{R}{2\lambda\ln\frac{8}{\beta}},\frac{\lambda R}{2DN}\right\}$ and summarizing all the above bounds, we have derived

	\begin{equation*}
	\begin{split}
	 R_{T+1}^2\leq{}& 2R_0^2 + \underbrace{8\gamma^2\sum_{t=0}^{T}\left(\|\theta_t^a\|_{2}^2-\mathbb{E}_{\bm{\xi}^{t}}[\|\theta_{t}^a\|_{2}^2]\right)}_{\circled{1}}+\underbrace{8\gamma^2\sum_{t=0}^{T}\left(\mathbb{E}_{\bm{\xi}^{t}}[\|\theta_{t}^a\|_{2}^2]\right)}_{\circled{2}}+\underbrace{8\gamma^2\sum_{t=0}^{T}\|\theta_{t}^b\|_{2}^2}_{\circled{3}}\\
	&{}+\underbrace{2\gamma\sum_{t=0}^{T}\left\langle \eta_t,\theta_t^a\right\rangle}_{\circled{4}}+\underbrace{2\gamma\sum_{t=0}^{T}\left\langle \eta_t,\theta_t^b\right\rangle}_{\circled{5}}+\underbrace{2\gamma\sum_{t=0}^{T}\left\langle \eta_t,z^t\right\rangle}_{\circled{6}}+\underbrace{2\gamma^2\sum_{t=0}^{T}\|z^t\|_{2}^2}_{\circled{7}},\\
	&{} E_1:\mathbb{P}\left\{\circled{1} \leq 64\gamma^2\lambda^2\ln\frac{8}{\beta}\right\} \geq 1-\frac{\beta}{4},\\
	&{} E_2:\circled{2} \leq\frac{16\gamma^2\lambda^2\ln \frac{8}{\beta}}{3},\ E_3:\circled{3} \leq \frac{128\gamma^2\lambda^2 \ln^2\frac{8}{\beta}}{729N},\\
	&{}E_4:\mathbb{P}\left\{\circled{4} \leq 4\gamma\lambda CR_0\ln\frac{8}{\beta}\right\} \geq 1-\frac{\beta}{4},\ E_5: \circled{5}\leq \frac{8\gamma\lambda CR_0 \ln\frac{8}{\beta}}{27},\\
	&{}E_6:\mathbb{P}\left\{\circled{6} \leq \frac{2\gamma CR_0D N}{\lambda}\right\} \geq 1-\frac{\beta}{4},\  E_7:\mathbb{P}\left\{\circled{7} \leq \frac{2\gamma^2D^2N}{\lambda^2}\right\} \geq 1-\frac{\beta}{4}.
	\end{split}	
	\end{equation*}

	Taken into account these inequalities and our assumptions on $\lambda$ and $\gamma,$ we have for each event $E_i,\   \bigcap_{i\in \{1,4,6,7\}}E_{i}\bigcap E_{T}$ implies

	\begin{equation}
	\begin{split}
	    R_{T+1}^2 \leq{}& R_0^2 + 4\gamma^2\sum_{k=0}^{T}\|\theta_k\|_{2}^2+2\gamma^2\sum_{k=0}^{T}\|z^k\|_{2}^2+4\gamma^2T\alpha^{\frac{2\nu}{1+\nu}}M_\nu^{\frac{2}{1+\nu}}\\
	&{}-2\gamma\sum_{k=0}^{T}\left\langle x^k-x^*,\theta_k\right\rangle-2\gamma\sum_{k=0}^{T}\left\langle x^k-x^*,z^k\right\rangle\\
	\leq{}& 2R_0^2+ \left(\frac{8}{49}+\frac{2}{147}+\frac{32}{35721}+\frac{2}{7}+\frac{4}{189}+\frac{1}{7}+\frac{1}{98}\right)C^2R^2\\
	\leq{}& C^2R^2.
	\end{split}
	\end{equation}

	Moreover, using the union bound we derive
	\begin{equation*}
	\mathbb{P}\left(\bigcap_{i\in  \{1,4,6,7\}}E_{i}\bigcap E_{T}\right)=1-\mathbb{P}\left(\bigcap_{i\in \{1,4,6,7\}}\bar{E_{i}}\bigcap \bar{E_{t}}\right)\geq 1-(T+1)\beta.   
	\end{equation*}

	That is, by definition of $E_{T+1}$ and $E_{T}$ we have proved that 
	\begin{equation*}
	\mathbb{P}(\mathbb{E}_{T+1})\geq \mathbb{P}\left(\bigcap_{i\in  \{1,4,6,7\}}E_{i}\bigcap E_{T}\right)\geq 1-(T+1)\beta,	    
	\end{equation*}	
	which implies that for all $T=0,\cdots,N$ we have $\mathbb{E}_{T} \geq 1-T\beta.$ Then for $T=N,$ we have that with probability at least $1-N\beta,$
    \begin{equation} 
	\begin{split}
	    AN(f(\bar{x}^{N})-f(x^{*}))\leq{}& A\sum_{k=0}^{N-1}(f(x^{k})-f(x^{*}))\\
	    \overset{(\ref{888})}{\leq}{}& 2R_0^2 + 4\gamma^2\sum_{k=0}^{N-1}\|\theta_k\|_{2}^2+2\gamma^2\sum_{k=0}^{N-1}\|z^k\|_{2}^2\\
	&{}-2\gamma\sum_{k=0}^{N-1}\left\langle x^k-x^*,\theta_k\right\rangle-2\gamma\sum_{k=0}^{N-1}\left\langle x^k-x^*,z^k\right\rangle\\
    \overset{(\ref{RT2})}{\leq}{}& C^2R^2,
	\end{split}
	\end{equation}
	where the first inequality follows from Jensen's inequality that $f(\bar{x}^{N})=f(\frac{1}{N}\sum_{k=0}^{N-1}x^k)\leq \frac{1}{N}\sum_{k=0}^{N-1}f(x^k).$ 
    Since $A=2\gamma(1-4\gamma \mathcal{M})>
    \gamma,$ we get that with probability at least $1-N\beta,$ 
    \begin{equation*}
        f(\bar{x}^{N})-f(x^{*})\leq{}\frac{C^2R^2}{AN} \leq{} \frac{C^2R^2}{\gamma N}.
    \end{equation*}

    When $\gamma = \min\left\{\frac{\alpha^{\frac{1-\nu}{1+\nu}}}{8M_\nu^{\frac{2}{1+\nu}}},\frac{R}{\sqrt{2N}\alpha^{\frac{\nu}{1+\nu}}M_\nu^{\frac{1}{1+\nu}}},\frac{R}{2\lambda\ln\frac{8}{\beta}},\frac{\lambda R}{2DN}\right\},\ \lambda=2M_\nu C^\nu R^\nu,\ m=\max\left\{1,\frac{27N\sigma^2}{\lambda^2\ln\frac{8}{\beta}}\right\},$
    we have that with probability at least $1-N\beta$,
    \begin{equation*}
    \begin{split}
        f(\bar{x}^{N})-f(x^{*})\leq{}\max\Biggl\{&\frac{8C^2M_\nu^{\frac{2}{1+\nu}}R^2}{\alpha^{\frac{1-\nu}{1+\nu}}N},\frac{\sqrt{2}C^2M_\nu^{\frac{1}{1+\nu}}R\alpha^{\frac{\nu}{1+\nu}}}{\sqrt{N}},\\
        &\frac{4C^{2+\nu}M_\nu R^{1+\nu}\ln\frac{8}{\beta}}{N},\frac{108C^{2-\nu}R^{1-\nu}N^{1.5}\sigma^2\sqrt{d\ln\frac{8N}{\beta}\ln\frac{1}{\delta}}}{n M_\nu\epsilon\ln\frac{8}{\beta}}\Biggl\}.
    \end{split}
    \end{equation*}

    If taking $N=\Tilde{O}\left(\frac{n\epsilon}{\sqrt{d\ln\frac{1}{\delta}}}\right)^{\frac{2}{7}},$ we have that with probability at least $1-N\beta,$
    \begin{equation*}
    \begin{split}
        f(\bar{x}^{N})-f(x^{*})\leq&{}\Tilde{O}\Biggl(\max\Biggl\{\frac{M_\nu^{\frac{2}{1+\nu}}(d\ln\frac{1}{\delta})^{\frac{1}{7}}}{\alpha^{\frac{1-\nu}{1+\nu}}(n\epsilon)^{\frac{2}{7}}},\frac{M_\nu (d\ln\frac{1}{\delta})^{\frac{1}{7}}}{(n\epsilon)^{\frac{2}{7}}}\\
        &\frac{M_\nu^{\frac{1}{1+\nu}}\alpha^{\frac{\nu}{1+\nu}}(d\ln\frac{1}{\delta})^{\frac{1}{14}}}{(n\epsilon)^{\frac{1}{7}}},\frac{(d\ln\frac{1}{\delta})^{\frac{2}{7}}\sqrt{\ln\frac{(n\epsilon)^2}{\beta d}}}{(n\epsilon)^{\frac{4}{7}} M_\nu}\Biggl\}\Biggl).
    \end{split}
    \end{equation*}
    
    The total gradient complexity is $\Tilde{O}\left(\max \left\{n^{\frac{2}{7}}d^{\frac{6}{7}},n^{\frac{4}{7}}d^{\frac{5}{7}}\right\}\right),$ and the Big-$\Tilde{O}$ notation here omits other logarithmic factors and the terms $\sigma,\alpha, R.$
\end{proof}

\section{Detailed description of experiments}\label{secB1}
We validate our proposed AClipped-dpSGD using real data (\emph{Adult} and \emph{Diabetes} datasets), and synthetic data (drawn from \textit{Student's t-distribution} (degree = 2), \textit{Laplace distribution} ($\mu =1, \sigma = 1$), and \textit{Chi-squared distribution} (degree=1)) for ridge and logistic regression tasks. Among them, \emph{Adult} data contains 32,561 samples with 123 dimensions. We select 21,000 samples for training and the rest for testing.  \emph{Diabetes} data contains 768 samples with 8 dimensions, where we select 500 samples for training.
We evaluate the output errors (divided by initial error) and running time after 30 epochs (real data) and 400 epochs (synthetic data) of AClipped-dpSGD, DP-SGD from \cite{das2023beyond} and DP-GD from \cite{wang2020differentially}. We present the additional experiments that measure the deviations across multiple repetitions. The privacy levels are $\epsilon = \{0.5,0.75,1.0,2.0\}$ and $\delta = n^{-1}.$

\begin{figure*}[t]
	\centering
    %\captionsetup[subfigure]{labelformat=empty}
	\subfigure{
		\includegraphics[width=0.23\textwidth]{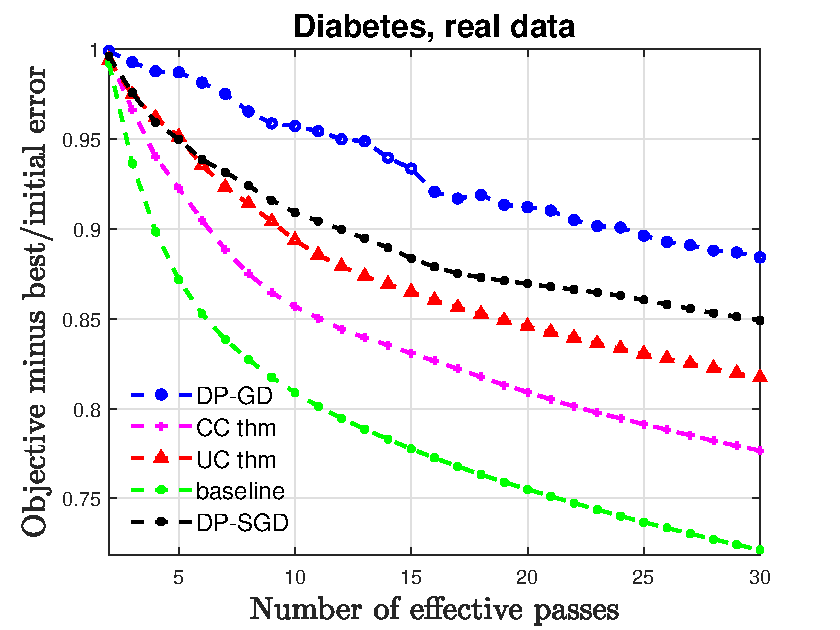}}
	\subfigure{
		\includegraphics[width=0.23\textwidth]{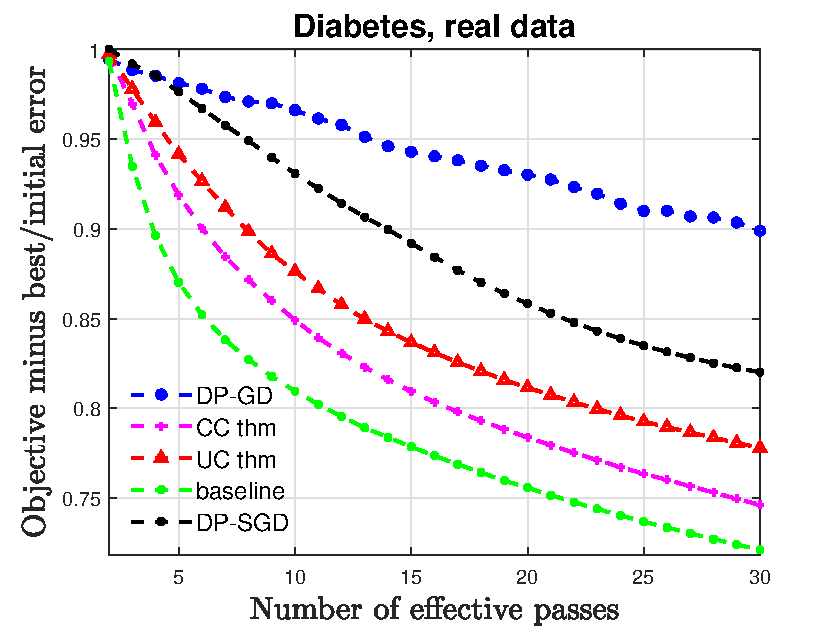}}
	\subfigure{
		\includegraphics[width=0.23\textwidth]{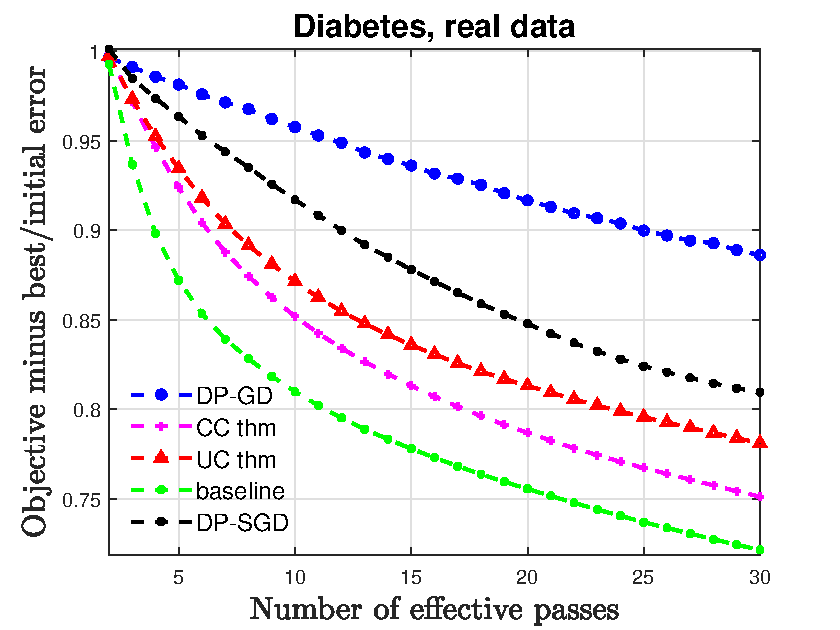}}
	\subfigure{
		\includegraphics[width=0.23\textwidth]{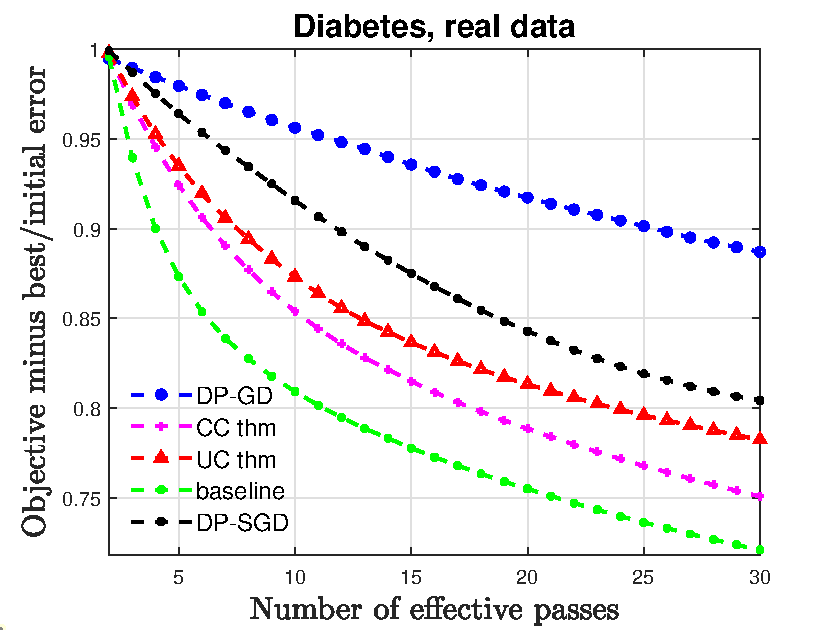}}
  \newline
   \setcounter{subfigure}{0}
  	\subfigure[$\epsilon = 0.5$]{
		\includegraphics[width=0.23\textwidth]{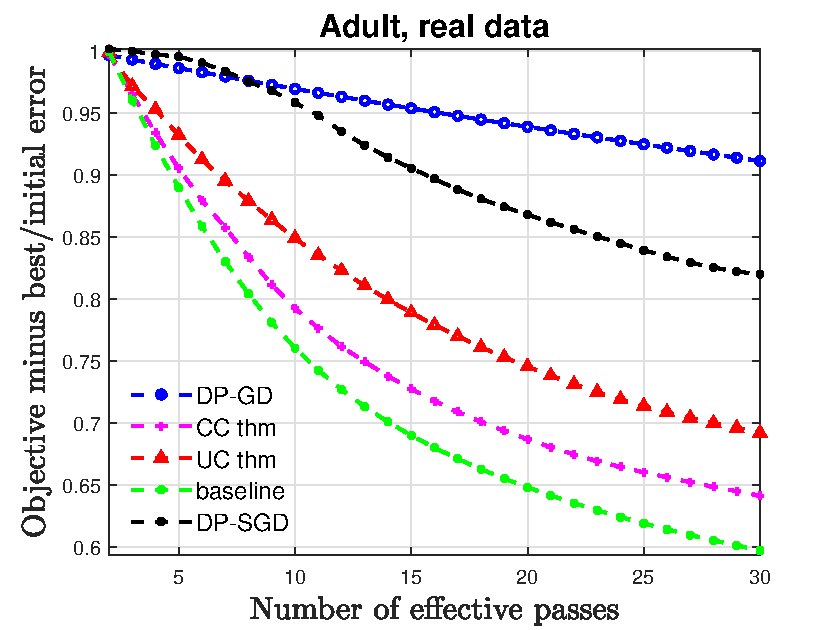}}
	\subfigure[$\epsilon = 0.75$]{
		\includegraphics[width=0.23\textwidth]{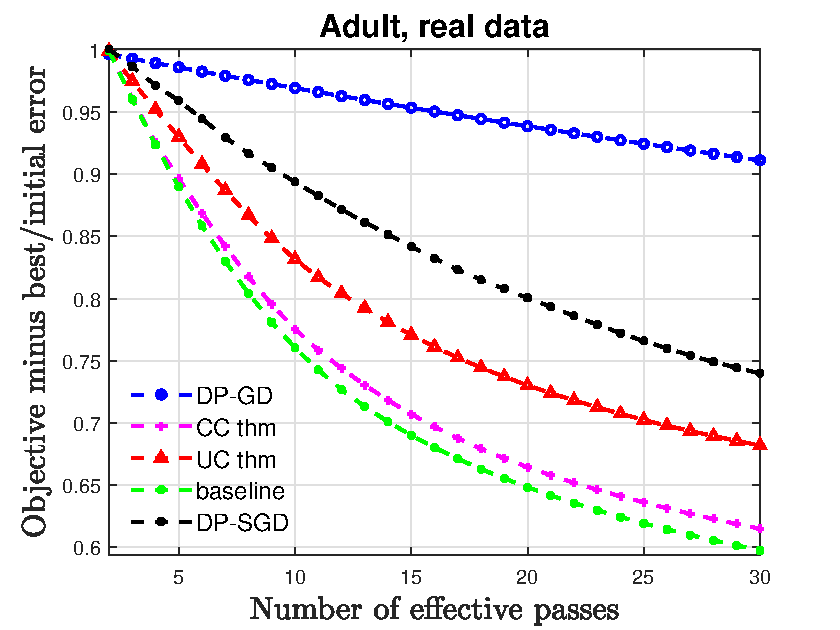}}
	\subfigure[$\epsilon = 1.0$]{
		\includegraphics[width=0.23\textwidth]{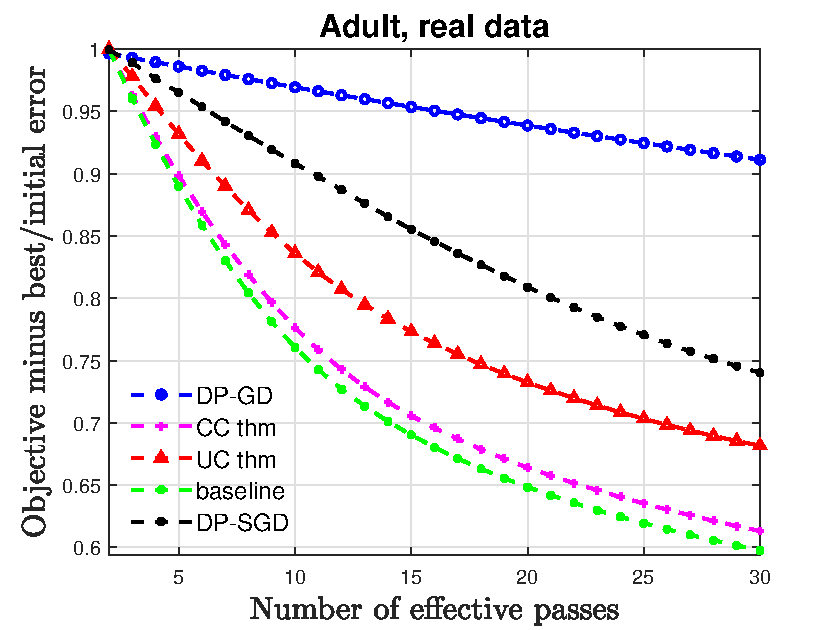}}
	\subfigure[$\epsilon = 2.0$]{
		\includegraphics[width=0.23\textwidth]{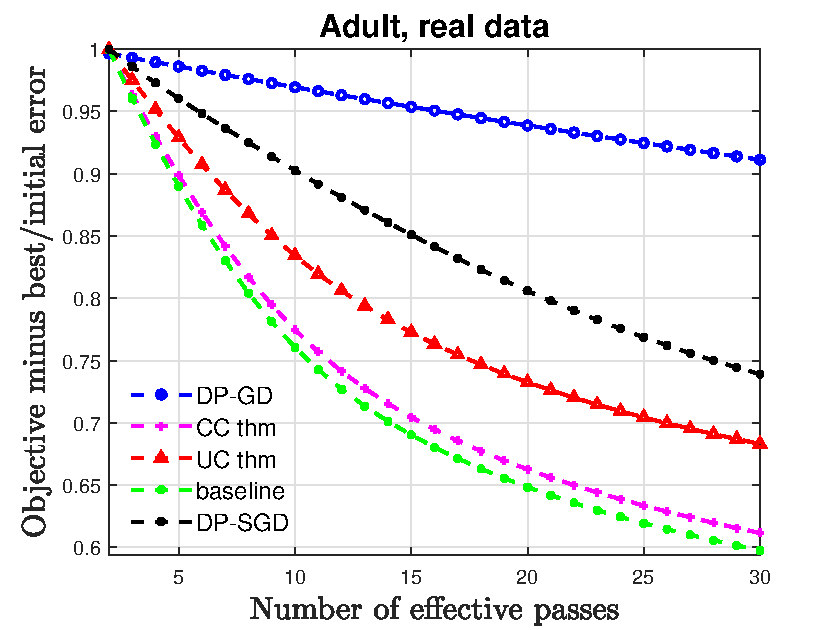}}
	\caption{Trajectories of the ridge regression model for the real-world data. The top and bottom rows correspond to the \textit{Diabetes} and \textit{Adult} datasets, respectively.}
    \label{fig3}
\end{figure*}

\begin{table*}[t]
    \caption{Summary of experimental results of private methods on the regression model. Bold values mark the results of our methods.}
	\centering	
    \resizebox{\textwidth}{!}{
	\begin{tabular}{ccccccc|cccc}
		\toprule
		\multirow{2}{*}{Model} &\multirow{2}{*}{Data} &\multirow{2}{*}{$\epsilon$} & \multicolumn{4}{c}{Estimation Error}&\multicolumn{4}{|c}{Runtime (wall clock time): $s$}\\
		\cmidrule{4-7} \cmidrule{8-11}
		
		& & & CC & UC &DP-SGD & DP-GD & CC & UC &DP-SGD& DP-GD \\
		\hline
		
		%Teacher &4.12G	&0.87841	&0.89512	&0.9014	&0.83587	&0.85736	&0.86297\\ 
		%\hline
		
		%0.25x-----------------
		\multirow{20}{*}{RR} &\multirow{4}{*}{\emph{Diabetes}}& 0.5  &
		\textbf{0.7550} & \textbf{0.7913} &	0.8211 &0.8987	 &\textbf{3.67E-4} & \textbf{3.58E-4} &4.62E-4&5.51E-3 \\
		%\cmidrule{2-4} \cmidrule{5-6} \cmidrule{7-8}

        \multirow{20}{*}{} &\multirow{4}{*}{} &0.75	&\textbf{0.7481}	&\textbf{0.7782}	&0.8179 &0.8951 &\textbf{4.34E-4} &\textbf{3.76E-4}	&4.56E-4	&5.63E-3\\

		\multirow{20}{*}{} &\multirow{4}{*}{} &1.0	&\textbf{0.7479}	&\textbf{0.7752}	&0.8074	&0.8844	&\textbf{3.97E-4}	&\textbf{3.56E-4} &4.44E-4 &5.60E-3\\

        \multirow{20}{*}{} &\multirow{4}{*}{} &2.0	&\textbf{0.7459}	&\textbf{0.7824}	&0.8059 &0.8834 &\textbf{3.66E-4} &\textbf{3.97E-4}	&4.79E-4	&5.33E-3\\
		\cmidrule{2-4} \cmidrule{5-6} \cmidrule{7-11}
		
		\multirow{20}{*}{}&\multirow{4}{*}{\emph{Adult}}&0.5	&\textbf{0.6381}	&\textbf{0.6884}	&0.8183	&0.9087	&\textbf{0.117}	&\textbf{0.117} &0.122 &8.382\\
		%\cmidrule{3-4} \cmidrule{5-6} \cmidrule{7-8}
		\multirow{20}{*}{}&\multirow{4}{*}{}&0.75	&\textbf{0.6113}	&\textbf{0.6794}	& 0.7359	&0.9085	&\textbf{0.123}	&\textbf{0.120} &0.129 &8.347\\
  
		\multirow{20}{*}{}&\multirow{4}{*}{}&1.0	&\textbf{0.6094}	&\textbf{0.6784}	&0.7352	&0.9082	&\textbf{0.118}	&\textbf{0.117} &0.123 &8.318\\

		\multirow{20}{*}{}&\multirow{4}{*}{}&2.0	&\textbf{0.6078}	&\textbf{0.6778}	& 0.7340 	&0.9081	&\textbf{0.117}	&\textbf{0.116} &0.121 &8.211\\
        \cmidrule{2-4} \cmidrule{5-6} \cmidrule{7-11}

		\multirow{20}{*}{}&\multirow{4}{*}{\emph{Student's t}}&0.5	&\textbf{0.7998}	&\textbf{0.8059}	& 0.8101	&0.8153	&\textbf{0.903}	&\textbf{0.909} &1.019 &30.523\\
		%\cmidrule{3-4} \cmidrule{5-6} \cmidrule{7-8}
		\multirow{20}{*}{}&\multirow{4}{*}{}&0.75	&\textbf{0.7971}	&\textbf{0.7990}	& 0.8018 	&0.8141	&\textbf{0.962}	&\textbf{0.977} &1.096 &32.314\\
  
		\multirow{20}{*}{}&\multirow{4}{*}{}&1.0	&\textbf{0.7967}	&\textbf{0.7985}	& 0.7974 	&0.8139	&\textbf{0.952}	&\textbf{0.961} &1.085 &32.720\\

		\multirow{20}{*}{}&\multirow{4}{*}{}&2.0	&\textbf{0.7963}	&\textbf{0.7991}	& 0.7968 	&0.8142	&\textbf{0.912}	&\textbf{0.961} &1.043 &32.210\\
        \cmidrule{2-4} \cmidrule{5-6} \cmidrule{7-11}
        
		\multirow{20}{*}{}&\multirow{4}{*}{\emph{Laplace}}&0.5	&\textbf{0.5141}	&\textbf{0.5286}	& 0.5371	&0.5568	&\textbf{1.002}	&\textbf{0.989} &1.164 &32.112\\
		%\cmidrule{3-4} \cmidrule{5-6} \cmidrule{7-8}
		\multirow{20}{*}{}&\multirow{4}{*}{}&0.75	&\textbf{0.5113}	&\textbf{0.5234}	& 0.5290 	&0.5565	&\textbf{0.936}	&\textbf{0.931} &1.054 &31.065\\
  
		\multirow{20}{*}{}&\multirow{4}{*}{}&1.0	&\textbf{0.5102}	&\textbf{0.5206}	& 0.5285 	&0.5560	&\textbf{0.912}	&\textbf{0.917} &1.048 &32.162\\

		\multirow{20}{*}{}&\multirow{4}{*}{}&2.0	&\textbf{0.5100}	&\textbf{0.5201}	& 0.5274 	&0.5558	&\textbf{0.987}	&\textbf{0.979} &1.135 &32.728\\
        \cmidrule{2-4} \cmidrule{5-6} \cmidrule{7-11}
        
		\multirow{20}{*}{}&\multirow{4}{*}{\emph{$\chi^2$}}&0.5	&\textbf{0.5541}	&\textbf{0.5701}	& 0.5766	&0.5937	&\textbf{0.919}	&\textbf{0.909} &1.043 &30.613\\
		%\cmidrule{3-4} \cmidrule{5-6} \cmidrule{7-8}
		\multirow{20}{*}{}&\multirow{4}{*}{}&0.75	&\textbf{0.5522}	&\textbf{0.5621}	& 0.5669 	&0.5935	&\textbf{0.961}	&\textbf{0.951} &1.160 &32.147\\
  
		\multirow{20}{*}{}&\multirow{4}{*}{}&1.0	&\textbf{0.5515}	&\textbf{0.5619}	& 0.5653 	&0.5933	&\textbf{0.918}	&\textbf{0.902} &1.025 &30.685\\

		\multirow{20}{*}{}&\multirow{4}{*}{}&2.0	&\textbf{0.5513}	&\textbf{0.5606}	& 0.5651 	&0.5930	&\textbf{0.911}	&\textbf{0.910} &1.034 &30.723\\
		\botrule
	\end{tabular}}
	\label{table3}
\end{table*}

The tuning strategy is to be consistent with theoretical recommendations, for example, the stepsize of Theorem \ref{ee} should be smaller than the one of Theorem \ref{BC}. Therefore,
for \emph{Diabetes} dataset, we fix the batch size $m$ to 24, thus the clipping level $\lambda$ could be tuned accordingly to avoid accessing $x^*$ (one may use the grid search to roughly estimate the scope of $x^*$). We set the stepsize $\gamma$ to $1E2,$ $5E-3,$ $6E-3$ and $8E-3$ for UC thm, CC thm, DP-SGD and DP-GD; for \emph{Adult} and synthetic dataset, we fix the batch size $m$ to 200, the clipping level $\lambda$ to 0.87, 0.54, 0.74 for UC thm, CC thm, DP-SGD. We set the stepsize $\gamma$ to $2E-4,$ $5E-4,$ $4E-4,$ and $1E-3$ for UC thm, CC thm, DP-SGD, and DP-GD.

\begin{figure*}[t]
	\centering
    %\captionsetup[subfigure]{labelformat=empty}
	\subfigure{
		\includegraphics[width=0.23\textwidth]{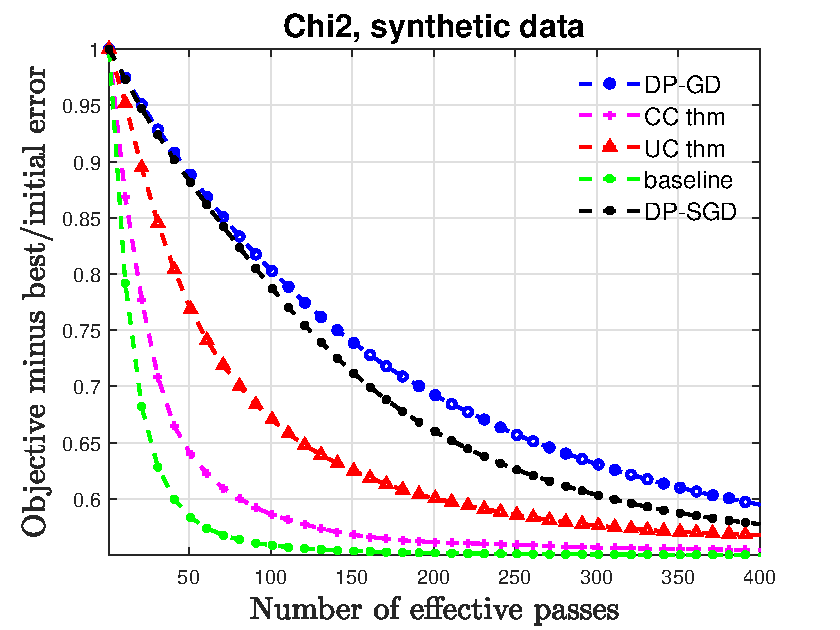}}
	\subfigure{
		\includegraphics[width=0.23\textwidth]{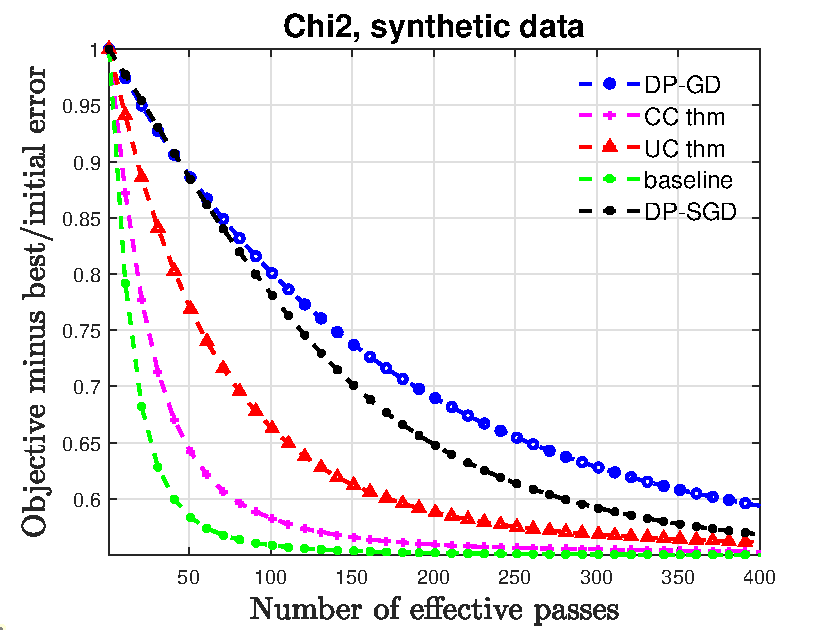}}
	\subfigure{
		\includegraphics[width=0.23\textwidth]{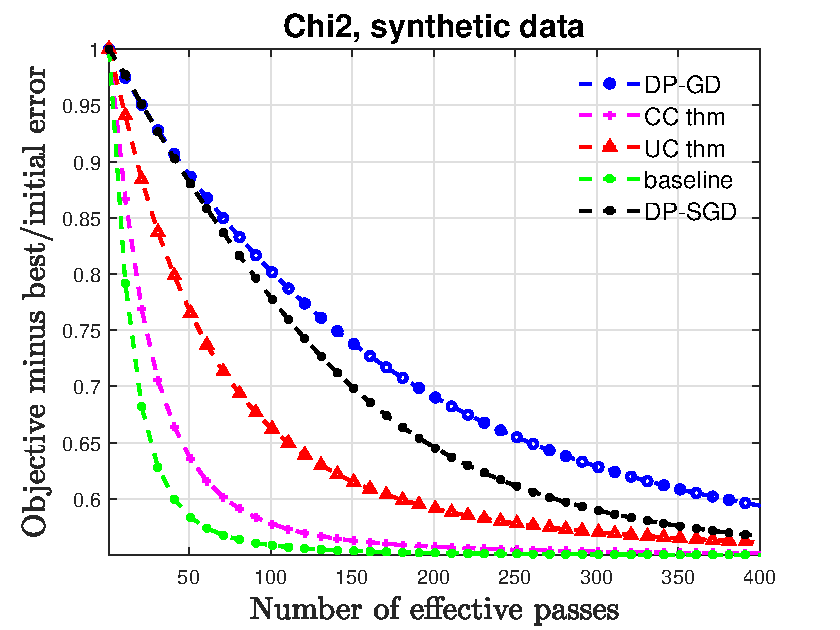}}
	\subfigure{
		\includegraphics[width=0.23\textwidth]{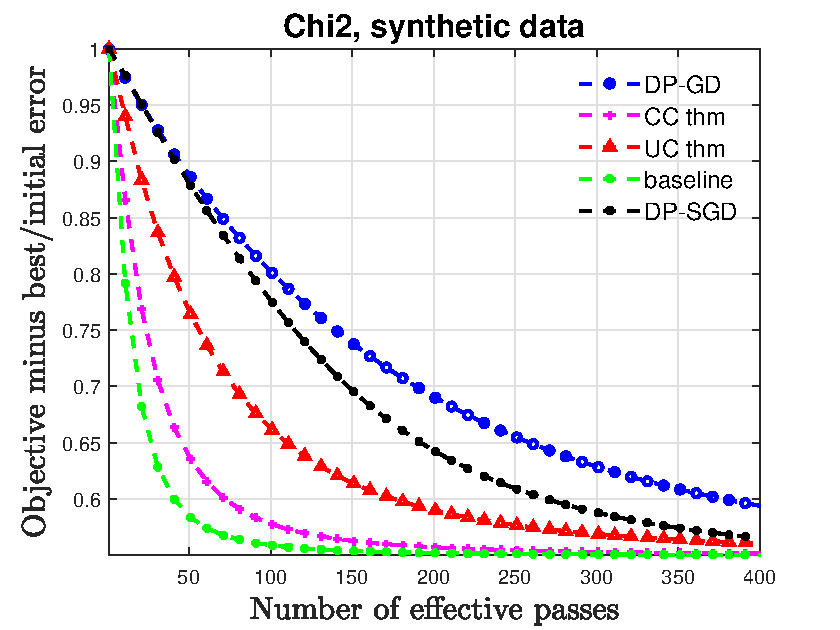}}
  \newline
  	\subfigure{
		\includegraphics[width=0.23\textwidth]{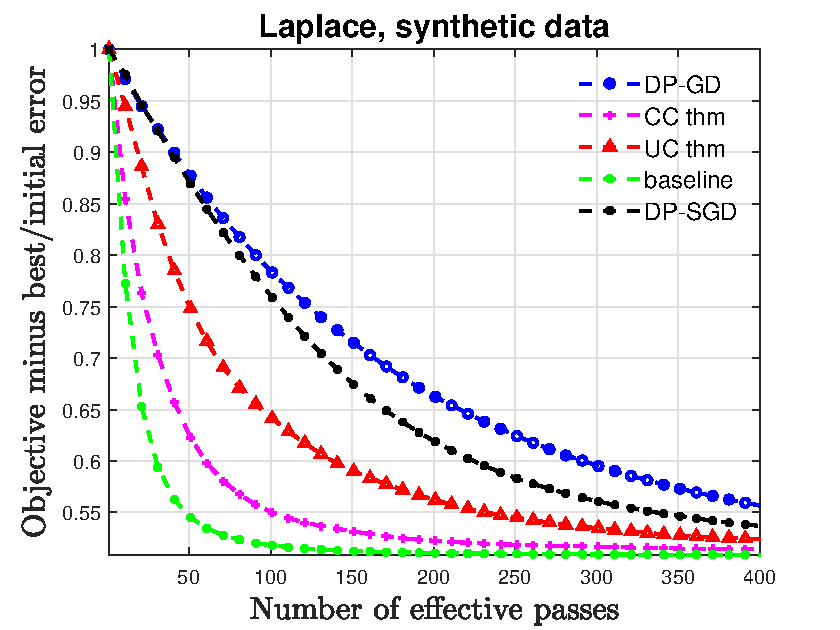}}
	\subfigure{
		\includegraphics[width=0.23\textwidth]{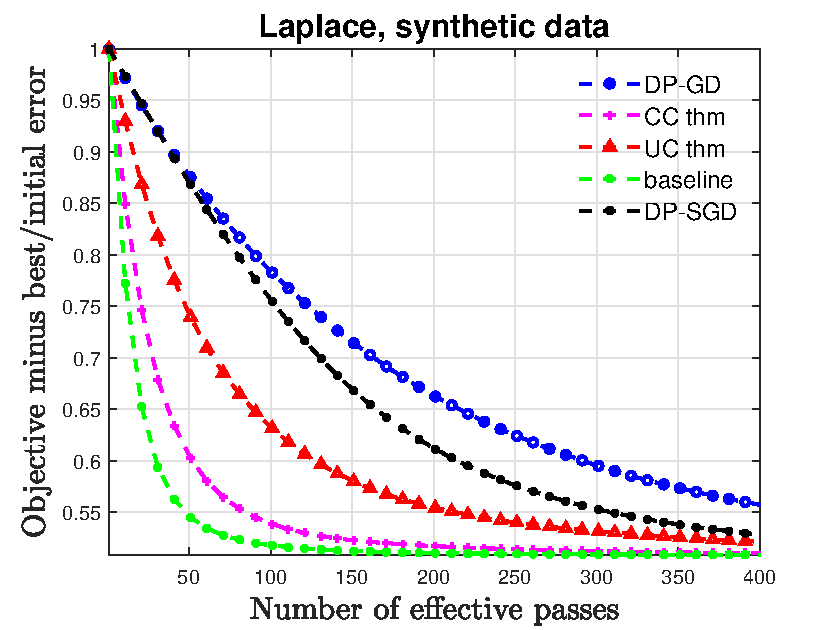}}
	\subfigure{
		\includegraphics[width=0.23\textwidth]{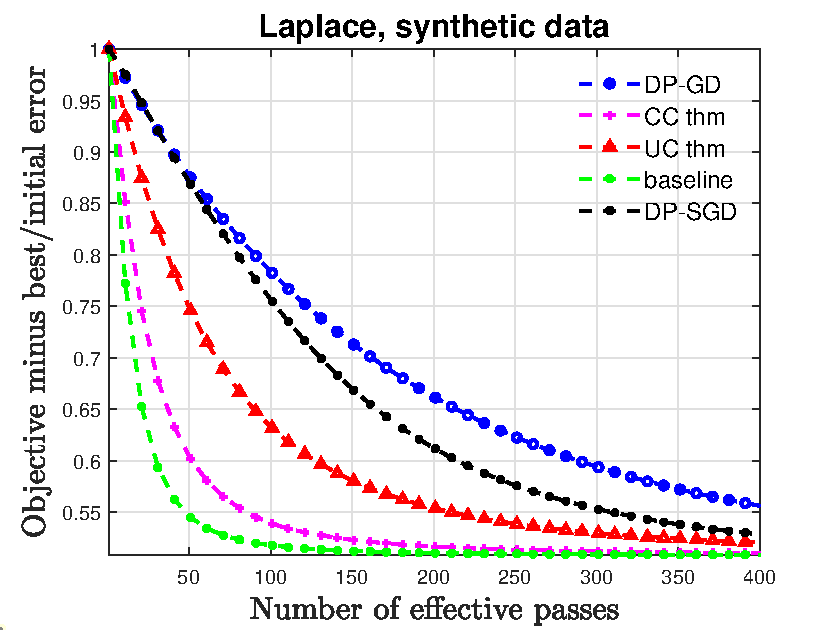}}
	\subfigure{
		\includegraphics[width=0.23\textwidth]{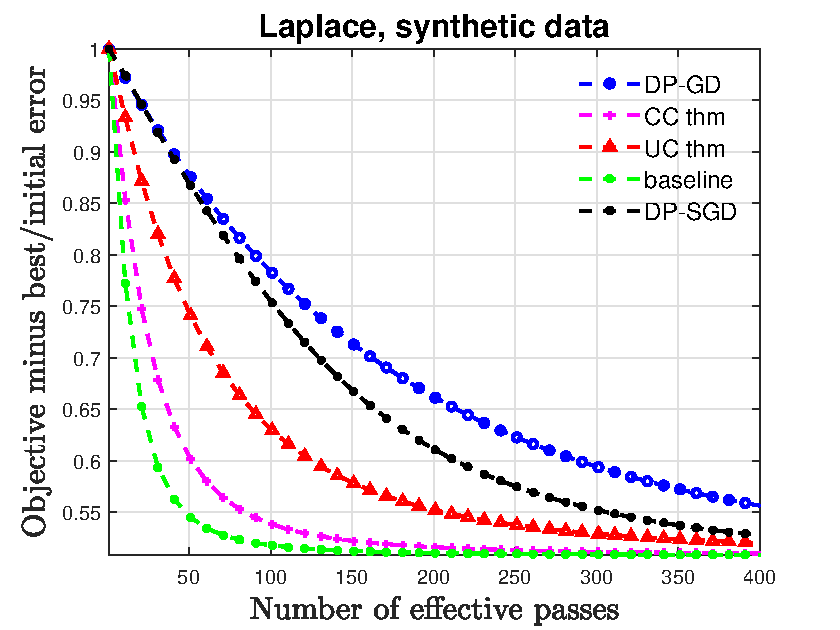}}
  \newline
   \setcounter{subfigure}{0}
  	\subfigure[$\epsilon = 0.5$]{
		\includegraphics[width=0.23\textwidth]{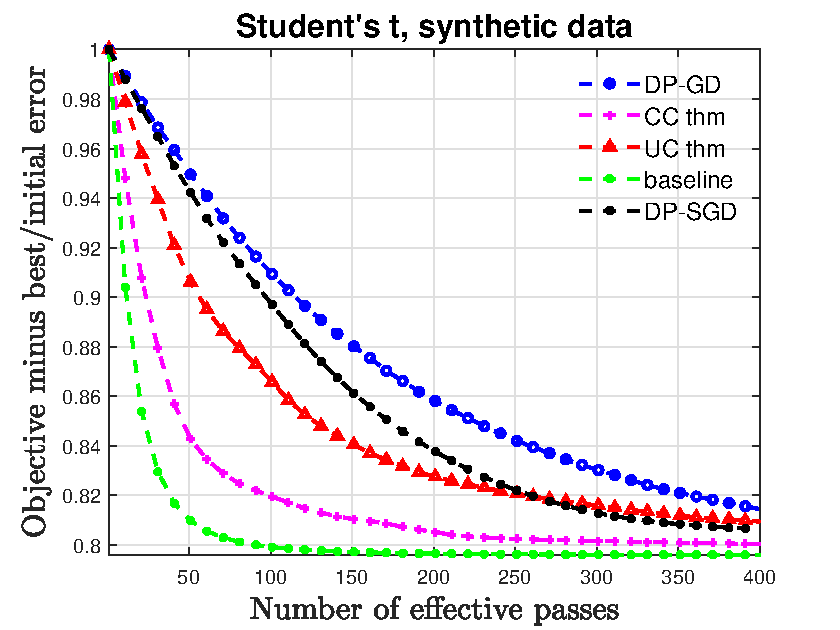}}
	\subfigure[$\epsilon = 0.75$]{
		\includegraphics[width=0.23\textwidth]{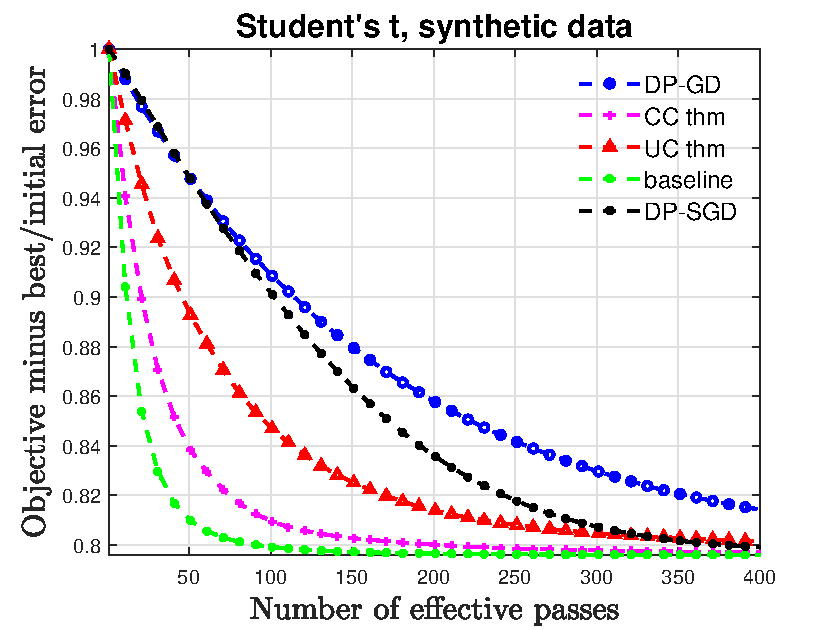}}
	\subfigure[$\epsilon = 1.0$]{
		\includegraphics[width=0.23\textwidth]{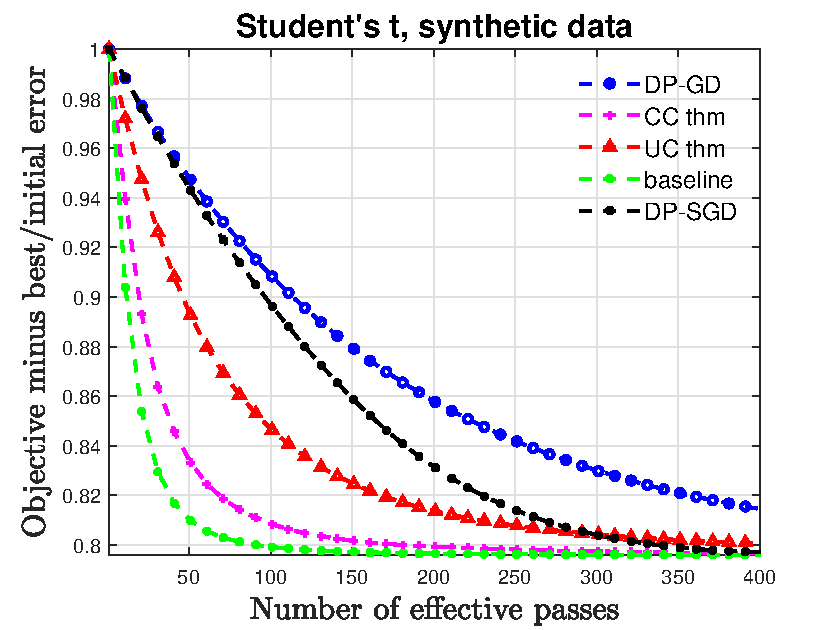}}
	\subfigure[$\epsilon = 2.0$]{
		\includegraphics[width=0.23\textwidth]{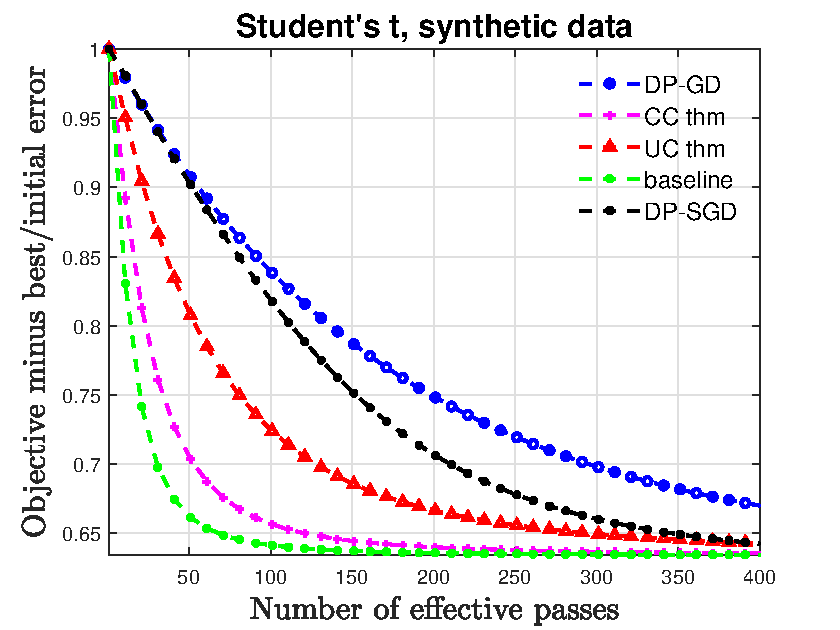}}
    \caption{Trajectories of the ridge regression model for the synthetic data. The three rows correspond to the \textit{Chi-squared distribution}, \textit{Laplace distribution}, and \textit{Student's t-distribution}, respectively.}
 \label{fig4}
\end{figure*}

\begin{figure*}[t]
	\centering
    %\captionsetup[subfigure]{labelformat=empty}
    \subfigure{
		\includegraphics[width=0.23\textwidth]{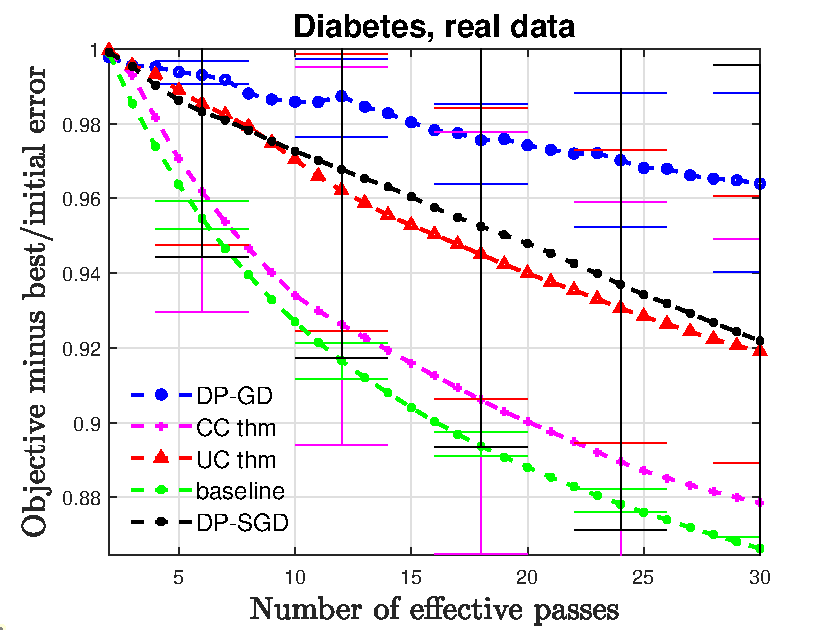}}
	\subfigure{
		\includegraphics[width=0.23\textwidth]{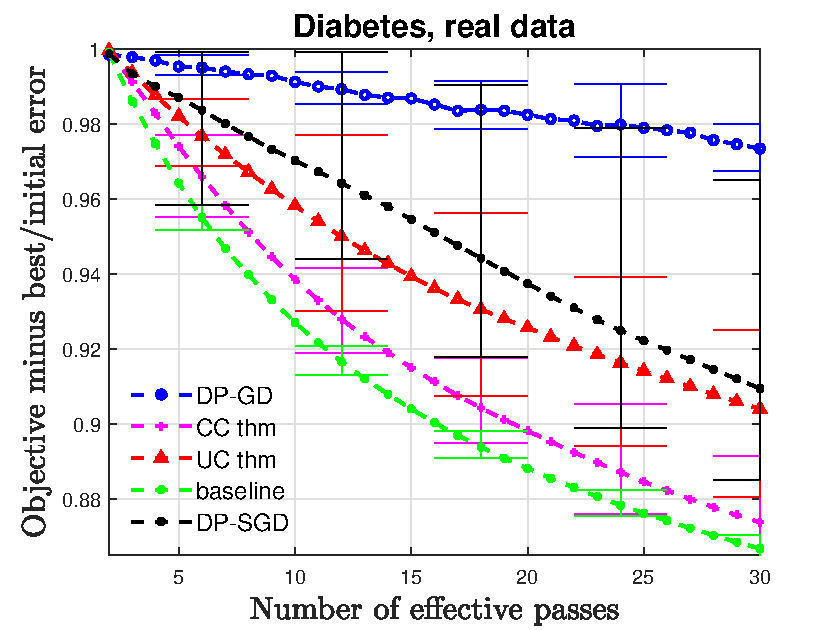}}
	\subfigure{
		\includegraphics[width=0.23\textwidth]{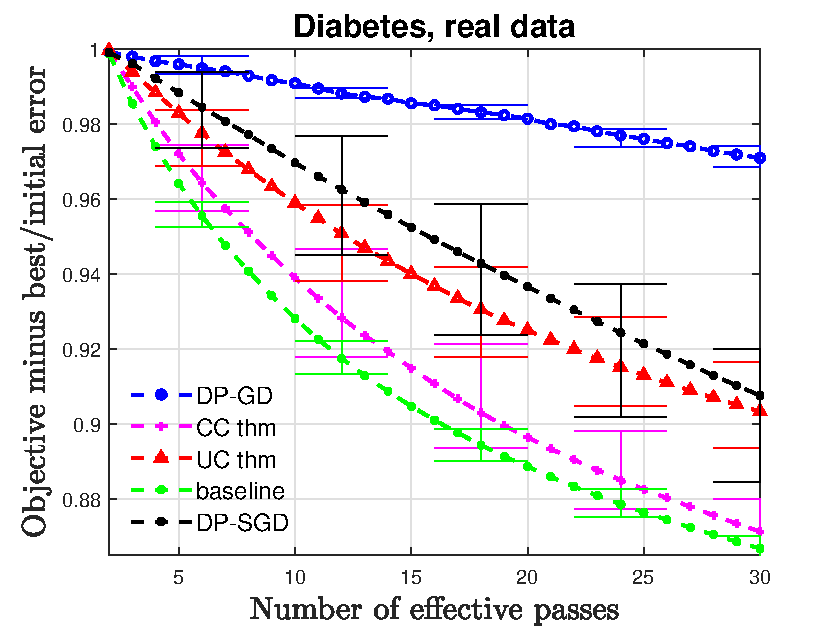}}
	\subfigure{
		\includegraphics[width=0.23\textwidth]{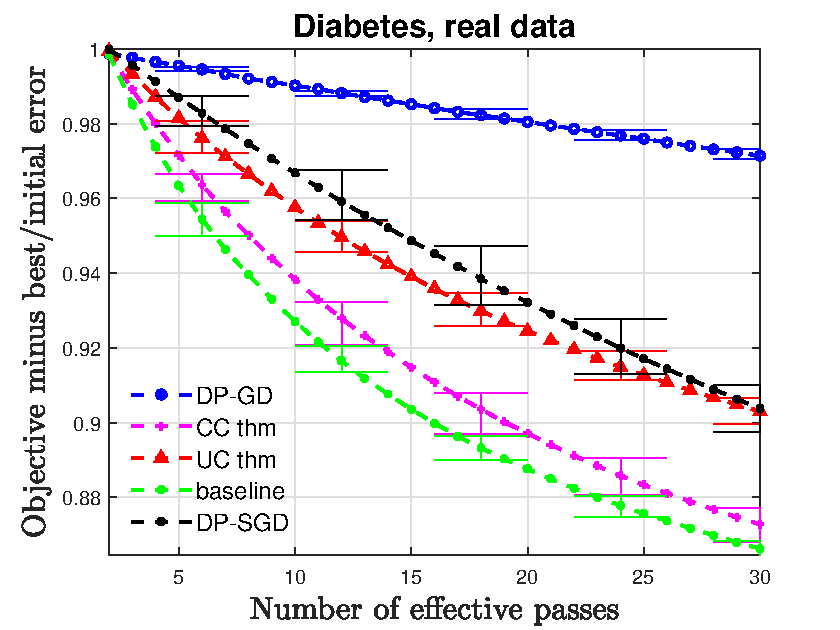}}
  \newline
	\subfigure{
		\includegraphics[width=0.23\textwidth]{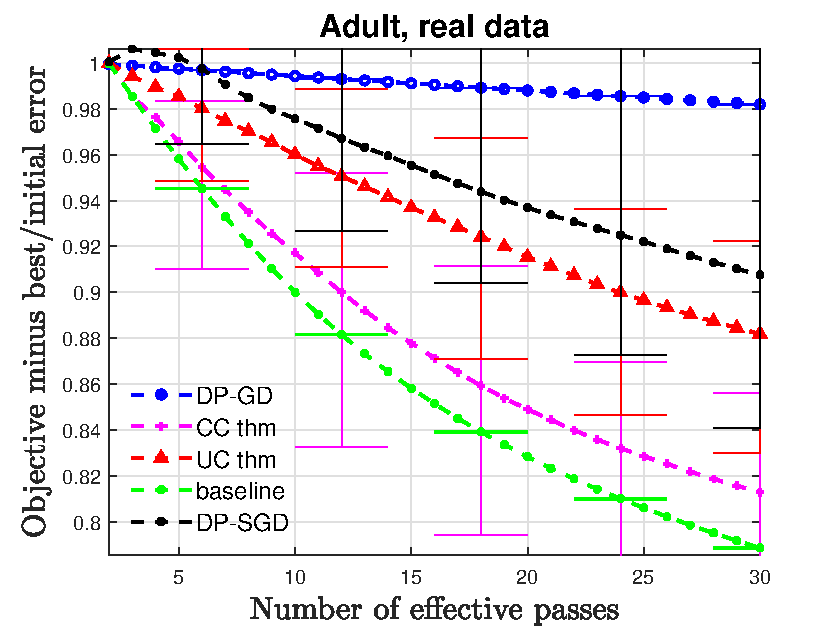}}
	\subfigure{
		\includegraphics[width=0.23\textwidth]{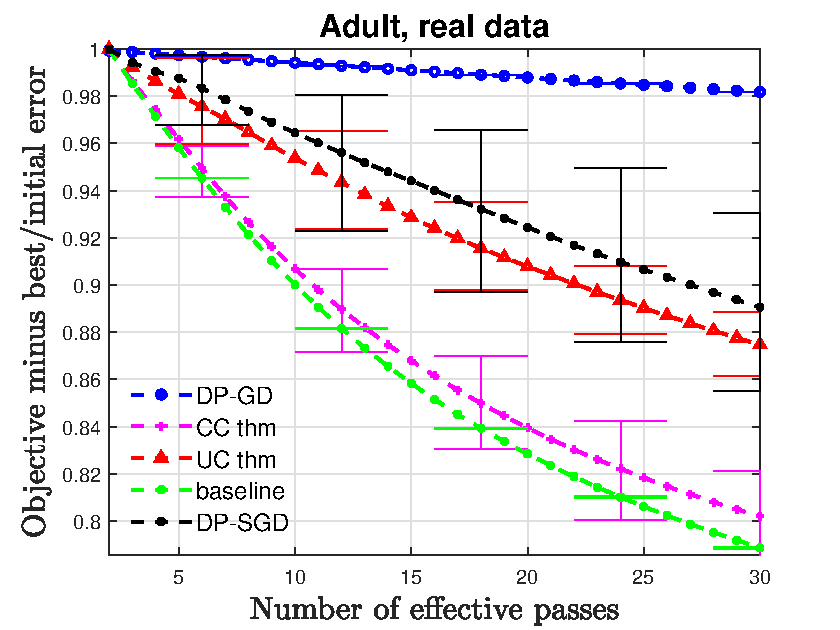}}
	\subfigure{
		\includegraphics[width=0.23\textwidth]{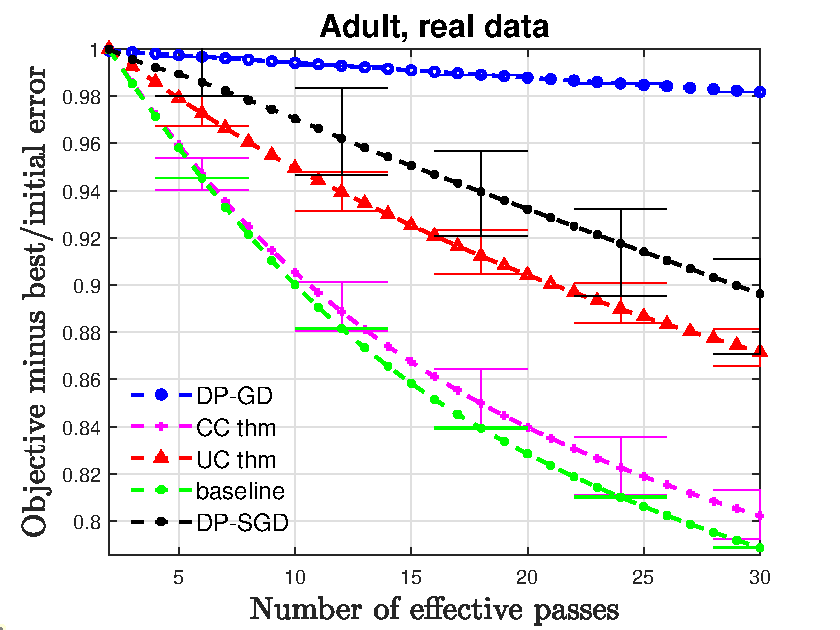}}
	\subfigure{
		\includegraphics[width=0.23\textwidth]{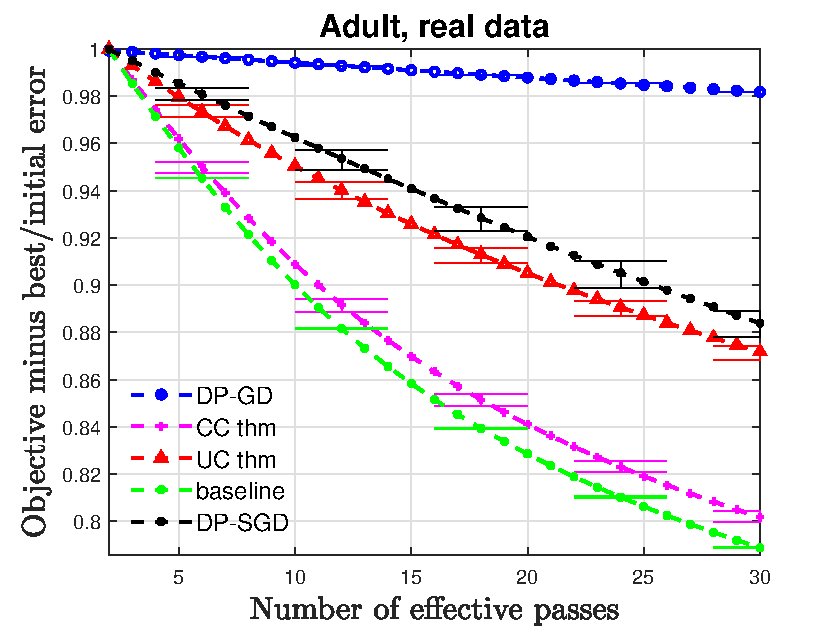}}
  \newline
    	\subfigure{
		\includegraphics[width=0.23\textwidth]{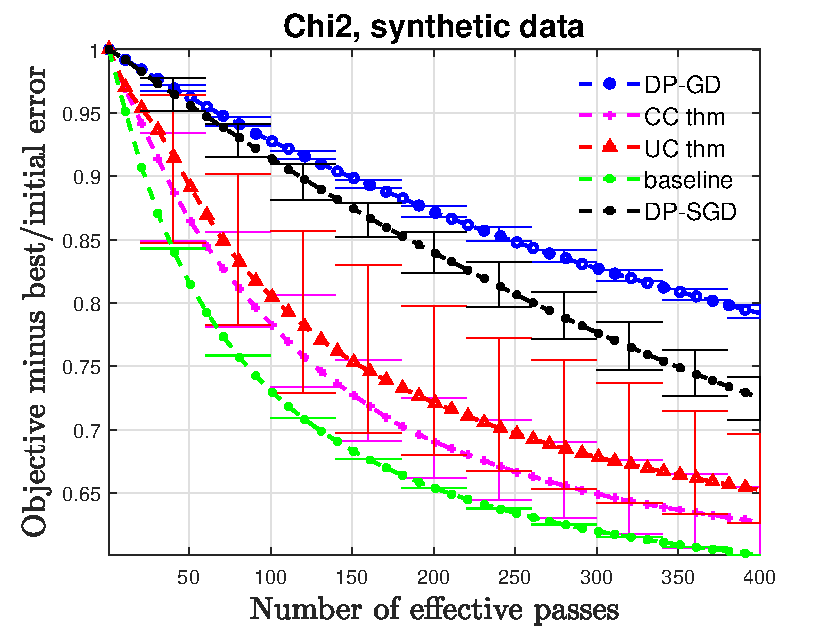}}
	\subfigure{
		\includegraphics[width=0.23\textwidth]{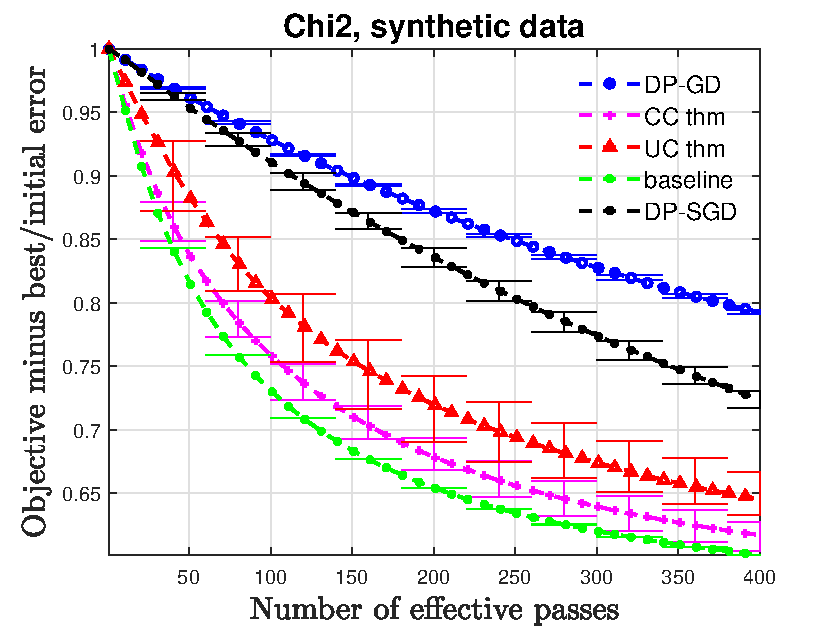}}
	\subfigure{
		\includegraphics[width=0.23\textwidth]{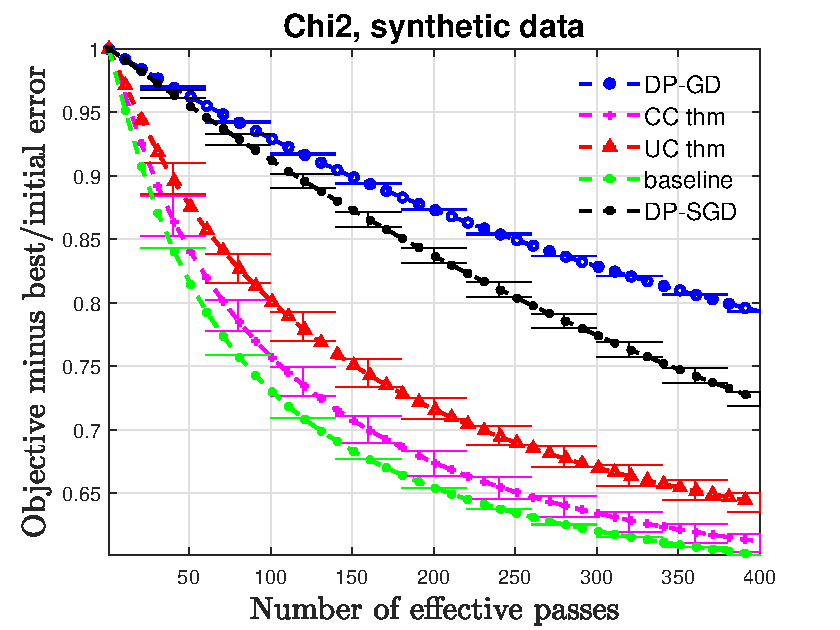}}
	\subfigure{
		\includegraphics[width=0.23\textwidth]{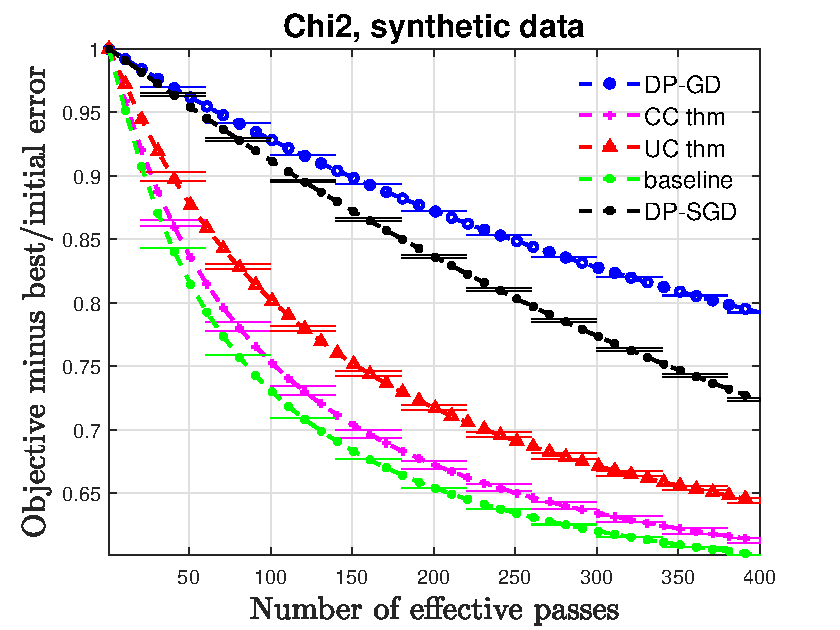}}
  \newline
    	\subfigure{
		\includegraphics[width=0.23\textwidth]{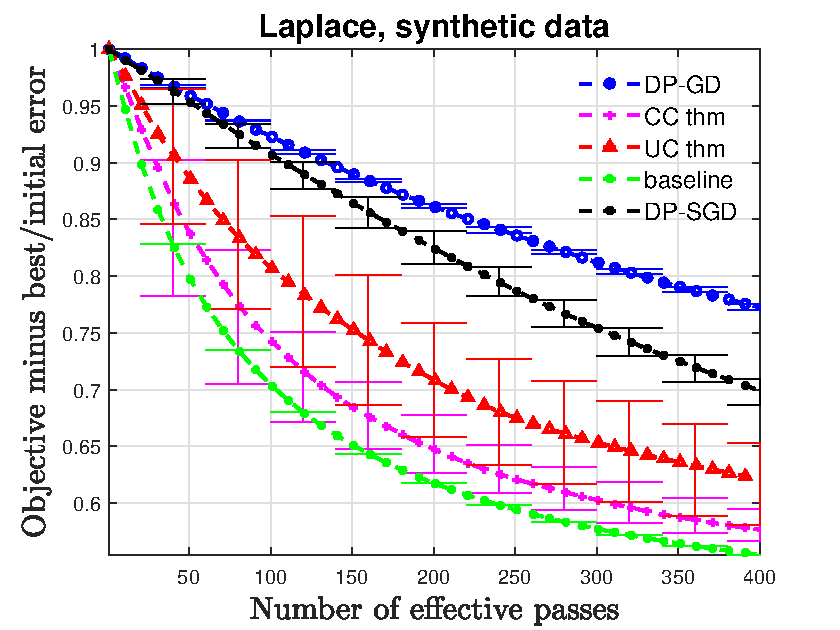}}
	\subfigure{
		\includegraphics[width=0.23\textwidth]{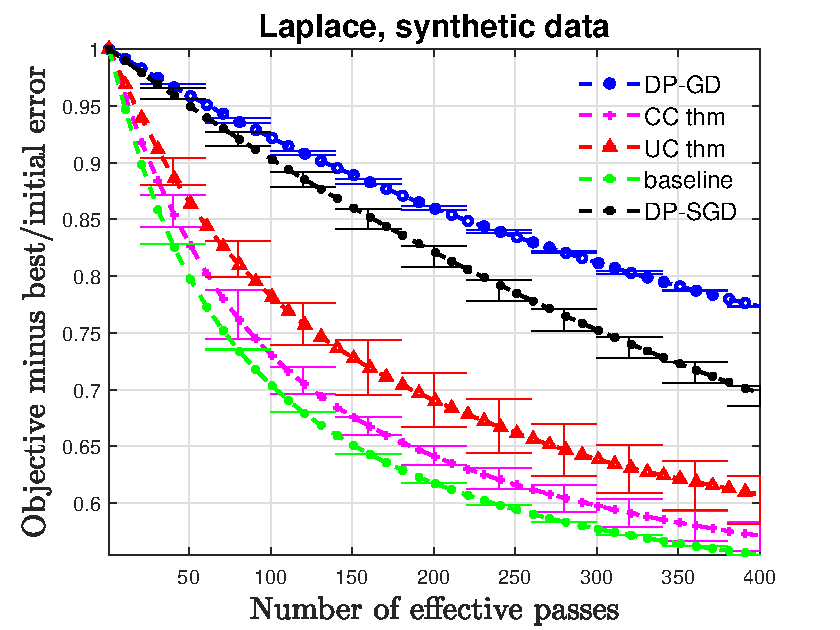}}
	\subfigure{
		\includegraphics[width=0.23\textwidth]{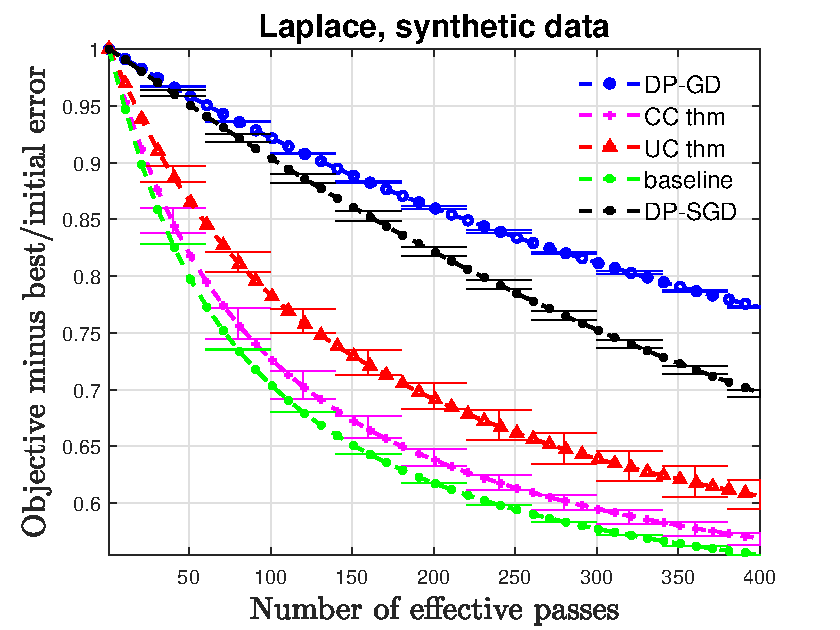}}
	\subfigure{
		\includegraphics[width=0.23\textwidth]{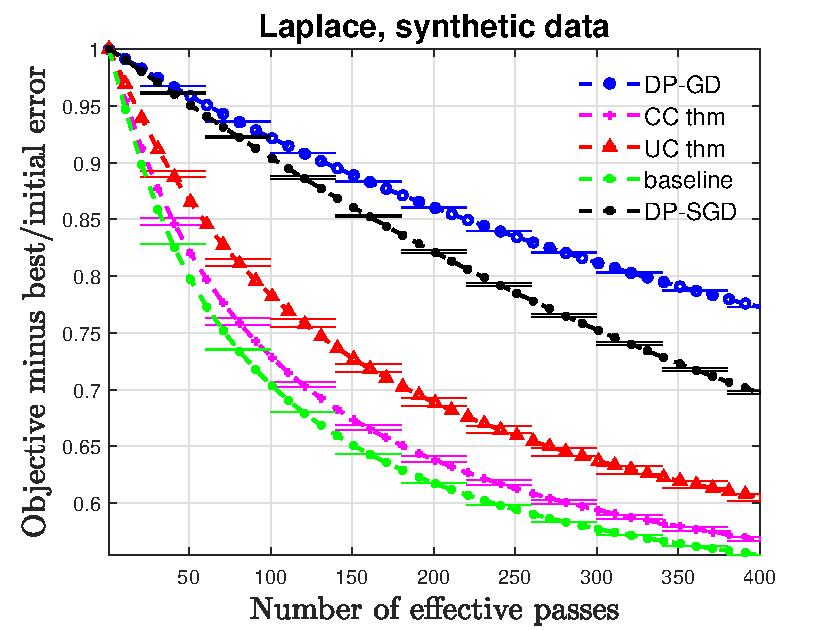}}
  \newline
   \setcounter{subfigure}{0}
  	\subfigure[$\epsilon = 0.5$]{
		\includegraphics[width=0.23\textwidth]{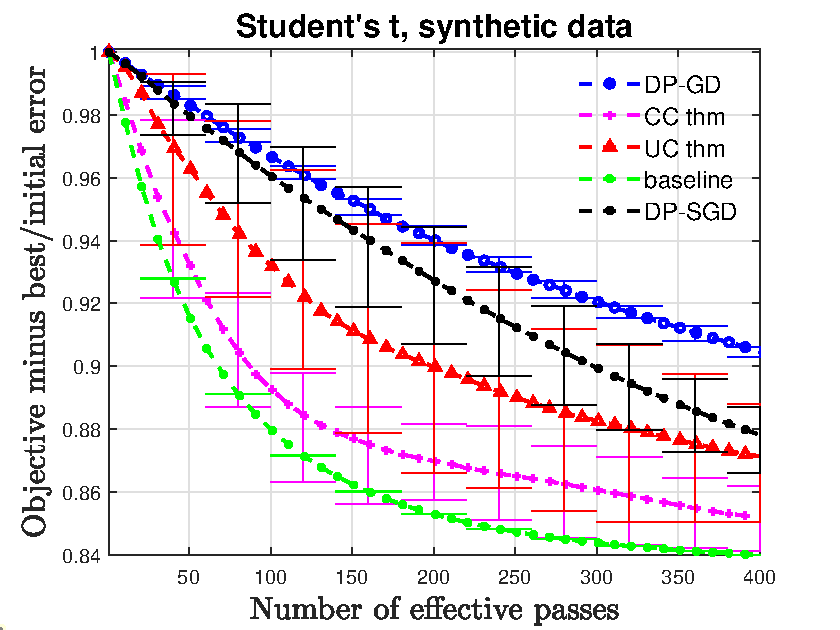}}
	\subfigure[$\epsilon = 0.75$]{
		\includegraphics[width=0.23\textwidth]{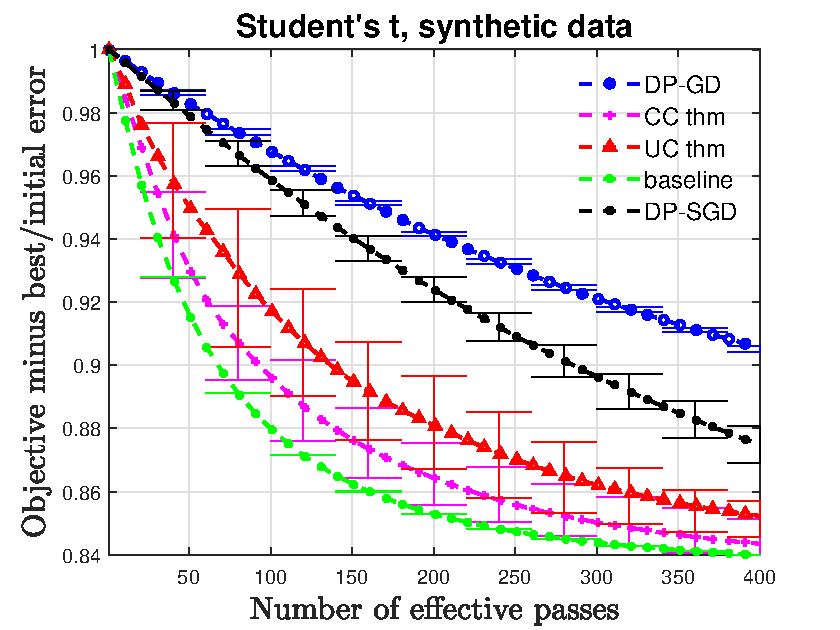}}
	\subfigure[$\epsilon = 1.0$]{
		\includegraphics[width=0.23\textwidth]{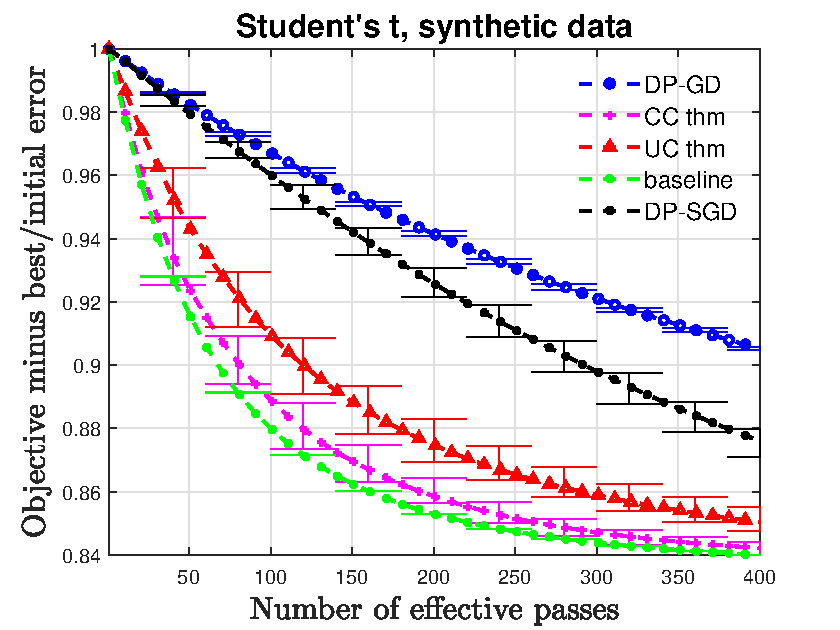}}
	\subfigure[$\epsilon = 2.0$]{
		\includegraphics[width=0.23\textwidth]{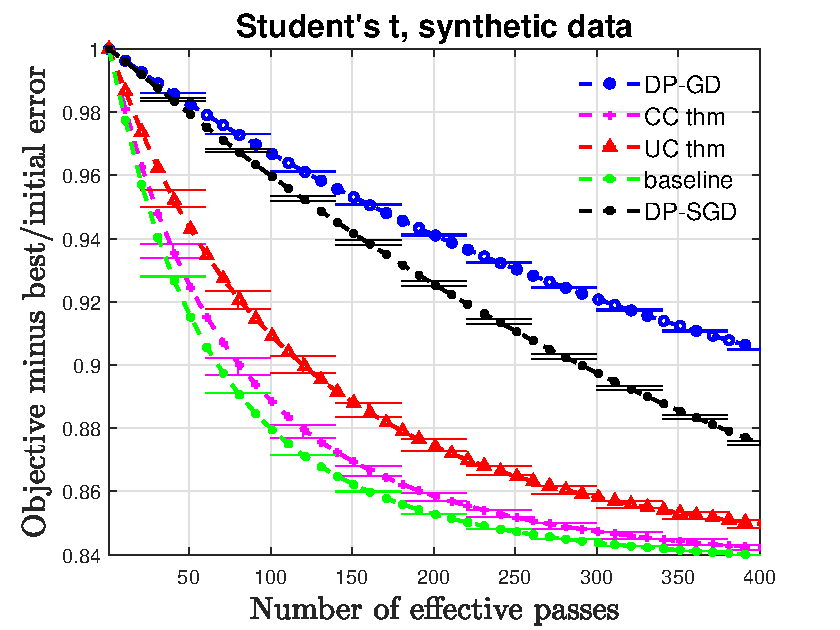}}
    \caption{Trajectories and deviations of the logistic regression model. The five rows correspond to the datasets of \textit{Diabetes}, \textit{Adult}, \textit{Chi-squared distribution}, \textit{Laplace distribution}, and \textit{Student's t-distribution}, respectively.}
 \label{fig5}
\end{figure*}

\begin{figure*}[t]
	\centering
    %\captionsetup[subfigure]{labelformat=empty}
    \subfigure{
		\includegraphics[width=0.23\textwidth]{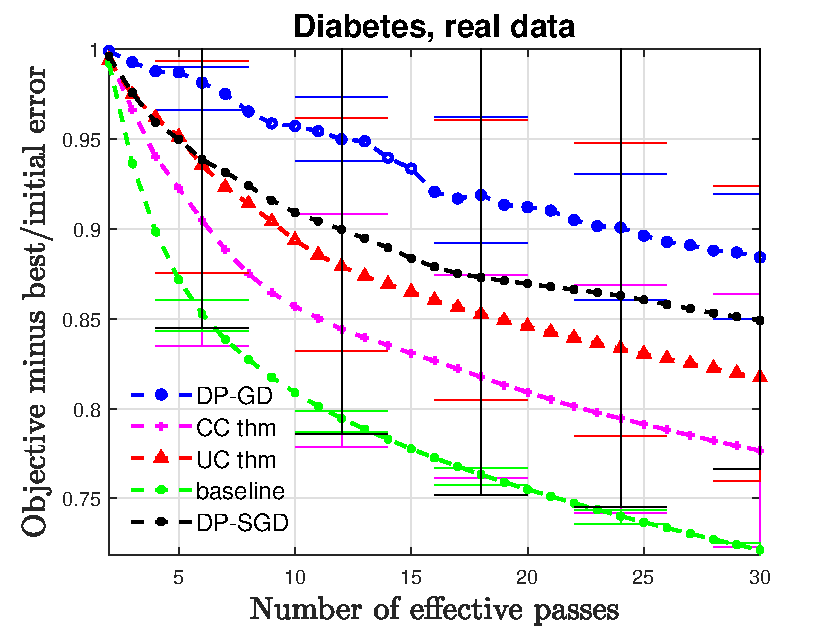}}
	\subfigure{
		\includegraphics[width=0.23\textwidth]{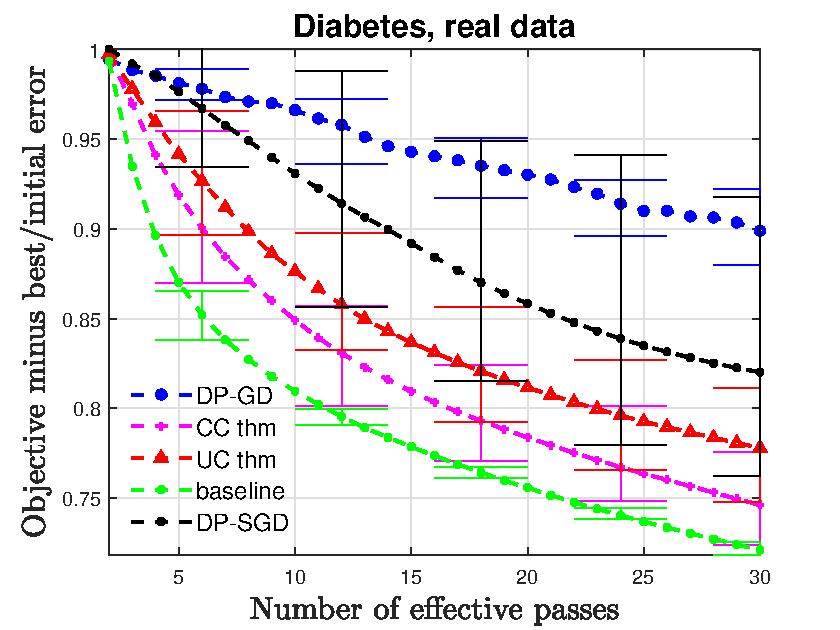}}
	\subfigure{
		\includegraphics[width=0.23\textwidth]{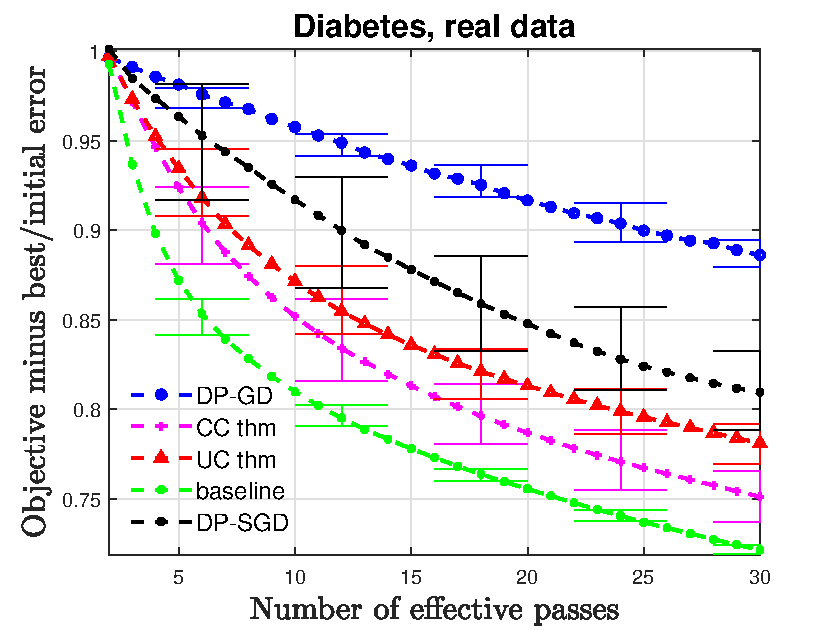}}
	\subfigure{
		\includegraphics[width=0.23\textwidth]{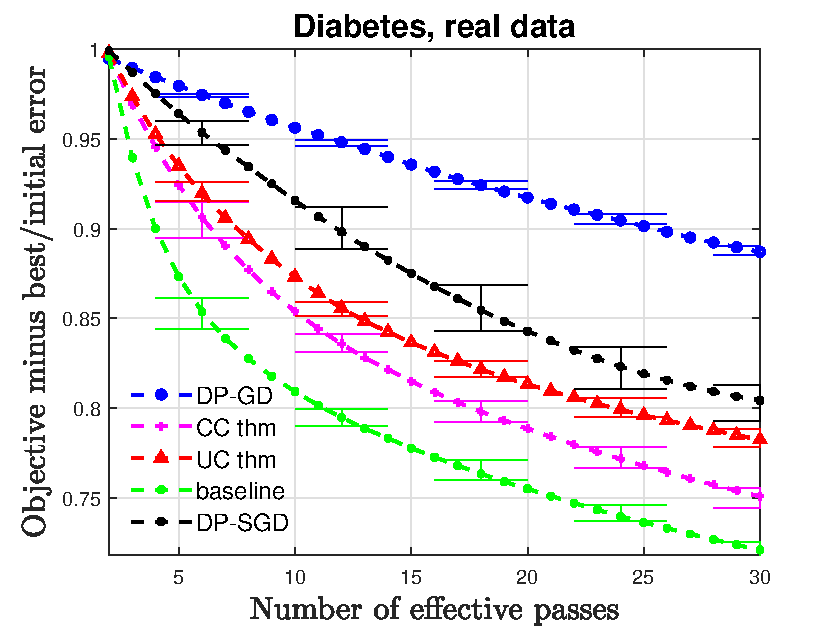}}
  \newline
	\subfigure{
		\includegraphics[width=0.23\textwidth]{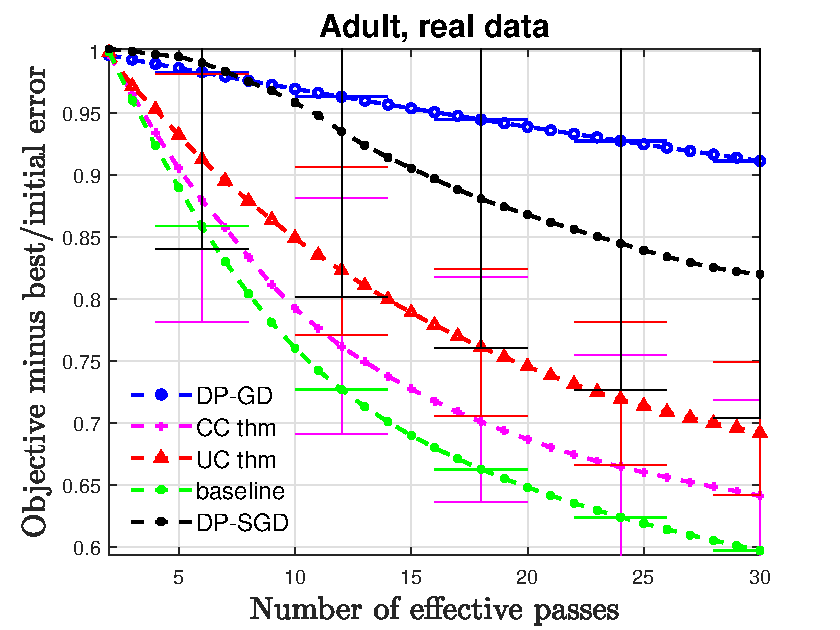}}
	\subfigure{
		\includegraphics[width=0.23\textwidth]{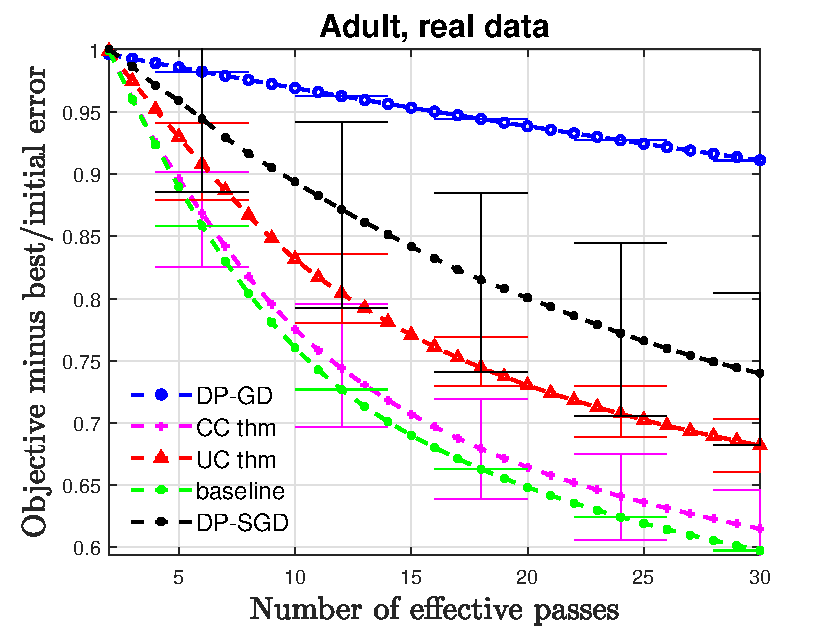}}
	\subfigure{
		\includegraphics[width=0.23\textwidth]{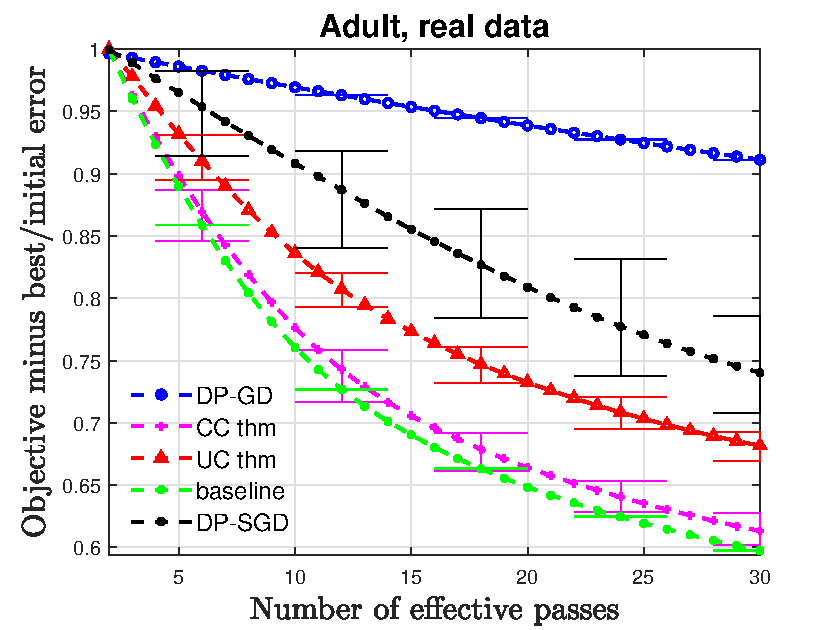}}
	\subfigure{
		\includegraphics[width=0.23\textwidth]{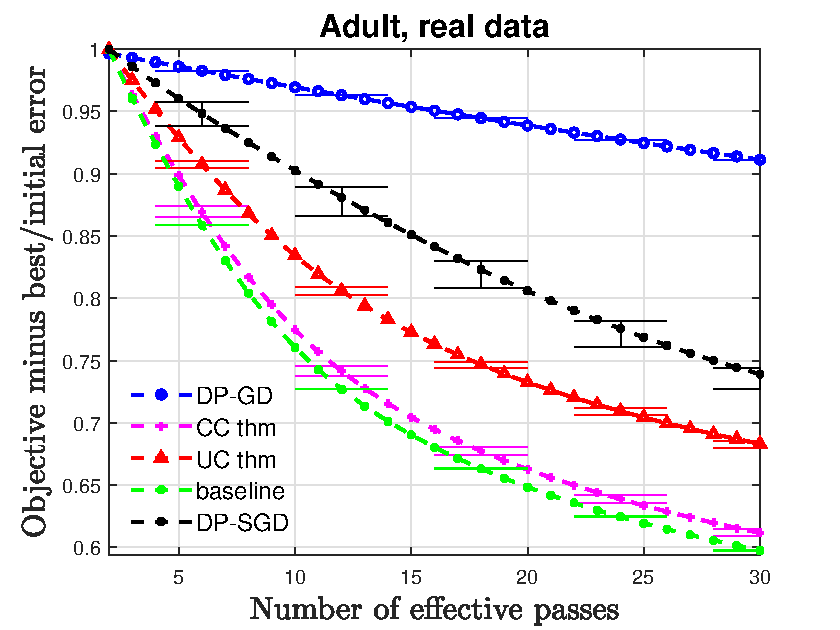}}
  \newline
    	\subfigure{
		\includegraphics[width=0.23\textwidth]{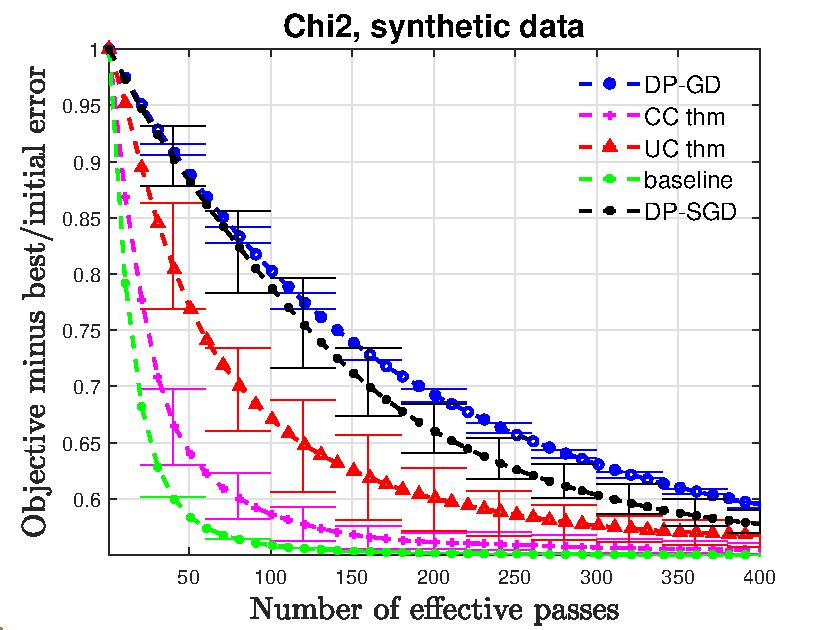}}
	\subfigure{
		\includegraphics[width=0.23\textwidth]{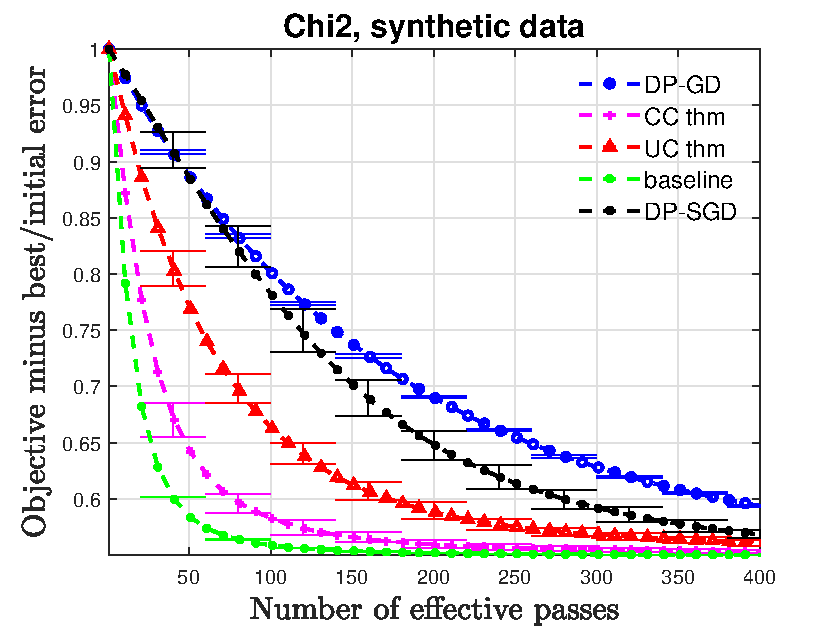}}
	\subfigure{
		\includegraphics[width=0.23\textwidth]{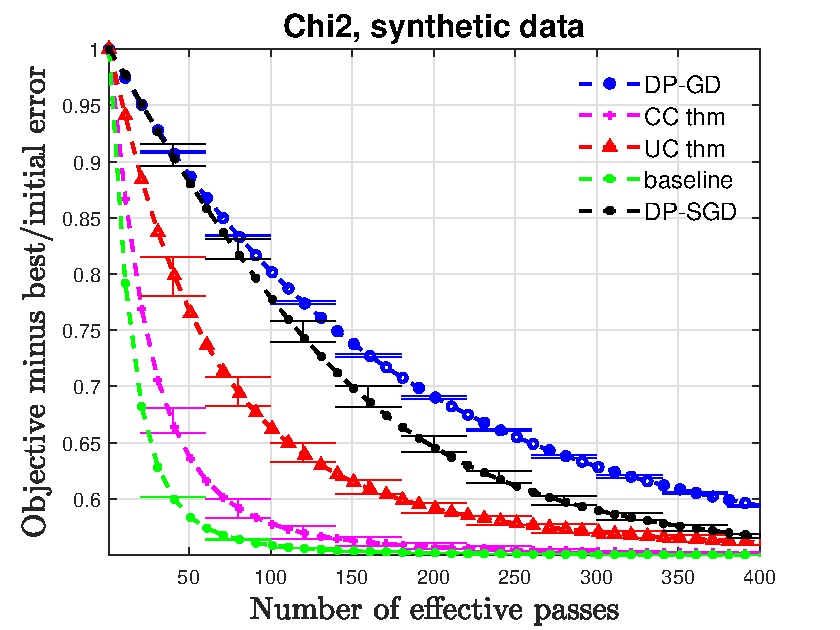}}
	\subfigure{
		\includegraphics[width=0.23\textwidth]{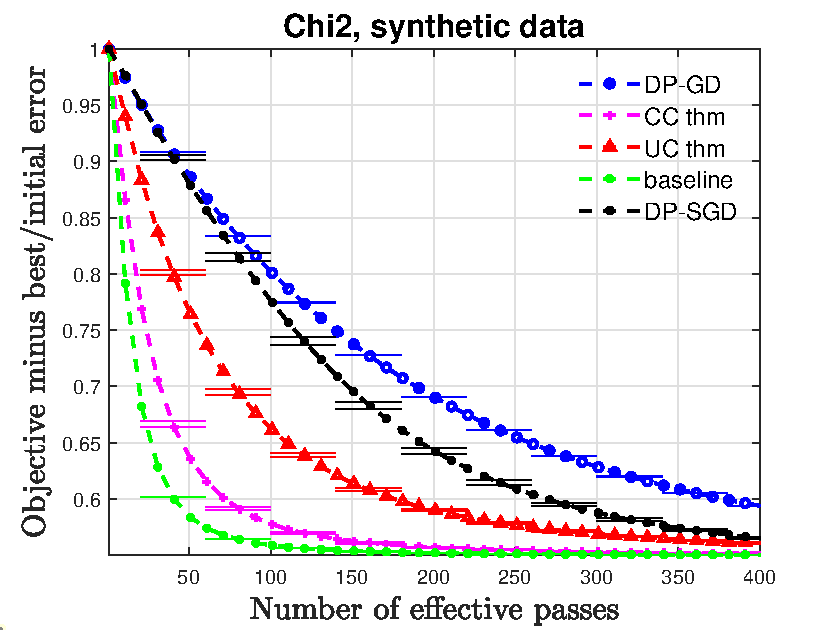}}
  \newline
    	\subfigure{
		\includegraphics[width=0.23\textwidth]{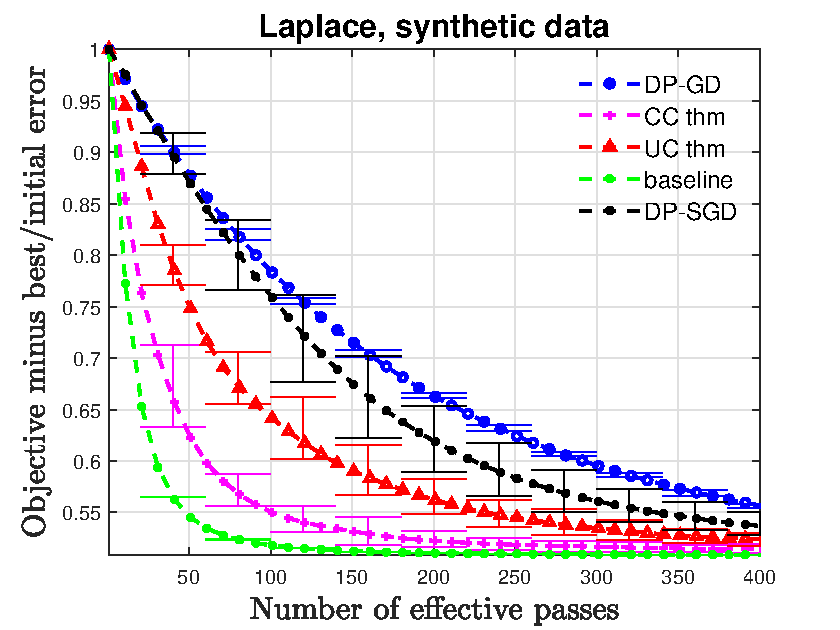}}
	\subfigure{
		\includegraphics[width=0.23\textwidth]{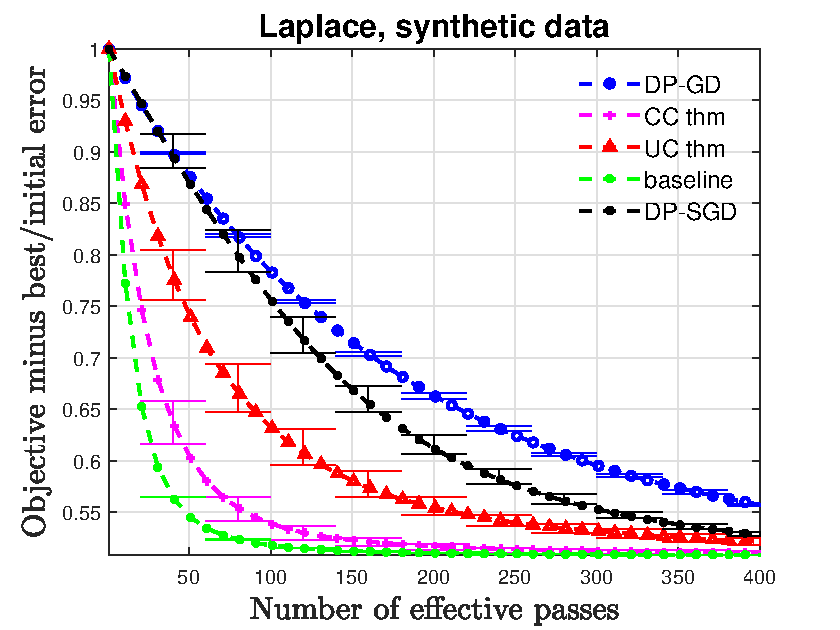}}
	\subfigure{
		\includegraphics[width=0.23\textwidth]{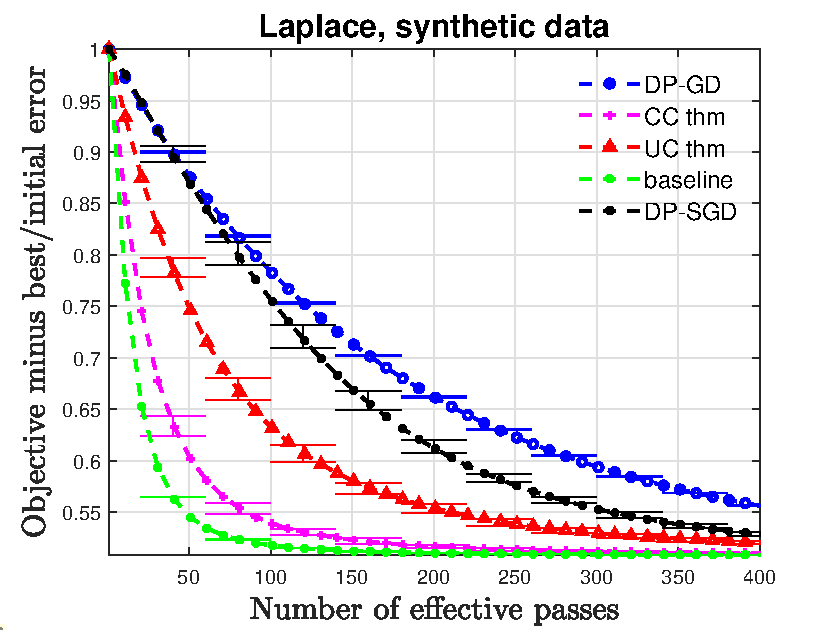}}
	\subfigure{
		\includegraphics[width=0.23\textwidth]{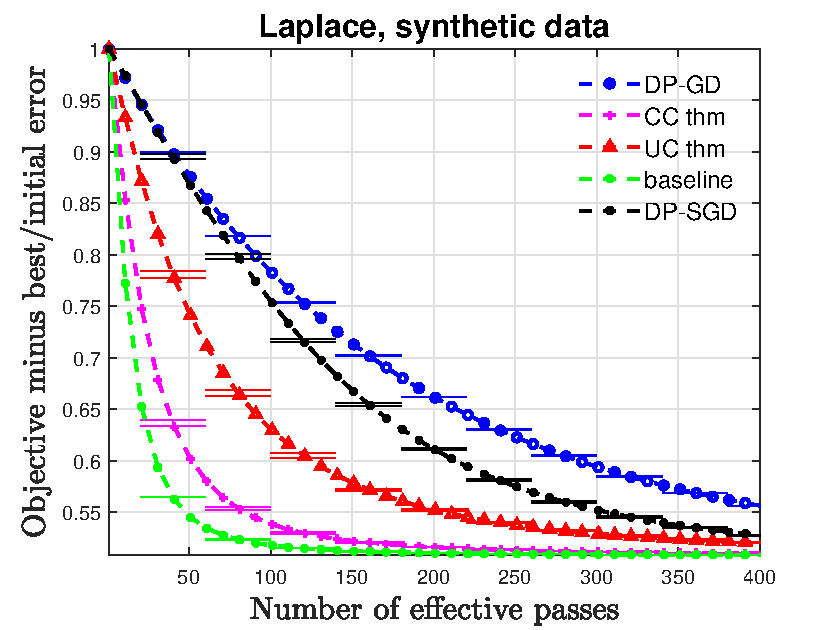}}
  \newline
   \setcounter{subfigure}{0}
  	\subfigure[$\epsilon = 0.5$]{
		\includegraphics[width=0.23\textwidth]{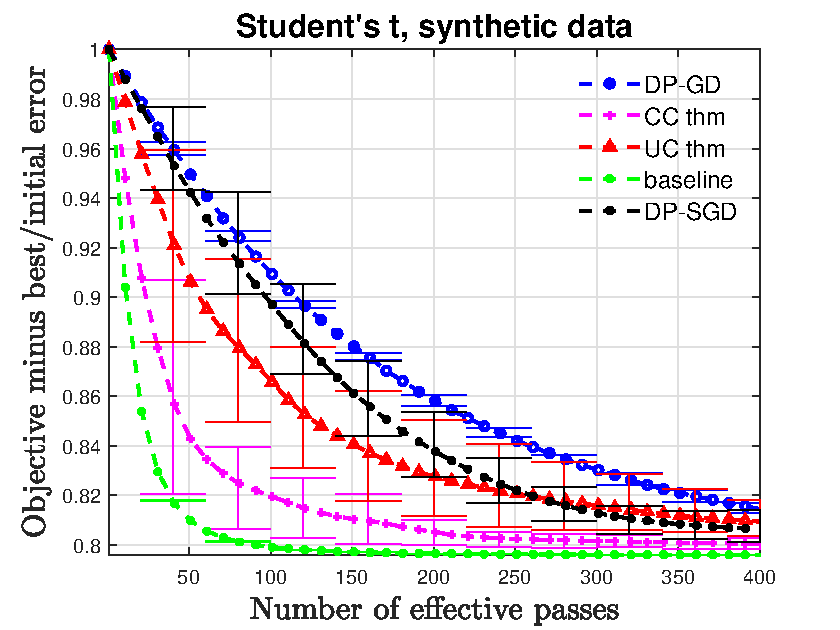}}
	\subfigure[$\epsilon = 0.75$]{
		\includegraphics[width=0.23\textwidth]{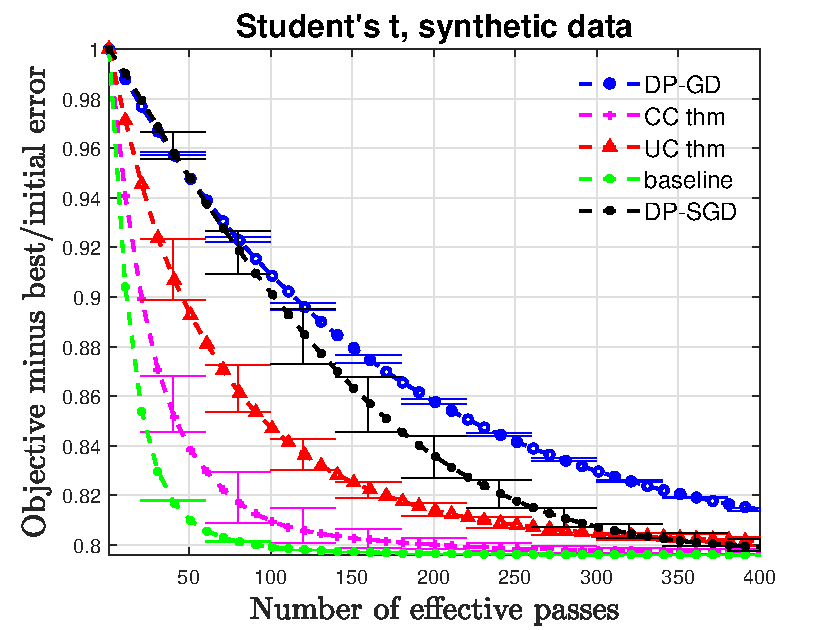}}
	\subfigure[$\epsilon = 1.0$]{
		\includegraphics[width=0.23\textwidth]{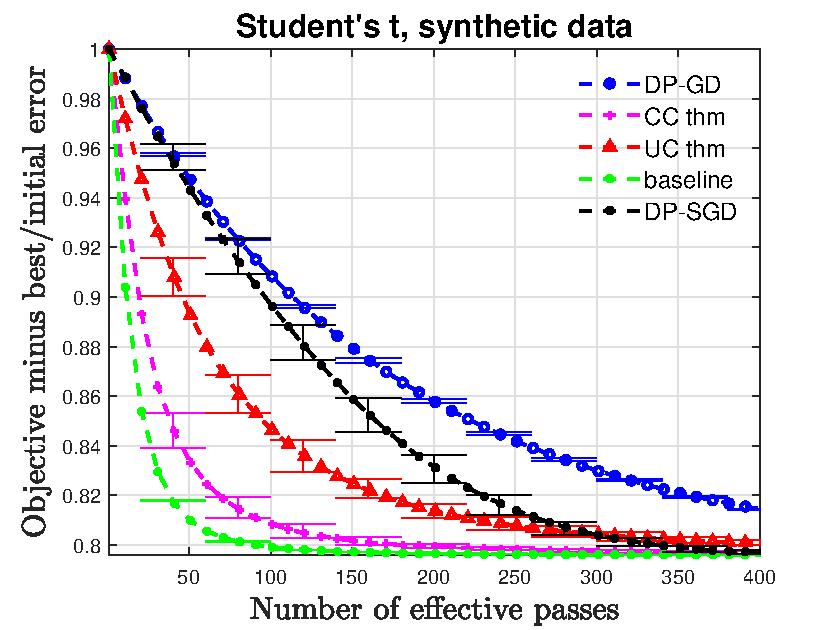}}
	\subfigure[$\epsilon = 2.0$]{
		\includegraphics[width=0.23\textwidth]{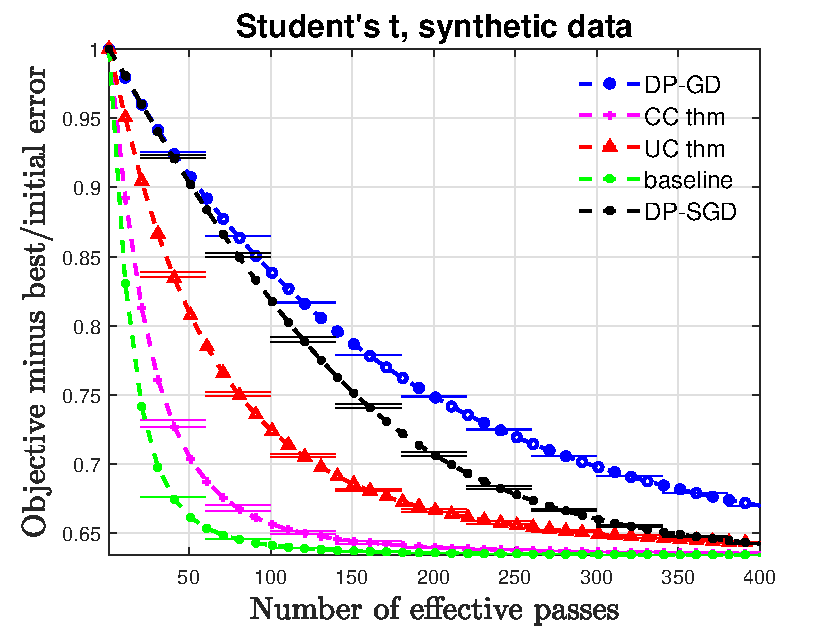}}
    \caption{Trajectories and deviations of the ridge regression model. The five rows correspond to the datasets of \textit{Diabetes}, \textit{Adult}, \textit{Chi-squared distribution}, \textit{Laplace distribution}, and \textit{Student's t-distribution}, respectively.}
 \label{fig6}
\end{figure*}

\bmhead{Acknowledgments}
We would like to acknowledge the support of Grant NSFC/RGC N\_CUHK 415/19, Grant ITF ITS/173/22FP, Grant RGC 14300219, 14302920, 14301121, and CUHK Direct Grant for Research for Prof. Tieyong Zeng; the NSFC General Program No. 62376235, Guangdong Basic and Applied Basic Research Foundation Nos. 2024A1515012399 and 2022A1515011652, HKBU Faculty Niche Research Areas No. RC-FNRA-IG/22-23/SCI/04, and HKBU CSD Departmental Incentive Grant for Prof. Bo Han. We thank all the reviewers for their thoughtful comments on this work.

\bmhead{Author contributions}
All authors contributed according to the order presented at the beginning of this work.

\bmhead{Funding}
This work was supported in part by Grant NSFC/RGC N\_CUHK 415/19, Grant ITF ITS/173/22FP, Grant RGC 14300219, 14302920, 14301121, and CUHK Direct Grant for Research.

\bmhead{Data availability}
The data supporting this study are publicly available and have been cited. 

\bmhead{Code availability}
Available at \url{https://github.com/Jchenhan/AClippingDP}.

\section*{Declarations}

\begin{itemize}
\item Conflict of interest: The authors declare that they have no conflict of interest.
\item Ethics approval: Not Applicable. 
\item Consent to participate: Not Applicable.
\item Consent for publication: Not Applicable.
\end{itemize}



\end{appendices}

%\bibliographystyle{apalike}
\clearpage
\bibliography{Manuscript}% common bib file

\end{document}